\documentclass{article}

\usepackage[final,nonatbib]{neurips_2020}

\usepackage[utf8]{inputenc} 
\usepackage[T1]{fontenc}    
\usepackage{hyperref}       
\usepackage{url}            
\usepackage{booktabs}       
\usepackage{amsfonts}       
\usepackage{nicefrac}       
\usepackage{microtype}      
\usepackage{cite}
\usepackage{amsmath,amssymb,amsfonts}
\usepackage{graphics}
\usepackage{graphicx}
\usepackage{textcomp}
\usepackage{xcolor}
\usepackage{mathtools}
\usepackage{enumitem}
\usepackage{bbm}
\usepackage{amsthm}
\usepackage[caption = false]{subfig}
\usepackage{tikz}
\usepackage{pgfplots}
\usepgfplotslibrary{fillbetween}
\pgfplotsset{compat=1.9}
\usepackage{algorithm}
\usepackage[noend]{algpseudocode}
\usepackage{adjustbox}

\def\BibTeX{{\rm B\kern-.05em{\sc i\kern-.025em b}\kern-.08em
    T\kern-.1667em\lower.7ex\hbox{E}\kern-.125emX}}

\newcommand\given{\;|\;}
\newcommand{\cP}{\mathcal{P}}

\newcommand{\gtMat}[0]{M_0}
\newcommand{\matSet}[0]{\mathcal{M}^{(\diff)}}
\newcommand{\vecV}[2]{v_{#1}^{#2}}
\newcommand{\vecU}[2]{u_{#1}^{#2}}

\newcommand{\grpG}[2]{G_{#1}^{#2}}
\newcommand{\kUsrs}[4]{k_{#1,#2}^{(#3,#4)}}
\newcommand{\dElmnts}[3]{d_{#1}^{#3}(#2)}
\newcommand{\estML}[0]{\psi_{\text{ML}}}

\newcommand{\numDiffElmnt}[2]{\Lambda (#1,#2)}
\newcommand{\hamDist}[2]{d_{\text{H}} \left(#1,#2\right)}
\newcommand{\diff}{\delta}
\newcommand{\intraTau}[0]{\diff_{{g}}}
\newcommand{\interTau}[0]{\diff_{{c}}}
\newcommand{\numEdge}[3]{e_{#1} \left(#2,#3\right)}
\newcommand{\indicatorFn}[1]{\mathbbm{1} \left[#1\right]}

\newcommand{\Ig}[0]{I_g}
\newcommand{\IcOne}[0]{I_{c1}}
\newcommand{\IcTwo}[0]{I_{c2}}
\newcommand{\Tau}{\mathcal{T}}
\newcommand{\tG}[2]{\tilde{G}_{#1}^{#2}}
\newcommand{\uG}[2]{\overline{G}_{#1}^{#2}}
\newcommand{\E}{\mathbb{E}}
\newcommand{\M}[1]{M_{\langle #1\rangle}}

\newcommand{\vecVhat}[2]{\widehat{v}_{#1}^{#2}}
\newcommand{\vecVElmnt}[3]{v_{#1}^{#2}(#3)}
\newcommand{\vecVhatElmnt}[3]{\widehat{v}_{#1}^{#2}(#3)}
\newcommand{\grpGhat}[3]{\widehat{G}_{#1}^{#2}(#3)}

\newcommand{\deltaDiff}[3]{\rho_{#1,#2}(#3)}
\newcommand{\Yelmnt}[2]{Y_{#1#2}}
\newcommand{\Zelmnt}[2]{Z_{#1#2}}
\newcommand{\misClassUsr}[1]{R_{#1}}
\newcommand{\misClassFrac}[1]{\eta_{#1}}

\newcommand{\estPrmtr}[1]{\widehat{#1}}
\newcommand{\alphaTild}[0]{\widetilde{\alpha}}
\newcommand{\betaTild}[0]{\widetilde{\beta}}
\newcommand{\gammaTild}[0]{\widetilde{\gamma}}

\newcommand{\userPairSet}[2]{\mathcal{P}_{#1}\left(#2\right)}
\newcommand{\alphaEdge}[0]{\alpha}
\newcommand{\betaEdge}[0]{\beta}
\newcommand{\gammaEdge}[0]{\gamma}
\newcommand{\cV}{\mathcal{V}}
\newcommand{\cZ}{\mathcal{Z}}
\newcommand{\tupleSetDelta}[0]{\mathcal{T}^{(\delta)}}

\newcommand{\misClassfSet}[2]{\mathcal{P}_{#1 \rightarrow #2}}
\newcommand{\diffEntriesSet}[0]{\mathcal{P}_{\text{d}}}
\newcommand{\Pleftrightarrow}[2]{P_{#1 \leftrightarrow #2}}
\newcommand{\Ir}{I_r}
\newcommand{\TsmallErr}[0]{\mathcal{T}_{\textrm{small}}^{(\delta)}}
\newcommand{\TlargeErr}[0]{\mathcal{T}_{\textrm{large}}^{(\delta)}}
\newcommand{\dElmntsGen}[4]{d_{#1,#2}^{(#3,#4)}}
\newcommand{\partitionSet}[0]{\mathcal{Z}}

\newcommand{\relabelConst}[0]{\tau}
\newcommand{\numErrColConfig}[0]{\kappa}
\newcommand{\totNumErrCol}[1]{f^{(#1)}}
\newcommand{\numErrColElemnt}[2]{f_{#1}^{(#2)}}
\newcommand{\colVecRhat}[1]{w_{#1}}
\newcommand{\ratingVecSet}[1]{\mathcal{R}^{(#1)}}
\newcommand{\kUsrsHat}[4]{\widehat{k}_{#1,#2}^{(#3,#4)}}
\newcommand{\dElmntsGenHat}[4]{\widehat{d}_{#1, #2}^{\:(#3, #4)}}

\newcommand{\gtRatingCluster}[1]{R_{0}^{(#1)}}
\newcommand{\gtRatingClusterHat}[1]{\widehat{R}^{(#1)}}

\newcommand{\XRatingCluster}[1]{R^{(#1)}}

\newcommand{\TlargeOne}[0]{\mathcal{R}_{1}}
\newcommand{\TlargeTwo}[0]{\mathcal{R}_{2}}

\newcommand{\TlargeOneOne}[0]{\mathcal{R}_{1,1}}
\newcommand{\TlargeOneTwo}[0]{\mathcal{R}_{1,2}}
\newcommand{\TlargeOneThree}[0]{\mathcal{R}_{1,3}}

\DeclareMathOperator*{\argmax}{arg\,max}

\newtheorem{theorem}{Theorem}
\newtheorem{lemma}{Lemma}

\newtheorem{prop}{Proposition}
\newtheorem{remark}{Remark}

\title{Matrix Completion with Hierarchical \\Graph Side Information}

\author{
	Adel Elmahdy\thanks{Equal contribution.} \\
	ECE, University of Minnesota\\
	\texttt{adel@umn.edu} \\
	\And
	Junhyung Ahn\footnotemark[1] \\
	EE, KAIST\\
	\texttt{tonyahn96@kaist.ac.kr } \\
	\AND
	\hspace{6mm}Changho Suh \\
	\hspace{6mm}EE, KAIST\\
	\hspace{6mm}\texttt{chsuh@kaist.ac.kr} \\
	\And
	\hspace{6mm}Soheil Mohajer \\
	\hspace{6mm}ECE, University of Minnesota\\
	\hspace{6mm}\texttt{soheil@umn.edu}
}

\allowdisplaybreaks

\begin{document}
\maketitle

\begin{abstract}
We consider a matrix completion problem that exploits social or item similarity graphs as side information. 
We develop a universal, parameter-free, and computationally efficient algorithm that starts with hierarchical graph clustering and then iteratively refines estimates both on graph clustering and matrix ratings. 
Under a hierarchical stochastic block model that well respects practically-relevant social graphs and a low-rank rating matrix model (to be detailed), we demonstrate that our algorithm achieves the information-theoretic limit on the number of observed matrix entries (i.e., optimal sample complexity) that is derived by maximum likelihood estimation together with a lower-bound impossibility result. 
One consequence of this result is that exploiting the \emph{hierarchical} structure of social graphs yields a substantial gain in sample complexity relative to the one that simply identifies different groups without resorting to the relational structure across them. 
We conduct extensive experiments both on synthetic and real-world datasets to corroborate our theoretical results as well as to demonstrate significant performance improvements over other matrix completion algorithms that leverage graph side information.
\end{abstract}

\section{Introduction}
Recommender systems have been powerful in a widening array of applications for providing users with   relevant items of their potential interest~\cite{koren2009matrix}. A prominent well-known technique for operating the systems is low-rank matrix completion~\cite{nguyen2019low,fazel2002matrix,candes2009exact,candes2010matrix,keshavan2010matrix,candes2010power,cai2010singular,fornasier2011low,mohan2012iterative,lee2010admira,wang2014rank,tanner2016low,wen2012solving,dai2705set,vandereycken2013low,hu2012fast,gotoh2018dc}: Given partially observed entries of an interested matrix, the goal is to predict the values of missing entries. One challenge that arises in the big data era is the so-called \emph{cold start problem} in which high-quality recommendations are not feasible for new users/items that bear little or no information. One natural and popular way to address the challenge is to exploit other available side information. Motivated by the social homophily theory~\cite{mcpherson2001birds} that users within the same community are more likely to share similar preferences, social networks such as Facebook's friendship graph have often been employed to improve the quality of recommendation. 

While there has been a proliferation of social-graph-assisted recommendation algorithms~\cite{koren2009matrix,tang2013social,cai2010graph,jamali2010matrix,li2009relation,ma2011recommender,kalofolias2014matrix,ma2008sorec,ma2009learning,guo2015trustsvd,zhao2017collaborative,chouvardas2017robust,massa2005controversial,golbeck2006filmtrust,jamali2009trustwalker,jamali2009using,yang2012bayesian,yang2012top,monti2017geometric,berg2017graph,rao2015collaborative,zhou2012kernelized}, few works were dedicated to developing theoretical insights on the usefulness of side information, and therefore the maximum gain due to side information has been unknown.
A few recent efforts have been made from an information-theoretic perspective~\cite{ahn2018binary,yoon2018joint,tan2019community,jo2020discrete}. Ahn et~al.~\cite{ahn2018binary} have identified the maximum  gain by characterizing the optimal sample complexity of matrix completion in the presence of graph side information under a simple setting in which there are two clusters and users within each cluster share the same ratings over items. A follow-up work~\cite{yoon2018joint} extended to an arbitrary number of clusters while maintaining the same-rating-vector assumption per user in each cluster. 
While~\cite{ahn2018binary,yoon2018joint} lay out the theoretical foundation for the problem, the assumption of the single rating vector per cluster limits  the practicality of the considered model.

In an effort to make a further progress on theoretical insights, and motivated by \cite{wang2015exploring}, we consider a more generalized setting in which each cluster exhibits another sub-clustering structure, each sub-cluster (or that we call a ``group'') being represented by a different rating vector yet intimately-related to other rating vectors within the same cluster. More specifically, we focus on a hierarchical graph setting wherein users are categorized into two clusters, each of which comprises three groups in which rating vectors are broadly similar yet distinct subject to a linear subspace of two basis vectors.

{\bf Contributions: } Our contributions are two folded. First we characterize the information-theoretic sharp threshold on the minimum number of observed matrix entries required for reliable matrix completion, as a function of the quantified quality (to be detailed) of the considered hierarchical graph side information. The second yet more practically-appealing contribution is to develop a computationally efficient algorithm that achieves the optimal sample complexity for a wide range of scenarios. 
One implication of this result is that our algorithm fully utilizing the \emph{hierarchical graph} structure yields a significant gain in sample complexity, compared to a simple variant of~\cite{ahn2018binary,yoon2018joint} that does not exploit the relational structure across rating vectors of groups. Technical novelty and algorithmic distinctions also come in the process of exploiting the hierarchical structure; see Remarks~\ref{rmrk:groupLimit} and \ref{rmrk:clusterLimited}. Our experiments conducted on both synthetic and real-world datasets corroborate our theoretical results as well as demonstrate the efficacy of our proposed algorithm.

{\bf Related works: } In addition to the initial works~\cite{ahn2018binary,yoon2018joint}, more generalized settings have been taken into consideration with distinct directions. Zhang et~al. \cite{tan2019community} explore a setting in which both social and item similarity graphs are given as side information, thus demonstrating a synergistic effect due to the availability of two graphs. Jo~et~al.~\cite{jo2020discrete} go beyond binary matrix completion to investigate a setting in which a matrix entry, say $(i,j)$-entry, denotes the probability of user $i$ picking up item $j$ as the most preferable, yet chosen from a known finite set of probabilities. 

Recently a so-called \emph{dual} problem has been explored in which clustering is performed with a partially observed matrix as side information~\cite{ashtiani2016clustering,mazumdar2017query}. Ashtiani et~al. \cite{ashtiani2016clustering} demonstrate that the use of side information given in the form of pairwise queries plays a crucial role in making an NP-hard clustering problem tractable via an efficient k-means algorithm. Mazumdar et~al. \cite{mazumdar2017query} characterize the optimal sample complexity of clustering in the presence of similarity matrix side information together with the development of an efficient algorithm. One distinction of our work compared to~\cite{mazumdar2017query} is that we are interested in both clustering and matrix completion, while~\cite{mazumdar2017query} only focused on finding the clusters, from which the rating matrix cannot be necessarily inferred.

Our problem can be viewed as the prominent low-rank matrix completion problem \cite{fazel2002matrix, candes2009exact,keshavan2010matrix,candes2010power, koren2009matrix, cai2010singular, fornasier2011low, mohan2012iterative, lee2010admira, wang2014rank,tanner2016low,wen2012solving,dai2705set,vandereycken2013low,hu2012fast,gotoh2018dc, nguyen2019low} which has been considered notoriously difficult. Even for the simple scenarios such as rank-1 or rank-2 matrix settings, the optimal sample complexity has been open for decades, although some upper and lower bounds are derived. The matrix of our consideration in this work is of rank 4. Hence, in this regard, we could make a progress on this long-standing open problem by exploiting the structural property posed by our considered application.

The statistical model that we consider for theoretical guarantees of our proposed algorithm relies on the Stochastic Block Model (SBM)~\cite{holland1983stochastic} and its hierarchical counterpart~\cite{clauset2008hierarchical,peixoto2014hierarchical,lyzinski2016community,cohen2019hierarchical} which have been shown to well respect many practically-relevant scenarios ~\cite{traud2012social,leskovec2012learning,newman2006modularity,newman2001structure}. Also our algorithm builds in part upon prominent clustering~\cite{abbe2015exact,gao2017achieving} and hierarchical clustering~\cite{lyzinski2016community,cohen2019hierarchical} algorithms, although it exhibits a notable distinction in other matrix-completion-related procedures together with their corresponding technical analyses.

{\bf Notations: } Row vectors and matrices are denoted by lowercase and uppercase letters, respectively. 
Random matrices are denoted by boldface uppercase letters, while their realizations are denoted by uppercase letters.
Sets are denoted by calligraphic letters. Let $\mathbf{0}_{m \times n}$ and $\mathbf{1}_{m \times n}$ be all-zero and all-one matrices of dimension $m \times n$, respectively. For an integer $n \geq 1$, $[n]$ indicates the set of integers $\{1,2,\ldots,n\}$. Let $\{0,1\}^n$ be the set of all binary numbers with $n$ digits. The hamming distance between two binary vectors $u$ and $v$ is denoted by $\hamDist{u}{v} := \| u \oplus v \|_{0}$, where $\oplus$ stands for modulo-2 addition operator. Let $\indicatorFn{\cdot}$ denote the indicator function. For a graph $\mathcal{G}=(V, E)$ and two disjoint subsets $X$ and $Y$ of $V$, $e \left(X,Y\right)$ indicates the number of edges between $X$ and $Y$.

\section{Problem Formulation}
\label{sec:probForm}
{\bf Setting:} Consider a rating matrix with $n$ users and $m$ items. Each user rates $m$ items by a binary vector, where $0/1$ components denote ``dislike''/``like'' respectively.
We assume that there are two clusters of users, say $A$ and $B$. To capture the low-rank of the rating matrix, we assume that each user's rating vector within a cluster lies in a linear subspace of \emph{two} basis vectors. Specifically, let $v_1^A \in \mathbb{F}_2^{1 \times m}$ and $v_2^A \in \mathbb{F}_2^{1 \times m}$ be the two linearly-independent basis vectors of cluster $A$. 
Then users in Cluster $A$ can be  split into three groups (e.g., say $G_1^A$, $G_2^A$ and $G_3^A$) based on their rating vectors. More precisely, we denote by $G_i^A$ the set of users whose rating vector is $v_i^A$ for $i=1,2$. Finally, the remaining users of cluster $A$ from group $G_3^A$, and their rating vector is $v_3^A = v_1^A \oplus v_2^A $ (a linear combination of the basis vectors). Similarly we have $v_1^B, v_2^B$ and $v_3^B=v_1^B \oplus v_2^B$ for cluster $B$. For presentational simplicity, we assume equal-sized groups (each being of size $n/6$), although our algorithm (to be presented in Section~\ref{sec:proposedAlgo}) allows for any group size, and our theoretical guarantees (to be presented in Theorem~\ref{Thm:Alg_theory}) hold as long as the group sizes are order-wise same. Let $M \in \mathbb{F}^{n \times m}$ be a rating matrix wherein the $i^{\text{th}}$ row corresponds to user $i$'s rating vector. 

We find the Hamming distance instrumental in expressing our main results (to be stated in Section~\ref{sec:result}) as well as proving the main theorems. Let $\intraTau$ be the normalized Hamming distance among distinct pairs of \emph{group's} rating vectors within the same cluster: $\intraTau = \frac{1}{m} \min_{c \in \{A,B\}} \min_{\substack{i,j \in [3]}} \hamDist{\vecV{i}{c}}{\vecV{j}{c}}$. Also let $\interTau$ be the counterpart w.r.t. distinct pairs of rating vectors \emph{across different clusters}: $\interTau  = \frac{1}{m}  \min_{\substack{i,j \in [3]}} \hamDist{\vecV{i}{A}}{\vecV{j}{B}}$, and define $\diff := \{ \intraTau, \interTau\}$.
We partition all the possible rating matrices into subsets depending on $\diff$. Let $\matSet$ be the set of rating matrices subject to $\diff$.

{\bf Problem of interest:} Our goal is to estimate a rating matrix $M \in \matSet$ given two types of information: (1) partial ratings $Y \in \{0,1,* \}^{n \times m}$; (2) a graph, say social graph ${\cal G}$. Here $*$ indicates no observation, and we denote the set of observed entries of $Y$ by $\Omega$, that is $\Omega=\{(r,c)\in [n]\times [m]: Y_{rc}\neq *\}$. Below is a list of assumptions made for the analysis of the optimal sample complexity (Theorem 1) and theoretical guarantees of our proposed algorithm (Theorem 2), but not for the algorithm itself. We assume that each element of $Y$ is observed with probability $p \in [0,1]$, independently from others, and its observation can possibly be flipped with probability $\theta \in [0, \frac{1}{2})$.  Let social graph ${\cal G} = ([n],E)$ be an undirected graph, where $E$ denotes the set of edges, each capturing the social connection between two associated users. The set $[n]$ of vertices is partitioned into two disjoint clusters, each being further partitioned into three disjoint groups. We assume that the graph follows the hierarchical stochastic block model (HSBM)~\cite{abbe2017community,lyzinski2016community} with three types of edge probabilities: (i) $\alpha$ indicates an edge probability between two users in the \emph{same group}; (ii) $\beta$ denotes the one w.r.t. two users of \emph{different groups} yet within the \emph{same cluster}; (iii) $\gamma$ is associated with two users of \emph{different clusters}. We focus on realistic scenarios in which users within the same group (or cluster) are more likely to be connected as per the social homophily theory~\cite{mcpherson2001birds}: $\alpha \geq \beta \geq \gamma$.

{\bf Performance metric:} Let $\psi$ be a rating matrix estimator that takes $(Y, \mathcal{G})$ as an input, yielding an estimate. As a performance metric, we consider the worst-case probability of error: 
\begin{align}
	\label{eq:errorProb_wc}
	P_e^{(\diff)} (\psi) := \max_{M \in \matSet} \mathbb{P}\left[\psi(Y, \mathcal{G}) \neq M\right].
\end{align}
Note that $\matSet$ is the set of ground-truth matrices $M$ subject to $\diff := \{ \intraTau, \interTau\}$. Since the error probability may vary depending on different choices of $M$ (i.e., some matrices may be harder to estimate), we employ a conventional minimax approach wherein the goal is to minimize the maximum error probability.
We characterize the optimal sample complexity for reliable exact matrix recovery, concentrated around $nmp^{\star}$ in the limit of $n$ and $m$. Here $p^{\star}$ indicates the sharp threshold on the observation probability: (i) above which the error probability can be made arbitrarily close to 0 in the limit; (ii) under which  $P_e^{(\diff)} (\psi) \nrightarrow 0$ no matter what and whatsoever.

\section{Optimal sample complexity}
\label{sec:result}
We first present the optimal sample complexity characterized under the considered model. We find that an intuitive and insightful expression can be made via the quality of hierarchical social graph, which can be quantified by the following: (i) $\Ig:= (\sqrt{\alpha} - \sqrt{\beta} )^2$ represents the capability of separating distinct \emph{groups} within a cluster; (ii) $\IcOne:= (\sqrt{\alpha} - \sqrt{\gamma} )^2$ and $\IcTwo:= (\sqrt{\beta} - \sqrt{\gamma} )^2$ capture the \emph{clustering capabilities} of the social graph. Note that the larger the quantities, the easier to do grouping/clustering. Our sample complexity result is formally stated below as a function of $(\Ig, \IcOne, \IcTwo)$. As in~\cite{ahn2018binary}, we make the same assumption on $m$ and $n$ that turns out to ease the proof via prominent large deviation theories: $m = \omega(\log n)$ and $\log m = o(n)$. This assumption is also practically relevant as it rules out highly asymmetric matrices.

\begin{theorem}[Information-theoretic limits] 
	\label{Thm:p_star}
	Assume that $m = \omega(\log n)$ and $\log m = o(n)$.
	Let the item ratings be drawn from a finite field $\mathbb{F}_{q}$. Let $c$ and~$g$ denote the number of clusters and groups, respectively. Within each cluster, let the set of $g$ rating vectors be spanned by any $r \leq g$ vectors in the same set. Let $\theta \in [0,(q-1)/q)$. Then, the following holds for any constant $\epsilon >0$: if 
	\begin{align}
		p^{\star} 
		\geq 
		\frac{1}{\left(\sqrt{1\!-\!\theta} \!-\! \sqrt{\frac{\theta}{q\!-\!1}}\right)^2} 
		\max \left\{ 
		\frac{(1\!+\!\epsilon) gc}{g\!-\!r\!+\!1} \frac{\log m }{n}, 
		\frac{(1\!+\!\epsilon) \log n \!-\! \frac{n}{gc} \Ig}{\intraTau m},
		\frac{(1\!+\!\epsilon) \log n \!-\!\frac{n}{gc} \IcOne \!-\! \frac{(g-1)n}{gc} \IcTwo}{\interTau m}
		\right\}\!,
		\label{eq:pstar-general_achv}
	\end{align}
	then there exists an estimator $\psi$ that outputs a rating matrix $X \in \matSet$ given $Y$ and ${\cal G}$ such that $\lim_{n \rightarrow \infty} P_e^{(\diff)} (\psi) = 0$; conversely, if
	\begin{align}
		p^{\star} 
		\leq 
		\frac{1}{\left(\sqrt{1\!-\!\theta} \!-\! \sqrt{\frac{\theta}{q\!-\!1}}\right)^2} 
		\max \left\{ 
		\frac{(1\!-\!\epsilon) gc}{g\!-\!r\!+\!1} \frac{\log m }{n}, 
		\frac{(1\!-\!\epsilon) \log n \!-\! \frac{n}{gc} \Ig}{\intraTau m},
		\frac{(1\!-\!\epsilon) \log n \!-\! \frac{n}{gc} \IcOne \!-\! \frac{(g-1)n}{gc} \IcTwo}{\interTau m}
		\right\}\!,
		\label{eq:pstar-general_conv}
	\end{align}
	then $ \lim_{n \rightarrow \infty} P_e^{(\diff)}(\psi) \neq 0$ for any estimator $\psi$.
	Therefore, the optimal observation probability $p^{\star}$ is given by
	\begin{align}
		p^{\star} 
		=
		\frac{1}{\left(\sqrt{1\!-\!\theta} \!-\! \sqrt{\frac{\theta}{q\!-\!1}}\right)^2} 
		\max \left\{ 
		\frac{gc}{g\!-\!r\!+\!1} \frac{\log m }{n}, 
		\frac{\log n- \frac{n}{gc} \Ig}{\intraTau m},
		\frac{\log n-\frac{n}{gc} \IcOne - \frac{(g-1)n}{gc} \IcTwo}{\interTau m}
		\right\}\!.
		\label{eq:pstar-general}
	\end{align} 
	Setting $(c,g,r,q)=(2,3,2,2)$, the bound in~\eqref{eq:pstar-general} reduces to 
	\begin{align}
		\label{eq:pstar}
		p^{\star} = \frac{1}{ (\sqrt{1- \theta} - \sqrt{\theta})^2  } \max \left \{ 
		\frac{3 \log m}{n}, \frac{\log n - \frac{1}{6} n \Ig }{ m \intraTau}, \frac{\log n - \frac{1}{6} n \IcOne - \frac{1}{3} n \IcTwo }{m \interTau} 
		\right \},
	\end{align}
	which is the optimal sample complexity of the problem formulated in Section~\ref{sec:probForm}.
\end{theorem}
\begin{proof}
	We provide a proof sketch for $(c,g,r)=(2,3,2)$. We defer the complete proof of Theorem~\ref{Thm:p_star} for $(c,g,r)=(2,3,2)$ to the supplementary material. The extension to general $(c,g,r)$ is a natural generalization of the analysis for the parameters $(c,g,r)=(2,3,2)$. We refer the interested reader to \cite{ahn2021fundamental} for the complete proof of Theorem~\ref{Thm:p_star} for general $(c,g,r)$.
	The achievability proof is based on maximum likelihood estimation (MLE). We first evaluate the likelihood for a given clustering/grouping of users and the corresponding rating matrix. We then show that if $p \geq (1+\epsilon)p^{\star}$, the likelihood is maximized only by the ground-truth rating matrix in the limit of $n$: $\lim_{n \rightarrow \infty} P_e^{(\diff)} (\psi_{\sf ML}) = 0$. 
	For the converse (impossibility) proof, we first establish a lower bound on the error probability, and show that it is minimized when employing the maximum likelihood estimator. Next we prove that if $p$ is smaller than any of the three terms in the RHS of~\eqref{eq:pstar}, then there exists another solution that yields a larger likelihood, compared to the ground-truth matrix. More precisely, if $p \leq \frac{(1-\epsilon) 3 \log m}{ (\sqrt{1-\theta}- \sqrt{\theta})^2 n}$, we can find a grouping with the only distinction in two user-item pairs relative to the ground truth, yet yielding a larger likelihood. Similarly when $p \leq \frac{(1-\epsilon) (\log n - \frac{1}{6} n \Ig ) }{ (\sqrt{1-\theta}- \sqrt{\theta})^2 m \intraTau}$, consider two users in the same cluster yet from distinct groups such that the hamming distance between their rating vectors is~$m \intraTau$. We can then show that a grouping in which their rating vectors are swapped provides a larger likelihood. Similarly when $p \leq \frac{(1-\epsilon) (\log n - \frac{1}{6} n \IcOne - \frac{1}{3} n \IcTwo ) }{ (\sqrt{1-\theta}- \sqrt{\theta})^2 m \interTau} $, we can swap the rating vectors of two users from different clusters with a hamming distance of $m \interTau$, and get a greater likelihood. 
	
	The technical distinctions w.r.t. the prior works~\cite{ahn2018binary,yoon2018joint} are three folded: (i) the likelihood computation requires more involved combinatorial arguments due to the hierarchical structure; (ii) sophisticated upper/lower bounding techniques are developed in order to exploit the relational structure across different groups; (iii) delicate choices are made for two users to be swapped in the converse proof.
\end{proof}

We next present the second yet more practically-appealing contribution: Our proposed algorithm in Section~\ref{sec:proposedAlgo} achieves the information-theoretic limits. The algorithm optimality is guaranteed for a certain yet wide range of scenarios in which graph information yields negligible clustering/grouping errors, formally stated below. We provide the proof outline in Section~\ref{sec:proposedAlgo} throughout the description of the algorithm, leaving details in the supplementary material.  

\begin{theorem}[Theoretical guarantees of the proposed algorithm]
	\label{Thm:Alg_theory}
	Assume that $m = \omega(\log n)$, $\log m = o(n)$,  $m = O(n)$, $\IcTwo > \frac{2 \log n}{n}$ and $\Ig > \omega (\frac{1}{n})$. Then, as long as the sample size is beyond the optimal sample complexity in Theorem~\ref{Thm:p_star} (i.e., $mnp > mnp^{\star}$), then the algorithm presented in Section~\ref{sec:proposedAlgo} with $T = O(\log n)$ iterations ensures the worse-case error probability tends to $0$ as $n \rightarrow \infty$. That is, the algorithm returns $\widehat{M}$ such that  $\mathbb{P}[\widehat{M} = M] = 1 - o(1)$.
\end{theorem}

Theorem~\ref{Thm:p_star} establishes the optimal sample complexity (the number of entries of the rating matrix to be observed) to be $mnp^\star$, where $p^\star$ is given in \eqref{eq:pstar}. The required sample complexity is a non-increasing function of $\intraTau$ and  $\interTau$. This makes an intuitive sense because increasing $\intraTau$ (or $\interTau$) yields more distinct rating vectors, thus ensuring easier grouping (or clustering). We emphasize three regimes depending on $(\Ig, \IcOne, \IcTwo)$. The first refers to the so-called \emph{perfect clustering/grouping regime} in which $(\Ig, \IcOne, \IcTwo)$ are large enough, thereby activating the $1^{\text{st}}$ term in the $\max$ function. The second is the \emph{grouping-limited regime}, in which the quantity $\Ig$ is not large enough so that the $2^{\text{nd}}$ term becomes dominant. The last is the \emph{clustering-limited regime} where the $3^{\text{rd}}$ term is activated.  A few observations are in order. For illustrative simplicity, we focus on the noiseless case, i.e., $\theta = 0$. 

\begin{remark}[Perfect clustering/grouping regime]
	\label{rmrk:perfectCluster}
	The optimal sample complexity reads $3m \log m$. This result is interesting. A naive generalization  of~\cite{ahn2018binary, yoon2018joint} requires $4m \log m$, as we have four rating vectors $(v_1^A, v_2^A, v_1^B, v_2^B)$ to estimate and each requires $m \log m$ observations under our random sampling, due to the coupon-collecting effect. On the other hand, we exploit the relational structure across rating vectors of different group, reflected in $v_3^A = v_1^A \oplus v_2^A$ and $v_3^B = v_1^B \oplus v_2^B$; and we find this serves to estimate $(v_1^A, v_2^A, v_1^B, v_2^B)$ more efficiently, precisely by a factor of $\frac{4}{3}$ improvement, thus yielding $3m \log m$. This exploitation is reflected as novel technical contributions in the converse proof, as well as the achievability proofs of MLE and the proposed algorithm. $\hfill \blacksquare$
\end{remark}

\begin{remark}[Grouping-limited regime] 
	\label{rmrk:groupLimit}
	We find that the sample complexity $\frac{n \log n - \frac{1}{6} n^2 \Ig }{ \intraTau}$ in this regime coincides with that of~\cite{yoon2018joint}. This implies that exploiting the relational structure across different groups does not help improving sample complexity when grouping information is not reliable. $\hfill \blacksquare$
\end{remark}

\begin{remark}[Clustering-limited regime] 
	\label{rmrk:clusterLimited}
	This is the most challenging scenario which has not been explored by any prior works. The challenge is actually reflected in the complicated sample complexity formula: $\frac{ n \log n - \frac{1}{6} n^2 \IcOne - \frac{1}{3} n^2 \IcTwo }{\interTau}$. When $\beta = \gamma$, i.e., groups and clusters are not distinguishable, $\Ig = \IcOne$ and $\IcTwo = 0$. Therefore, in this case, it indeed reduces to a 6-group setting: $\frac{ n \log n - \frac{1}{6} n^2 I_{g} }{\interTau}$. The only distinction appears in the denominator. We read $\interTau$ instead of $\intraTau$ due to different rating vectors across clusters and groups. When $\IcTwo \neq 0$, it reads the complicated formula, reflecting non-trivial technical contribution as well. $\hfill\blacksquare$
\end{remark}

\begin{figure}
	\centering
	\subfloat[$\intraTau = \frac{1}{3}$ and $\interTau = \frac{1}{6}$.]{\includegraphics[width=0.33\textwidth]{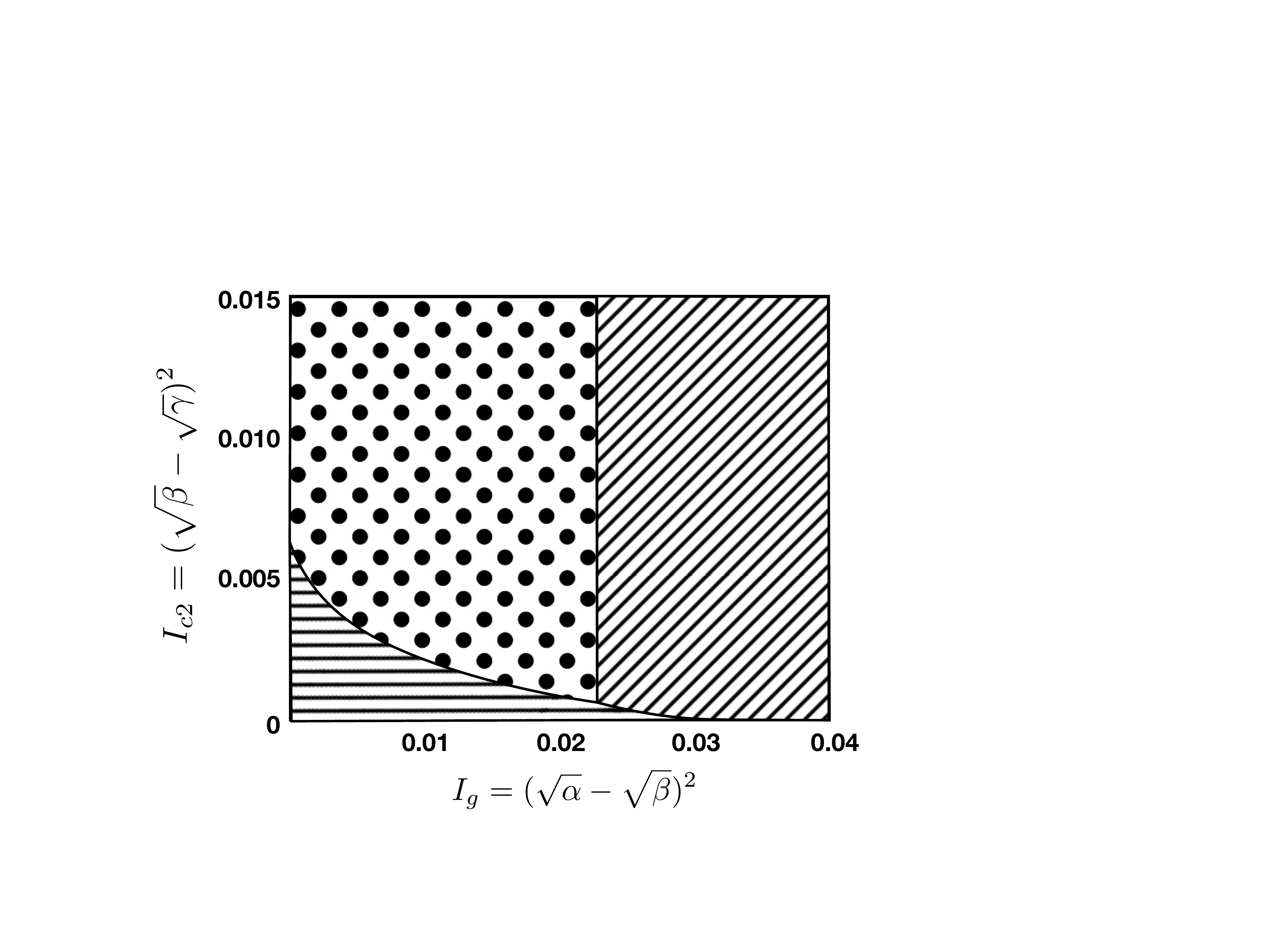}
		\label{fig:3reg}}
	\hfill
	\subfloat[$\intraTau = \frac{1}{7}$ and $\interTau = \frac{1}{6}$.]{\includegraphics[width=0.33\textwidth]{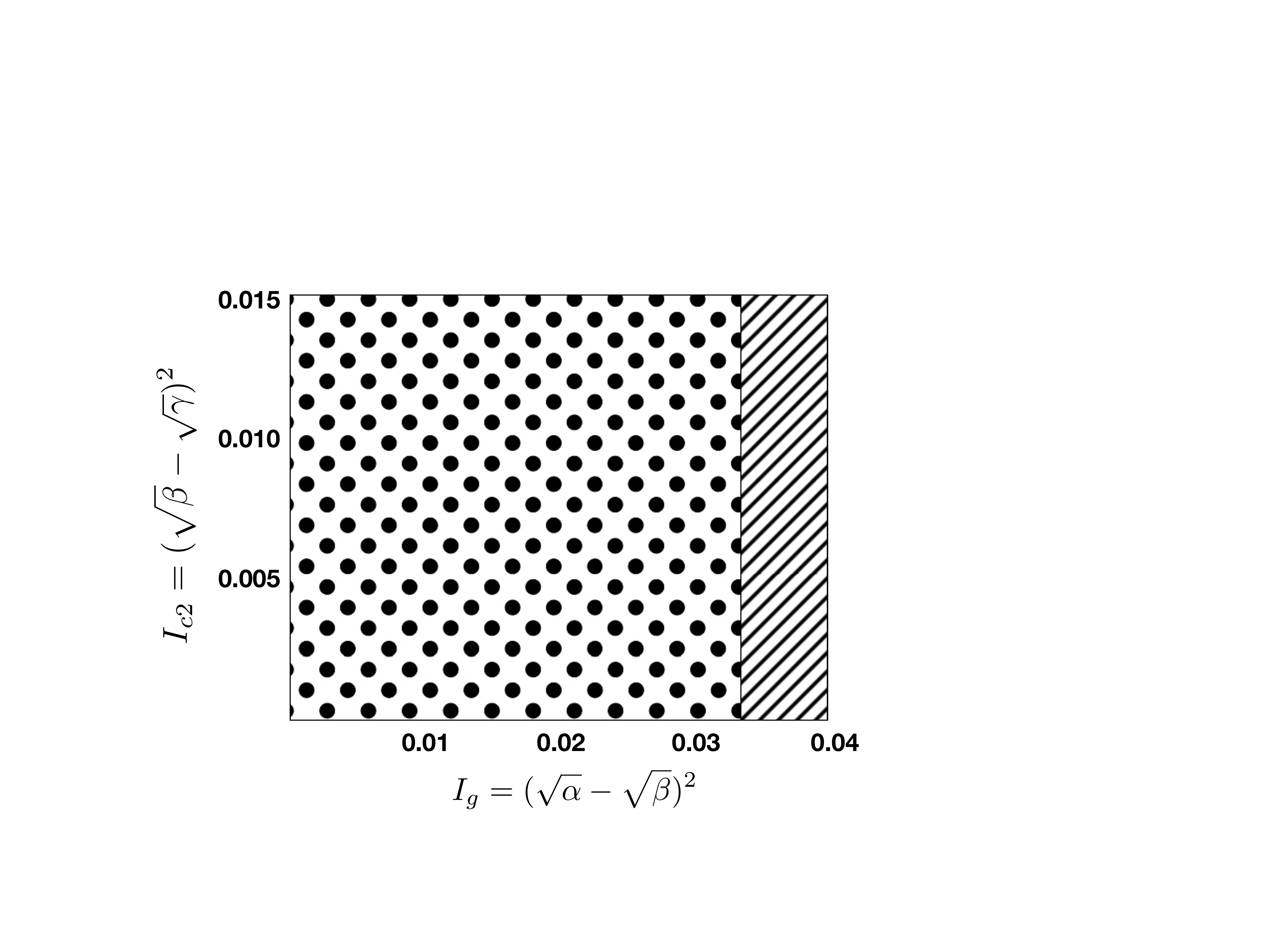}
		\label{fig:2reg}}
	\hfill
	\subfloat[$\intraTau = \frac{1}{3}$ and $\interTau = \frac{1}{6}$.]{\includegraphics[width=0.31\textwidth]{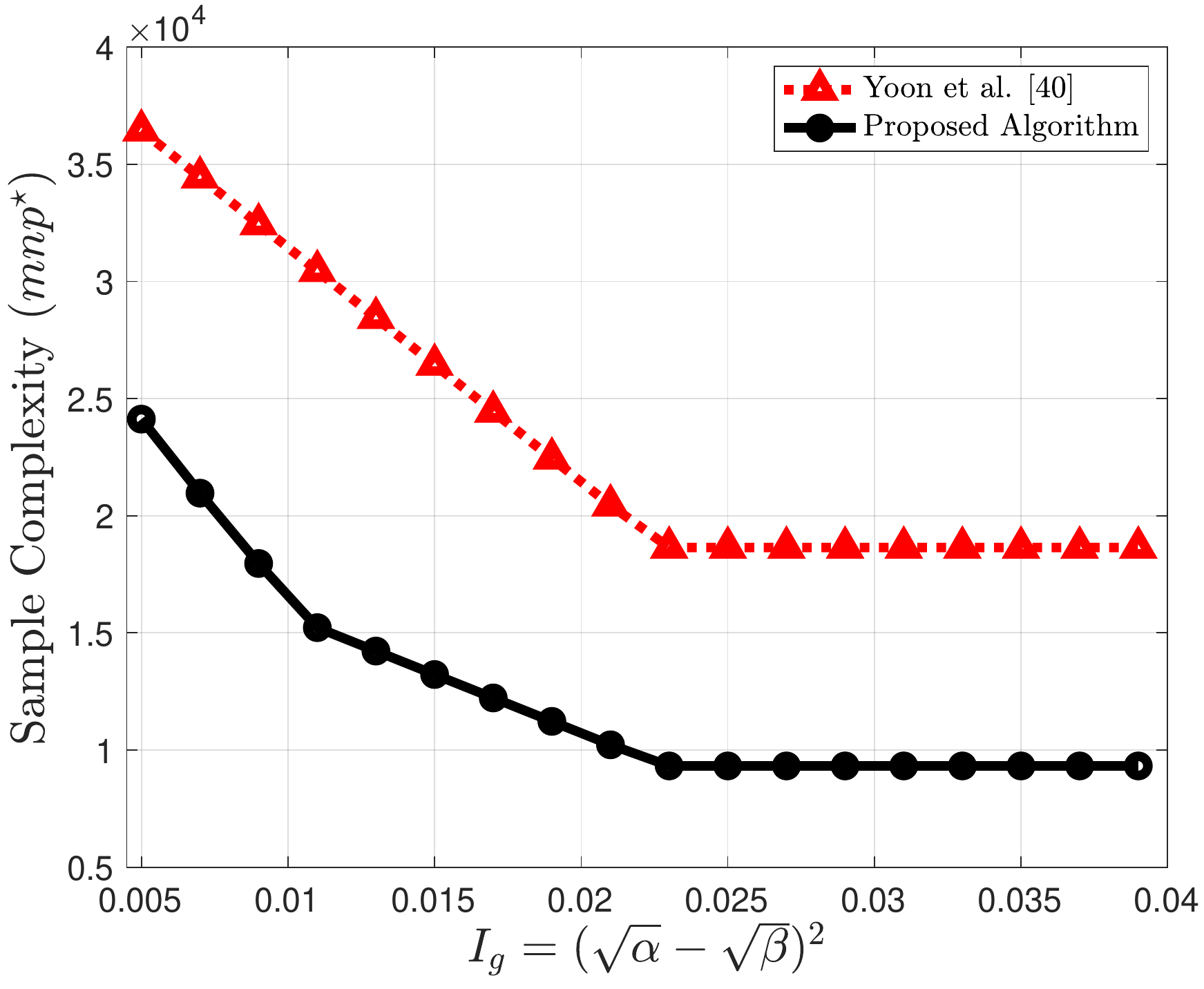}
		\label{fig:baselineComp}}
	\caption{Let $n=1000$, $m=500$ and $\theta=0$. (a), (b) The different regimes of the optimal sample complexity reported in \eqref{eq:pstar}, where in (a) $\interTau < \intraTau$ and in (b) $\interTau > \intraTau$.
	Diagonal stripes, dots, and horizontal stripes refer to perfect clustering/grouping, grouping-limited, and clustering-limited regimes, respectively.
	(c) Comparison between the sample complexity reported in \eqref{eq:pstar} for $\gamma = 0.01$ and $\IcTwo = 0.002$ and that of \cite{yoon2018joint}.
	}
	\label{fig:regions}
	\vspace{-3mm}
\end{figure}

Fig.~\ref{fig:regions} depicts the different regimes of the optimal sample complexity as a function of $(\Ig,\IcTwo)$ for $n=1000$, $m=500$ and $\theta=0$. In Fig.~\ref{fig:3reg}, where $\intraTau = \frac{1}{3}$ and $\interTau = \frac{1}{6}$, the region depicted by diagonal stripes corresponds to the perfect clustering/grouping regime. Here, $\Ig$ and $\IcTwo$ are large, and graph information is rich enough to perfectly retrieve the clusters and groups. In this regime, the $1^{\text{st}}$ term in \eqref{eq:pstar} dominates. The region shown by dots corresponds to grouping-limited regime, where the $2^{\text{nd}}$ term in \eqref{eq:pstar} is dominant. In this regime, graph information suffices to exactly recover the clusters, but we need to rely on rating observation to exactly recover the groups. Finally, the $3^{\text{rd}}$ term in \eqref{eq:pstar} dominates in the region captured by horizontal stripes. This indicates the clustering-limited regimes, where neither clustering nor grouping is exact without the side information of the rating vectors. It is worth noting that in practically-relevant systems, where $\interTau > \intraTau$ (for rating vectors of users in the same cluster are expected to be more similar compared to those in a different cluster), the third regime vanishes, as shown by Fig.~\ref{fig:2reg}, where $\intraTau = \frac{1}{7}$ and $\interTau = \frac{1}{6}$.  It is straightforward to show that the third term in \eqref{eq:pstar} is inactive whenever $\interTau > \intraTau$.
Fig.~\ref{fig:baselineComp} compares the optimal sample complexity between the one reported in~\eqref{eq:pstar}, as a function of $I_g$, and that of \cite{yoon2018joint}. The considered setting is $n \!=\! 1000$, $m\!=\!500$, $\theta\!=\!0$, $\intraTau \!=\! \frac{1}{3}$, $\interTau \!=\! \frac{1}{6}$, $\gamma \!=\! 0.01$ and $\IcTwo \!=\! 0.002$. Note that \cite{yoon2018joint} leverages neither the hierarchical structure of the graph, nor the linear dependency among the rating vectors. Thus, the problem formulated in Section~\ref{sec:probForm} will be translated to a graph with six clusters with linearly independent rating vectors in the setting of  \cite{yoon2018joint}. Also, the minimum hamming distance for \cite{yoon2018joint} is $\interTau$. In Fig.~\ref{fig:baselineComp}, we can see that the noticeable gain in the sample complexity of our result in the diagonal parts of the plot (the two regimes on the left side) is due to leveraging the hierarchical graph structure, while the improvement in the sample complexity in the flat part of the plot is a consequence of exploiting the linear dependency among the rating vectors within each cluster (See Remark~\ref{rmrk:perfectCluster}).

\section{Proposed Algorithm}
\label{sec:proposedAlgo}
We propose a computationally feasible matrix completion algorithm that achieves the optimal sample complexity characterized by Theorem~\ref{Thm:p_star}. 
The proposed algorithm is motivated by a line of research on iterative algorithms that solve non-convex optimization problems \cite{keshavan2010matrix,jain2013low,yun2014accurate,abbe2015community,chen2016community,gao2017achieving,chen2015spectral,netrapalli2013phase,candes2015phase,yi2016fast,balakrishnan2017statistical,wu2015clustering,shah2017reducing}. The idea is to first find a good initial estimate, and then successively refine this estimate until the optimal solution is reached. This approach has been employed in several problems such as matrix completion \cite{keshavan2010matrix,jain2013low}, community recovery \cite{yun2014accurate,abbe2015community,chen2016community,gao2017achieving}, rank aggregation \cite{chen2015spectral}, phase retrieval \cite{netrapalli2013phase,candes2015phase}, robust PCA \cite{yi2016fast}, EM-algorithm \cite{balakrishnan2017statistical}, and rating estimation in crowdsourcing \cite{wu2015clustering,shah2017reducing}. In the following, we  describe the proposed algorithm that consists of four phases to recover clusters, groups and rating vectors. Then, we discuss the computational complexity of the algorithm.

Recall that $Y \in \left\{0,+1,*\right\}^{n \times m}$. For the sake of tractable analysis, it is convenient to map $Y$ to $Z \in \left\{-1,0,+1\right\}^{n \times m}$ where the mapping of the alphabet of $Y$ is as follows: $0 \longleftrightarrow +1$, $+1 \longleftrightarrow -1$ and $* \longleftrightarrow 0$. Under this mapping, the modulo-2 addition over $\{0, 1\}$ in $Y$ is represented by the multiplication of integers over $\{+1, -1\}$ in $Z$. Also, note that all recovery guarantees are asymptotic, i.e., they are characterized with high probability as $n \rightarrow \infty$. Throughout the design and analysis of the proposed algorithm, the number and size of clusters and groups are assumed to be known. 

\subsection{Algorithm Description}
\textbf{Phase~1 (Exact Recovery of Clusters):}
We use the community detection algorithm in \cite{abbe2015exact} on $\mathcal{G}$ to  \emph{exactly} recover the two clusters $A$ and $B$. As proved in \cite{abbe2015exact}, the decomposition of the graph into two clusters is correct  with high probability when $\IcTwo > \frac{2 \log n}{n}$. 

\textbf{Phase~2 (Almost Exact Recovery of Groups):}
The goal of  Phase~$2$ is to decompose the set of users in cluster $A$ (cluster $B$) into three groups, namely $\grpG{1}{A}$, $\grpG{2}{A}$, $\grpG{3}{A}$ (or $\grpG{1}{B}$, $\grpG{2}{B}$, $\grpG{3}{B}$ for cluster $B$). It is worth noting that grouping at this stage is \emph{almost exact}, and will be further refined in the next phases.  To this end, we run a spectral clustering algorithm \cite{gao2017achieving} on $A$ and $B$ separately. Let $\grpGhat{i}{x}{0}$ denote the initial estimate of the $i^{\text{th}}$ group of cluster $x$ that is recovered by Phase~$2$ algorithm, for $i\in[3]$ and $x\in\{A,B\}$. It is shown that the groups within each cluster are recovered with a vanishing fraction of error if $\Ig = \omega(1/n)$.
It is worth mentioning that there are other clustering algorithms \cite{shi2000normalized,ng2002spectral,abbe2015community,chin2015stochastic,lei2015consistency,javanmard2016phase,krzakala2013spectral,mossel2016density} that can be employed for this phase. Examples include: spectral clustering \cite{shi2000normalized,ng2002spectral,abbe2015community,chin2015stochastic,lei2015consistency}, semidefinite programming (SDP) \cite{javanmard2016phase}, non-backtracking matrix spectrum \cite{krzakala2013spectral}, and belief propagation \cite{mossel2016density}.

\textbf{Phase~3 (Exact Recovery of Rating Vectors)}: 
We propose a novel algorithm that optimally recovers the rating vectors of the groups within each cluster. The algorithm is based on maximum likelihood~(ML) decoding of users' ratings based on the partial and noisy observations. For this model, the ML decoding boils down to a counting rule: for each item, find the group with maximum gap between the number of observed zeros and ones, and set the rating entry of this group to $0$. The other two rating vectors are either both $0$ or both $1$ for this item, which will be determined based on the majority of the union of their observed entries. It turns out that the vector recovery is exact with probability $1\!-\!o(1)$. This is one of the technical distinctions, relative to the prior works~\cite{ahn2018binary,yoon2018joint} which employ the simple majority voting rule under non-hierarchical SBMs.

Define $\vecVhat{i}{x}$ as the estimated rating vector of $\vecV{i}{x}$, i.e., the output of Phase~$3$ algorithm. 
Let the $c^{\text{th}}$ element of the rating vector $\vecV{i}{x}$ (or $\vecVhat{i}{x}$) be denoted by $\vecVElmnt{i}{x}{c}$ (or $\vecVhatElmnt{i}{x}{c}$), for $i \in [3]$, $x \in \{ A, B\}$ and $c \in [m]$. Let $Y_{r,c}$ be the entry of matrix $Y$ at row $r$ and column $c$, and $Z_{r,c}$ be its mapping to $\{+1,0,-1\}$. The pseudocode of Phase~$3$ algorithm is given by Algorithm~\ref{algo:phase3}. 

\textbf{Phase~4 (Exact Recovery of Groups):}
Finally, the goal is to \emph{refine} the groups which are \emph{almost recovered} in Phase~$2$, to obtain an \emph{exact} grouping. To this end, we propose an iterative algorithm that locally refines the estimates on the user grouping within each cluster for $T$ iterations. 
Specifically, at each iteration, the affiliation of each user is updated to the group that yields the maximum local likelihood. This is determined based on (i) the number of edges between the user and the set of users which belong to that group, and (ii) the number of observed rating matrix entries of the user that coincide with the corresponding entries of the rating vector of that group.
Algorithm~\ref{algo:phase4_groups} describes the pseudocode of Phase~$4$ algorithm. Note that we do not assume the knowledge of the model parameters $\alpha$, $\beta$ and $\theta$, and estimate them using $Y$ and $\mathcal{G}$, i.e., the proposed algorithm is parameter-free.

In order to prove the exact recovery of groups after running Algorithm~\ref{algo:phase4_groups}, we need to show that the number of misclassified users in each cluster strictly decreases with each iteration of Algorithm~\ref{algo:phase4_groups}. More specifically, assuming that the previous phases are executed successfully, if we start with $\eta n$ misclassified users within one cluster, for some small $\eta > 0$, then one can show that we end up with $\frac{\eta}{2} n$ misclassified users with high probability as $n \rightarrow \infty$ after one iteration of refinement. Hence, running the local refinement for $T = \frac{\log (\eta n)}{\log 2}$ within the groups of each cluster would suffice to converge to the ground truth assignments.
The analysis of this phase follows the one in \cite[Theorem~2]{yoon2018joint} in which the problem of recovering $K$ communities of possibly different sizes is studied. By considering the case of three equal-sized communities, the guarantees of exact recovery of the groups within each cluster readily follows when $T = O(\log n)$. 

\begin{algorithm}[t] 
	\caption{Exact Recovery of Rating Vectors}
	\label{algo:phase3}
	\begin{algorithmic}[1]
		\Function{VecRcv\:}{$n, m, Z, \{\grpGhat{i}{x}{0}: i \in [3], x \in \{A, B\}\}$}
		\For{$c\in[m]$ and $x\in \{A,B\}$}
		\State\algorithmicfor\ {$i \in [3]$}\ \algorithmicdo\ $\deltaDiff{i}{x}{c} \gets \sum_{r\in \grpGhat{i}{x}{0}} Z_{r,c}$ 
		\State $j \gets \argmax_{i \in [3]} \deltaDiff{i}{x}{c}$
		\State $\vecVhatElmnt{j}{x}{c} \gets 0$
		\If {$\sum_{i\in[3]\setminus\{j\}} \deltaDiff{i}{x}{c} \geq 0$}
		\State\algorithmicfor\ {$i \in [3]\setminus\{j\}$}\ \algorithmicdo\ {$\vecVhatElmnt{i}{x}{c} \gets 0$}
		\Else
		\State\algorithmicfor\ {$i \in [3]\setminus\{j\}$}\ \algorithmicdo\ {$\vecVhatElmnt{i}{x}{c} \gets 1$}
		\EndIf
		\EndFor
		\State \Return $\{\vecVhat{i}{x}: i \in [3], x \in \{A, B\}\}$
		\EndFunction
	\end{algorithmic}
\end{algorithm}

\begin{algorithm}[t] 
	\caption{Local Iterative Refinement of Groups (Set $flag = 0$)}
	\label{algo:phase4_groups}
	\begin{algorithmic}[1]
		\Function{Refine\:}{$flag, n, m, T, Y, Z, \mathcal{G}, \{(\grpGhat{i}{x}{0}, \vecVhat{i}{x}): i \in [3], x \in \{ A,  B\}\}$}
		\State $\estPrmtr{\alpha} \gets \frac{1}{6 \binom{n/6}{2}} \left| \left\{(f,g) \in E: f,g \in \grpG{i}{x}, x \in \{A,B\}, i\in[3] \right\} \right|$ 
		\State $\estPrmtr{\beta} \gets \frac{6}{n^2} \left| \left\{(f,g) \in E: f \in \grpG{i}{x}, g \in \grpG{j}{x}, x \in \{A,B\}, i\in[3], j\in[3]\setminus i\right\} \right|$
		\State $\estPrmtr{\theta} \gets |\{(r,c) \in \Omega: \Yelmnt{r}{c} \neq 
		\vecVhatElmnt{i}{x}{c}, 
		r\in \grpGhat{i}{x}{0}\}|/ |\Omega|$
		\For{$t \in[T]$ and $x\in\{A,B\}$} 
		\State\algorithmicfor\ 
		{$i\in [3]$} \algorithmicdo\ {$\grpGhat{i }{x}{t} \gets \varnothing$}
		\For{$r \gets 1$ {\bfseries to} $n$}
		\State $j \gets \argmax_{i \in [3]}  |\{c\!:Y_{r,c}=\vecVhat{i}{x}(c)\}| \cdot \log \left(\frac{1\!-\!\estPrmtr{\theta}}{\estPrmtr{\theta}}\right) 
		+ e \left(\!\{r\}\!, \grpGhat{i}{x}{t-1}\!\right) \cdot\log\! \left(\frac{(1\!-\!\estPrmtr{\beta}) \estPrmtr{\alpha}}{(1\!-\!\estPrmtr{\alpha})
			\estPrmtr{\beta}}\right)$
		\State $\grpGhat{j}{x}{t} \gets \grpGhat{j}{x}{t} \cup \{r\}$
		\If{$flag == 1$} 
		\State $\{\vecVhat{i}{x}\!: i \!\in\! [3], x \!\in\! \{ A,  B\}\} \gets$ {\scshape VecRcv\:}($n, m, Z, \{\grpGhat{i}{x}{t}\!: i \!\in\! [3], x \!\in\! \{A,  B\}\}$)
		\EndIf
		\EndFor
		\EndFor
		\State \Return $\{\grpGhat{i}{x}{T}: i \in [3], x \in \{ A,  B\}\}, \{\vecVhat{i}{x}: i \in [3], x \in \{ A,  B\}\}$
		\EndFunction
	\end{algorithmic}
\end{algorithm}

\begin{remark}
	\label{rmrk:2algo}
	The iterative refinement in  Algorithm~\ref{algo:phase4_groups} can be applied only on the groups  (when $flag = 0$), or on the groups as well as the rating vectors (for $flag = 1$). Even though the former is sufficient for reliable estimation of the rating matrix, we show, through our simulation results in the following section, that the latter achieves a better performance for finite regimes of $n$ and $m$.
	$\hfill\blacksquare$
\end{remark}

\begin{remark}
	\label{rmrk:XOR}
	The problem is formulated under the finite-field model only for the purpose of making an initial step towards a more generalized and realistic algorithm. Fortunately, as many of the theory-inspiring works do, the theory process of characterizing the optimal sample complexity under this model could also shed insights into developing a universal algorithm that is applicable to a general problem setting rather than the specific problem setting considered for the theoretical analysis, as long as some slight algorithmic modifications are made. 
	To demonstrate the universality of the algorithm, we consider a practical scenario in which ratings are real-valued (for which linear dependency between rating vectors is well-accepted) and observation noise is Gaussian. In this setting, the \emph{detection} problem (under the current model) will be replaced by an \emph{estimation} problem. Consequently, we update Algorithm~\ref{algo:phase3} to incorporate an MLE of the rating vectors; and modify the local refinement criterion on Line 8 in Algorithm~\ref{algo:phase4_groups} to find the group that minimizes some properly-defined distance metric between the observed and estimated ratings such as Root Mean Squared Error (RMSE). In Section~\ref{sec:expResults}, we conduct experiments under the aforementioned setting, and show that our algorithm achieves superior performance over the state-of-the-art algorithms.
	$\hfill\blacksquare$
\end{remark}

\subsection{Computational Complexity}
One of the crucial aspects of the proposed algorithm is its computational efficiency. Phase~$1$ can be done in polynomial time in the number of vertices \cite{abbe2015exact,massoulie2014community}. Phase~$2$ can be done in $O(\left|E\right| \log n)$ using the power method \cite{boutsidis2015spectral}. Phase~$3$ requires a single pass over all entries of the observed matrix, which corresponds to $O(\left|\Omega\right|)$. Finally, in each iteration of Phase~$4$, the affiliation update of user~$r \in [n]$ requires reading the entries of the $r^{\text{th}}$ row of $Y$ and the edges connected to user~$r$, which amounts to $O(\left|\Omega\right| + \left|E\right|)$ for each of the $T$ iterations, assuming an appropriate data structure. Hence, the overall computational complexity reads $\mathsf{poly}(n) + O( \left|\Omega\right| \log n)$.

\begin{remark}
	The complexity bottleneck is in Phase~1 (exact clustering), as it relies upon \cite{abbe2015exact,massoulie2014community}, exhibiting $\mathsf{poly}(n)$ runtime. 
	This can be improved, without any performance degradation, by replacing the \emph{exact} clustering in Phase~1 with \emph{almost exact} clustering, yielding $O(\vert E \vert \log n)$ runtime \cite{boutsidis2015spectral}. 
	In return, Phase~4 should be modified so that the local iterative refinement is applied on cluster affiliation, as well as group affiliation and rating vectors. As a result, the improved overall runtime reads $O((\left|\Omega\right| + \left|E\right|)\log n)$. 
	$\hfill\blacksquare$
\end{remark}

\section{Experimental Results}
\label{sec:expResults}
We first conduct Monte Carlo experiments to corroborate Theorem~\ref{Thm:p_star}. Let $\alpha = \alphaTild \frac{\log n}{n}$, $\beta = \betaTild \frac{\log n}{n}$, and $\gamma = \gammaTild \frac{\log n}{n}$. We consider a setting where $\theta=0.1$, $\betaTild=10$, $\gammaTild=0.5$, $\intraTau=\interTau=0.5$. The synthetic data is generated as per the model in Section~\ref{sec:probForm}.
In Figs. \ref{fig:synth_alph40} and \ref{fig:synth_alph17}, we evaluate the performance of the proposed algorithm (with local iterative refinement of groups and rating vectors), and quantify the empirical success rate as a function of the normalized sample complexity, over $10^3$~randomly drawn realizations of rating vectors and hierarchical graphs. We vary $n$ and $m$, preserving the ratio $n/m=3$. 
Fig.~\ref{fig:synth_alph40} depicts the case of $\alphaTild = 40$ which corresponds to perfect clustering/grouping regime (Remark~\ref{rmrk:perfectCluster}). On the other hand, Fig.~\ref{fig:synth_alph17} depicts the case of $\alphaTild = 17$ which corresponds to grouping-limited regime (Remark~\ref{rmrk:groupLimit}). In both figures, we observe a phase transition\footnote{The transition is ideally a step function at $p = p^\star$ as $n$ and $m$ tend to infinity.} in the success rate at $p = p^\star$, and as we increase $n$ and $m$, the phase transition gets sharper. These figures corroborate Theorem~\ref{Thm:p_star} in different regimes when the graph side information is not scarce.
Fig.~\ref{fig:synth_algoComp} compares the performance of the proposed algorithm for $n=3000$ and $m=1000$ under two different strategies of local iterative refinement: (i) local refinement of groups only (set $flag=0$ in Algorithm~\ref{algo:phase4_groups}); and (ii) local refinement of both groups and rating vectors (set $flag=1$ in Algorithm~\ref{algo:phase4_groups}). It is clear that the second strategy outperforms the first in the finite regime of $n$ and $m$, which is consistent with Remark~\ref{rmrk:2algo}. Furthermore, the gap between the two versions shrinks as we gradually increase $\alphaTild$ (i.e., as the quality of the graph gradually improves).

\begin{figure*}[t]
	\centering
	\subfloat[]{\includegraphics[width=0.32\textwidth]{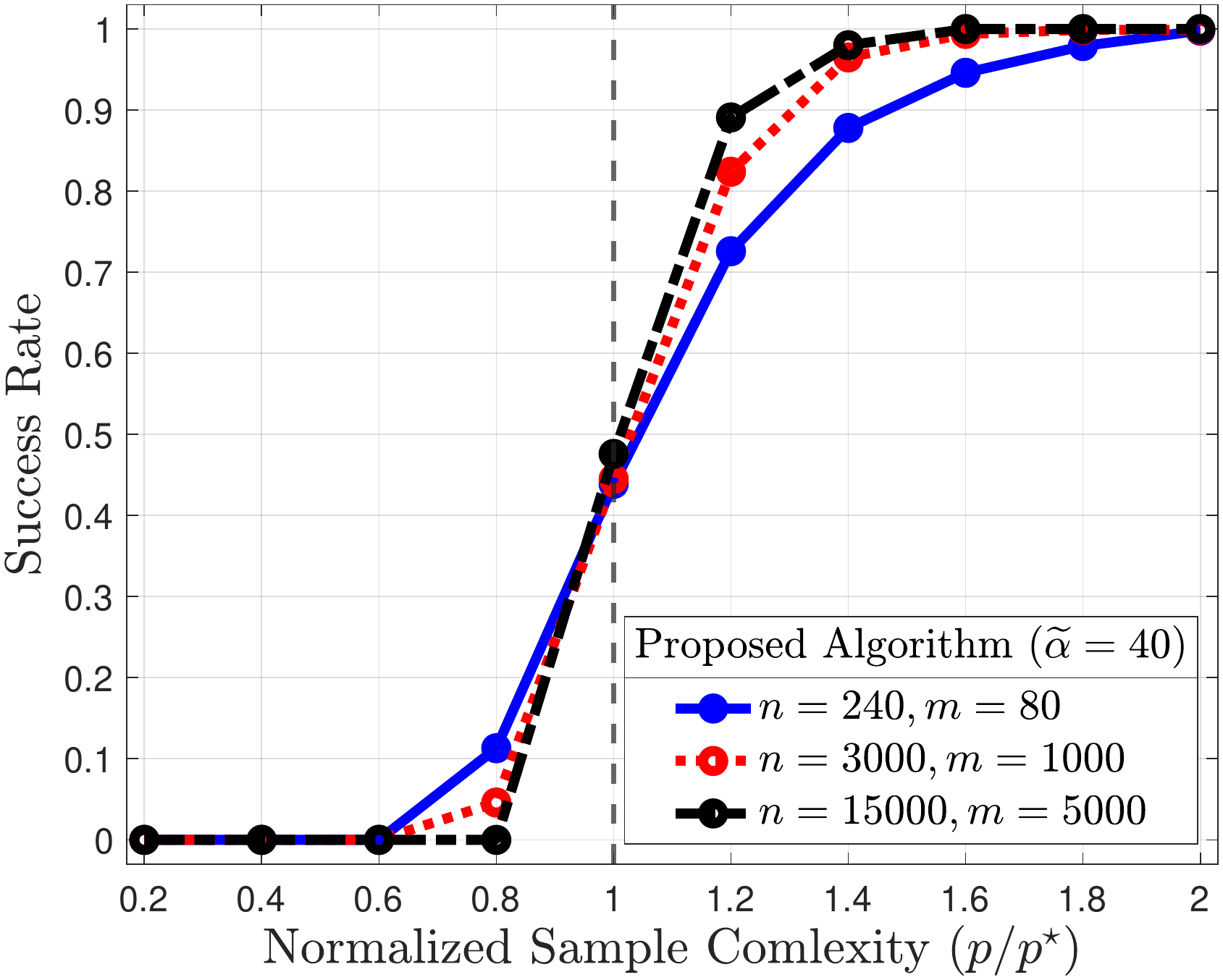}
		\label{fig:synth_alph40}}
	\hfill
	\subfloat[]{\includegraphics[width=0.32\textwidth]{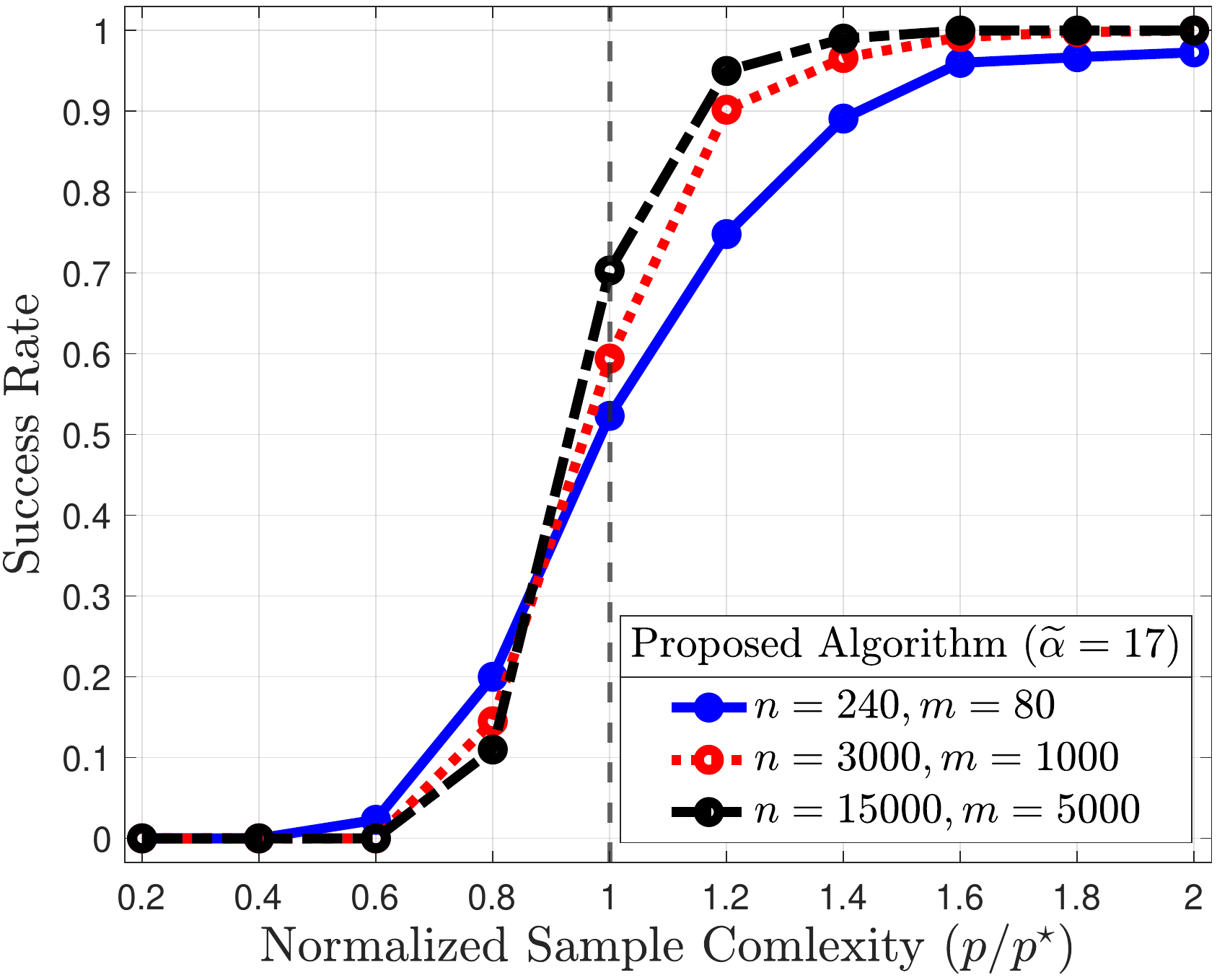}
		\label{fig:synth_alph17}}
	\hfill
	\subfloat[$n=3000$ and $m=1000$.]{\includegraphics[width=0.32\textwidth]{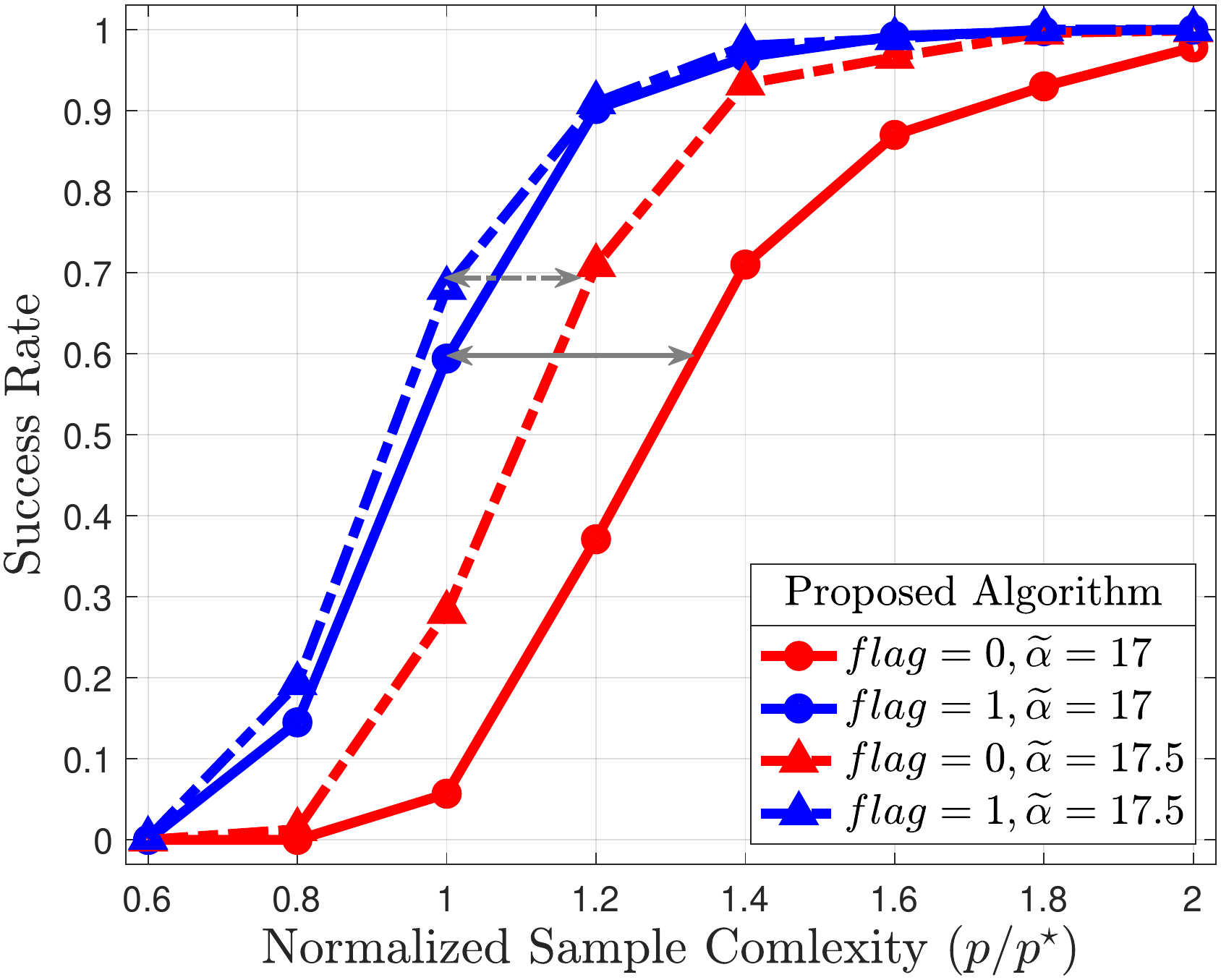}
		\label{fig:synth_algoComp}}
	\hfill
	\subfloat[$\sigma^2=0.5$.]{\includegraphics[width=0.4\textwidth]{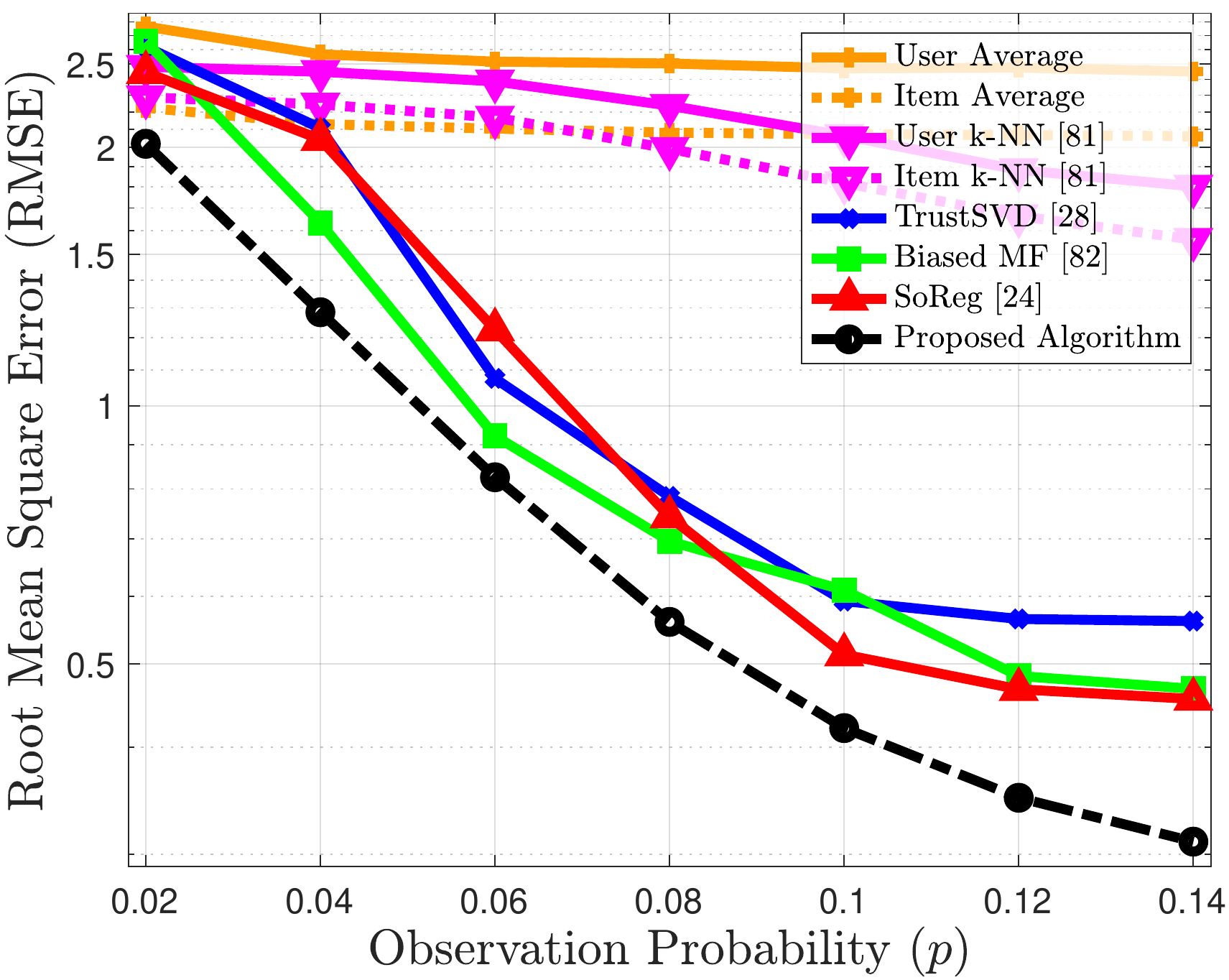}
		\label{fig:real_p_RMSE}}
	\hspace{5mm}
	\subfloat[$p=0.08$.]{\includegraphics[width=0.4\textwidth]{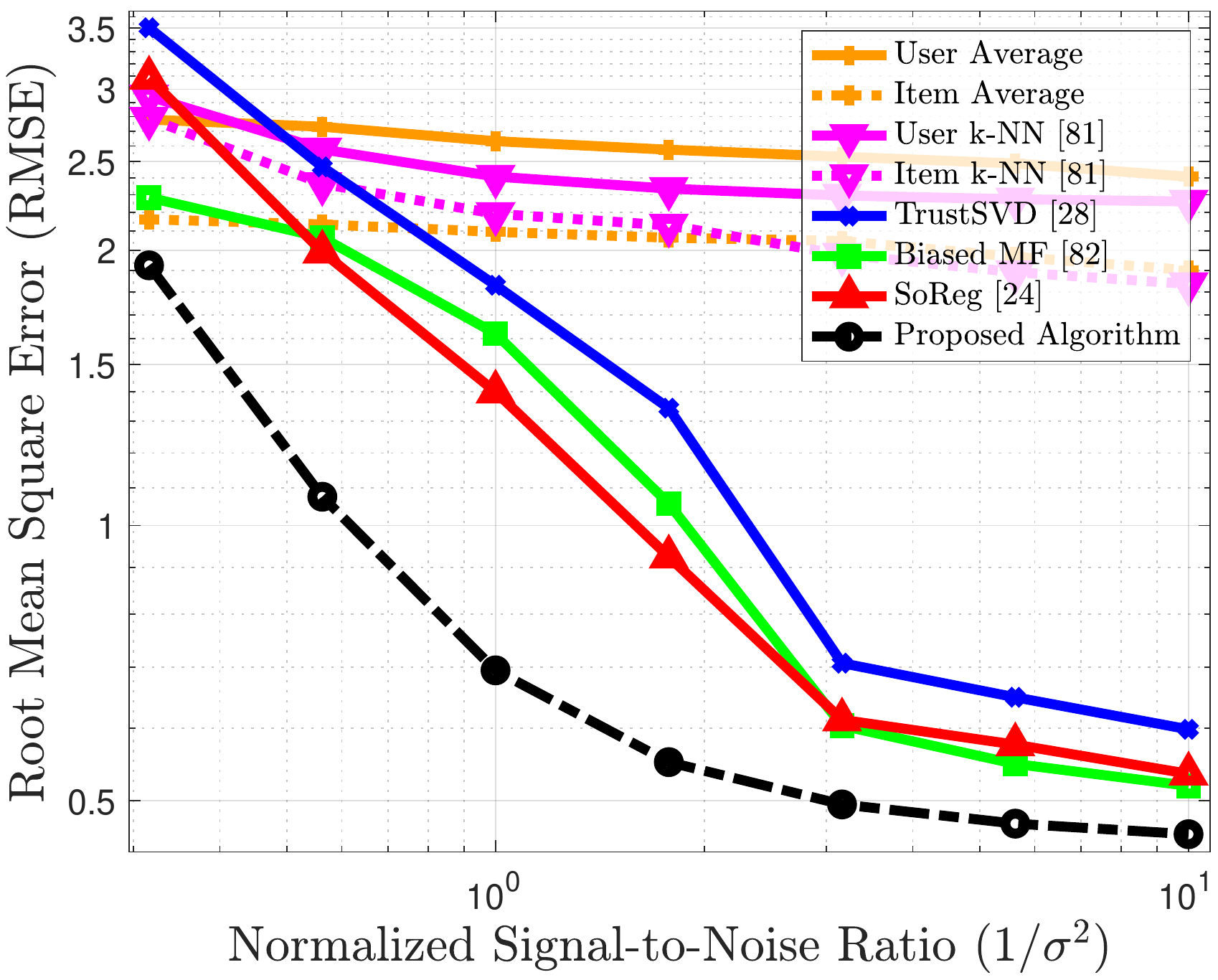}
		\label{fig:real_snr_RMSE}}
	\caption{
			(a), (b) The success rate of the proposed algorithm as a function of $p/p^\star$ under different values of $n$, $m$ and $\alphaTild{}$, where (a) corresponds to perfect clustering/grouping regime, and (b) corresponds to grouping-limited regime. (c) Performance of the proposed algorithm under two different local iterative refinement strategies.
			(d) Comparison of RMSE achieved by various recommendation algorithm on the Poliblog dataset \cite{adamic2005political} as a function of $p$.
			(e) Comparison of RMSE achieved by various recommendation algorithm on the Poliblog dataset as a function of $1/\sigma^2$.
	}
\end{figure*}

\begin{table}[t]
	\caption{Runtimes of recommendation algorithms for the experiment setting of Fig.~\ref{fig:real_p_RMSE} and $p=0.1$.}
	\label{tble:runTime}
	\centering
	\begin{adjustbox}{width=\textwidth}
		\begin{tabular}{ccccccccc}
			\toprule
			& User Average & Item Average & User k-NN & Item k-NN & TrustSVD  &Biased MF & SoReg & Proposed Algorithm \\
			\midrule
			Time (sec) & 0.021 & 0.025 & 0.299 & 0.311 & 0.482 & 0.266 & 0.328 & 0.055\\
			\bottomrule
		\end{tabular}
	\end{adjustbox}
	\vspace{-4mm}
\end{table}

Next, similar to~\cite{ahn2018binary, yoon2018joint,tan2019community,jo2020discrete}, the performance of the proposed algorithm is assessed on semi-real data (real graph but synthetic rating vectors). We consider a subgraph of the political blog network \cite{adamic2005political}, which is shown to exhibit a hierarchical structure \cite{peixoto2014hierarchical}. 
In particular, we consider a tall matrix setting of $n = 381$ and $m = 200$ in order to investigate the gain in sample complexity due to the graph side information. The selected subgraph consists of two clusters of political parties, each of which comprises three groups. The three groups of the first cluster consist of $98$, $34$ and $103$ users, while the three groups of the second cluster consist of $58$, $68$ and $20$ users\footnote{We refer to the supplementary material for a visualization of the selected subgraph of the political blog network using t-SNE algorithm.}. The corresponding rating vectors are generated such that the ratings are drawn from $[0,10]$ (i.e., real numbers), and the observations are corrupted by a Gaussian noise with mean zero and a given variance~$\sigma^2$. We use root mean square error (RMSE) as the evaluation metric, and assess the performance of the proposed algorithm against various recommendation algorithms, namely User Average, Item Average, User k-Nearest Neighbor (k-NN) \cite{pan2012knn}, Item k-NN \cite{pan2012knn}, TrustSVD \cite{guo2015trustsvd}, Biased Matrix Factorization (MF) \cite{koren2008factorization}, and Matrix Factorization with Social Regularization (SoReg) \cite{ma2011recommender}. Note that \cite{ahn2018binary,yoon2018joint} are designed to work for rating matrices whose elements are drawn from a finite field, and hence they cannot be run under the practical scenario considered in this setting.
In Fig.~\ref{fig:real_p_RMSE}, we compute RMSE as a function of $p$, for fixed $\sigma^2=0.5$. On the other hand, Fig.~\ref{fig:real_snr_RMSE} depicts RMSE as a function of the normalized signal-to-noise ratio $1/\sigma^2$, for fixed  $p=0.08$. It is evident that the proposed algorithm achieves superior performance over the state-of-the-art algorithms for a wide range of observation probabilities and Gaussian noise variances, demonstrating its viability and efficiency in practical scenarios.

Finally, Table~\ref{tble:runTime} demonstrates the computational efficiency of the proposed algorithm, and reports the runtimes of recommendation algorithms for the experiment setting of Fig.~\ref{fig:real_p_RMSE} and $p=0.1$. 
The runtimes are averaged over 20 trials. The proposed algorithm achieves a faster runtime than all other algorithms except for User Average and Item Average. However, as shown in Fig.~\ref{fig:real_p_RMSE}, the performance of these faster algorithms, in terms of RMSE, is inferior to the majority of other algorithms.

\section*{Broader Impact}
We emphasize two positive impacts of our work. First, it serves to enhance the performance of \emph{personalized} recommender systems (one of the most influential commercial applications) with the aid of social graph which is often available in a variety of applications. Second, it achieves \emph{fairness} among all users by providing high quality recommendations even to new users who have not rated any items before. One negative consequence of this work is w.r.t. the \emph{privacy} of users. User privacy may not be preserved in the process of exploiting \emph{indirect} information posed in social graphs, even though direct information, such as user profiles, is protected.

\begin{ack}
	The work of A. Elmahdy and S. Mohajer is supported in part by the National Science Foundation under Grants CCF-1617884 and CCF-1749981.
	The work of J. Ahn and C. Suh is supported by the National Research Foundation of Korea (NRF) grant funded by the Korea government (MSIP) (No.2018R1A1A1A05022889).
\end{ack}

\newpage
\bibliographystyle{IEEEtran}
\bibliography{myRef_ML}

\newpage
\section*{\hfil \hfil \LARGE Supplementary Material}
\setcounter{section}{1}
\addtocounter{section}{-1}
\section{List of Underlying Assumptions}
The proofs of Theorem~\ref{Thm:p_star} and Theorem~\ref{Thm:Alg_theory} rely on a number of assumptions on the model parameters $(n,m,p,\theta,\alpha,\beta, \gamma)$. We enumerate them before proceeding with the formal proofs in the following sections.
\begin{itemize}[leftmargin=8mm]
	\item $n$ and $m$ tend to $\infty$.
	\item $m = \omega(\log n)$ and $\log m = o(n)$. These assumptions rule out extremely tall or wide matrices, respectively, so that we can resort to large deviation theories in the proofs.
	\item $m = O(n)$. This is a sufficient condition for reliable estimation of $(\alpha, \beta, \theta)$ for the proposed computationally-efficient algorithm. If these parameters are known a priori, this assumption can be disregarded.
	\item $\theta = \Theta(1)$.
	\item $\alpha \geq \beta \geq \gamma$. This assumption reflects realistic scenarios in which users within the same group (or cluster) are more likely to be connected as per the social homophily theory~\cite{mcpherson2001birds}.
	\item $\alpha, \beta, \gamma = \Theta \left(\frac{\log n}{n}\right)$.
\end{itemize}

\section{Proof of Theorem~\ref{Thm:p_star}}
\subsection{Achievability proof}
\label{sec:achv}
Let $\psi_{\text{ML}}$ be the maximum likelihood estimator. Fix $\epsilon >0$. Consider the sufficient conditions claimed in Theorem~\ref{Thm:p_star}:
\begin{align}
	& p   \geq \frac{ (1+\epsilon) 3 \log m}{ (\sqrt{1-\theta} - \sqrt{\theta})^2 n }, \\
	& p \geq  \frac{  (1+\epsilon) \log n - \frac{1}{6} n I_g }{ (\sqrt{1-\theta} - \sqrt{\theta})^2 m \delta_g },  \\
	& p \geq  \frac{   (1+\epsilon) \log n - \frac{1}{6} n I_{c1} - \frac{1}{3} n I_{c2} }{ (\sqrt{1-\theta} - \sqrt{\theta})^2 m \delta_c}. 
\end{align}
For notational simplicity, let us define $I_r \coloneqq p (\sqrt{1-\theta} - \sqrt{\theta})^2$. Then, the above conditions can be rewritten as: 
\begin{align}
	&\frac{1}{3}nI_r \geq (1+\epsilon) \log m , \label{eq:suffCond_1} \\
	& m \intraTau I_r + \frac{1}{6}n\Ig \geq (1+\epsilon) \log n, \label{eq:suffCond_2}  \\
	& m \interTau I_r + \frac{1}{6}n\IcOne + \frac{1}{3}n\IcTwo \geq (1+\epsilon) \log n \label{eq:suffCond_3}. 
\end{align}
In what follows, we will show that the probability of error when applying $\psi_{\text{ML}}$ tends to zero if all of the above conditions are satisfied. 

Recall that each cluster consists of three groups and the rating vectors of the three groups respect some dependency relationship, reflected in $v_3^A = v_{1}^A \oplus v_2^A$ and $v_3^B = v_{1}^B \oplus v_2^B$. Here, $\vecV{i}{x}$ denotes the rating vector of the $i^{\text{th}}$ group in cluster $x$ where $i \in \{1,2,3\}$ and $x \in \{A, B\}$. This then motivates us to assume that without loss of generality, the ground-truth rating matrix, say $\gtMat\in \matSet$, reads:
\begin{align}
	\label{eq_M0}
	\gtMat 
	\!\coloneqq\!\!
	\begin{bmatrix}
		\begin{array}{c|c|c|c|c|c|c|c}
			\!\!\mathbf{1}_{\frac{n}{6}\times \tau_{000} m}\!\! & \!\!\mathbf{1}_{\frac{n}{6}\times \tau_{001} m}\!\! & \!\!\mathbf{1}_{\frac{n}{6}\times \tau_{010} m}\!\! & \!\!\mathbf{1}_{\frac{n}{6}\times \tau_{011} m}\!\! & \!\!\mathbf{1}_{\frac{n}{6}\times \tau_{100} m}\!\! & \!\!\mathbf{1}_{\frac{n}{6}\times \tau_{101} m}\!\! & \!\!\mathbf{1}_{\frac{n}{6}\times \tau_{110} m}\!\! & \!\!\mathbf{1}_{\frac{n}{6}\times \tau_{111} m}\!\!\\
			\hline
			\!\!\mathbf{0}_{\frac{n}{6}\times \tau_{000} m}\!\! & \!\!\mathbf{0}_{\frac{n}{6}\times \tau_{001} m}\!\! & \!\!\mathbf{0}_{\frac{n}{6}\times \tau_{010} m}\!\! & \!\!\mathbf{0}_{\frac{n}{6}\times \tau_{011} m}\!\! & \!\!\mathbf{1}_{\frac{n}{6}\times \tau_{100} m}\!\! & \!\!\mathbf{1}_{\frac{n}{6}\times \tau_{101} m}\!\! & \!\!\mathbf{1}_{\frac{n}{6}\times \tau_{110} m}\!\! & \!\!\mathbf{1}_{\frac{n}{6}\times \tau_{111} m}\!\!\\
			\hline
			\!\!\mathbf{1}_{\frac{n}{6}\times \tau_{000} m}\!\! & \!\!\mathbf{1}_{\frac{n}{6}\times \tau_{001} m}\!\! & \!\!\mathbf{1}_{\frac{n}{6}\times \tau_{010} m}\!\! & \!\!\mathbf{1}_{\frac{n}{6}\times \tau_{011} m}\!\! & \!\!\mathbf{0}_{\frac{n}{6}\times \tau_{100} m}\!\! & \!\!\mathbf{0}_{\frac{n}{6}\times \tau_{101} m}\!\! & \!\!\mathbf{0}_{\frac{n}{6}\times \tau_{110} m}\!\! & \!\!\mathbf{0}_{\frac{n}{6}\times \tau_{111} m}\!\!\\
			\hline
			\!\!\mathbf{0}_{\frac{n}{6}\times \tau_{000} m}\!\! & \!\!\mathbf{0}_{\frac{n}{6}\times \tau_{001} m}\!\! & \!\!\mathbf{1}_{\frac{n}{6}\times \tau_{010} m}\!\! & \!\!\mathbf{1}_{\frac{n}{6}\times \tau_{011} m}\!\! & \!\!\mathbf{0}_{\frac{n}{6}\times \tau_{100} m}\!\! & \!\!\mathbf{0}_{\frac{n}{6}\times \tau_{101} m}\!\! & \!\!\mathbf{1}_{\frac{n}{6}\times \tau_{110} m}\!\! & \!\!\mathbf{1}_{\frac{n}{6}\times \tau_{111} m}\!\!\\
			\hline
			\!\!\mathbf{0}_{\frac{n}{6}\times \tau_{000} m}\!\! & \!\!\mathbf{1}_{\frac{n}{6}\times \tau_{001} m}\!\! & \!\!\mathbf{0}_{\frac{n}{6}\times \tau_{010} m}\!\! & \!\!\mathbf{1}_{\frac{n}{6}\times \tau_{011} m}\!\! & \!\!\mathbf{0}_{\frac{n}{6}\times \tau_{100} m}\!\! & \!\!\mathbf{1}_{\frac{n}{6}\times \tau_{101} m}\!\! & \!\!\mathbf{0}_{\frac{n}{6}\times \tau_{110} m}\!\! & \!\!\mathbf{1}_{\frac{n}{6}\times \tau_{111} m}\!\!\\
			\hline
			\!\!\mathbf{0}_{\frac{n}{6}\times \tau_{000} m}\!\! & \!\!\mathbf{1}_{\frac{n}{6}\times \tau_{001} m}\!\! & \!\!\mathbf{1}_{\frac{n}{6}\times \tau_{010} m}\!\! & \!\!\mathbf{0}_{\frac{n}{6}\times \tau_{011} m}\!\! & \!\!\mathbf{0}_{\frac{n}{6}\times \tau_{100} m}\!\! & \!\!\mathbf{1}_{\frac{n}{6}\times \tau_{101} m}\!\! & \!\!\mathbf{1}_{\frac{n}{6}\times \tau_{110} m}\!\! & \!\!\mathbf{0}_{\frac{n}{6}\times \tau_{111} m}\!\!
		\end{array}
	\end{bmatrix}
\end{align}
where $0 < \tau_\ell < 1$ for $\ell \in \{0,1\}^3$, and $\sum_{\ell \in \{0,1\}^3} \tau_\ell = 1$. Here, we divide the columns of $\gtMat$ into eight sections $\Tau_\ell$ where 
\begin{align*}
	\Tau_{b_1b_2b_3} \!=\left\{c\in[m]: \textrm{column $c$ of } \gtMat 
	= \begin{bmatrix}
		\mathbf{1}_{\frac{n}{6}} \!&\! b_1\mathbf{1}_{\frac{n}{6}} \!&\! (1\oplus b_1)\mathbf{1}_{\frac{n}{6}} \!&\! b_2\mathbf{1}_{\frac{n}{6}} \!&\! b_3\mathbf{1}_{\frac{n}{6}} \!&\! (b_2\oplus b_3)\mathbf{1}_{\frac{n}{6}} 
	\end{bmatrix}^\intercal\right\},
\end{align*}
for $b_1,b_2,b_3 \in \{0,1\}^3$, and we have $\tau_\ell = |\Tau_\ell|/m$.  

For a user partitioning $\mathcal{Z}$, let $\userPairSet{\alpha}{\mathcal{Z}}$ denote the set of pairs of users within any group; $\userPairSet{\beta}{\mathcal{Z}}$ denote the set of pairs of users in different groups within any cluster; and $\userPairSet{\gamma}{\mathcal{Z}}$ denote the set of pairs of users in different clusters. Formally, we have
\begin{align}
	\begin{split}
		\userPairSet{\alpha}{\mathcal{Z}} 
		&=
		\left\{
		(a,b) : a \in Z(x,i), \: b \in Z(x,i), 
		\mbox{ for } x \in \{A,B\}, \: i \in [3]
		\right\},
		\\
		\userPairSet{\beta}{\mathcal{Z}}
		&=
		\left\{
		(a,b) : a \in Z(x,i), \: b \in Z(x,j), \:
		\mbox{ for } x \in \{A,B\},\: i, j \in [3],\: i\neq j
		\right\},
		\\
		\userPairSet{\gamma}{\mathcal{Z}} 
		&=
		\left\{
		(a,b) : a \in Z(x,i), \: b \in Z(y,j), \:
		\mbox{ for } x,y \in \{A,B\},\: x \neq y, \: i, j \in [3]
		\right\}.
	\end{split}
\end{align}
Recall that the user partitioning induced by any rating matrix in $\matSet$ should satisfy the property that all groups have equal size of $n/6$ users. This implies that the sizes of $\userPairSet{\alpha}{\mathcal{Z}}$, $\userPairSet{\beta}{\mathcal{Z}}$ and $\userPairSet{\gamma}{\mathcal{Z}}$ are constants and given by 
\begin{align}
	\vert \userPairSet{\alpha}{\mathcal{Z}} \vert = 6 \binom{n/6}{2}, 
	\qquad
	\vert \userPairSet{\beta}{\mathcal{Z}} \vert = 2 \left(n/6\right)^2,
	\qquad
	\vert \userPairSet{\gamma}{\mathcal{Z}} \vert = \left(n/2\right)^2,
\end{align}
for any user partitioning. Furthermore, for a graph $\mathcal{G}$ and a user partitioning $\mathcal{Z}$, define $\numEdge{\alpha}{\mathcal{G}}{\mathcal{Z}}$ as the number of edges within any group; $\numEdge{\beta}{\mathcal{G}}{\mathcal{Z}}$ as the number of edges across groups within any cluster; and $\numEdge{\gamma}{\mathcal{G}}{\mathcal{Z}}$ as the number of edges across clusters. More formally, we have
\begin{align}
	\numEdge{\mu}{\mathcal{G}}{\mathcal{Z}}
	&=
	\sum_{(a,b) \in \userPairSet{\mu}{\mathcal{Z}}} 
	\indicatorFn{(a,b) \in \mathcal{E}},
	\label{eq:e_mu}
\end{align}
for $\mu \in \{\alphaEdge,\betaEdge,\gammaEdge\}$. The following lemma gives a precise expression of $\mathsf{L}(X)$, where $\mathsf{L}(X)$ denotes the negative log-likelihood of a candidate matrix $X$. 

\begin{lemma}
	For a given fixed input pair $(Y,\mathcal{G})$ and any $X \in \matSet$, we have
	\begin{align}
		\mathsf{L}(X) 
		=
		\log \left( \frac{1-\theta}{\theta}\right) \numDiffElmnt{Y}{X} 
		+
		\sum\limits_{\mu\in\left\{\alphaEdge,\betaEdge,\gammaEdge\right\}}
		\left(
		\log \left(\frac{1-\mu}{\mu}\right) \numEdge{\mu}{\mathcal{G}}{\mathcal{Z}}
		-
		\log(1-\mu) \left\vert \userPairSet{\mu}{\mathcal{Z}} \right\vert 
		\right).
		\label{eq:neg_log_likelihood}
	\end{align}
	\label{lm:neg_log_likelihood}
\end{lemma}

By symmetry, $\mathbb{P}\left[\psi_{\text{ML}}(Y, \mathcal{G}) \neq M \right]$ is the same for all $M$'s as long as the considered matrix respects the $\delta$-constraint, i.e., belongs to the class of ${\cal M}^{(\delta)}$ where $\delta := \{ \delta_c, \delta_g \}$. Hence,
\begin{align}
	P_e^{(\diff)}(\psi_{\text{ML}}) 
	: = \max_{M \in \matSet} \mathbb{P}\left[\psi_{\text{ML}}(Y, \mathcal{G}) \neq M\right]
	=
	\mathbb{P}\left[\psi_{\text{ML}} (Y, \mathcal{G}) \neq \gtMat \:\vert\: \mathbf{M} = \gtMat \right]. \label{eq:errorProb_wc_app}
\end{align}
By applying the union bound together with the definition of MLE, we then obtain 
\begin{align}
	P_e^{(\diff)}(\psi_{\text{ML}}) \leq
	\sum\limits_{X \neq \gtMat} \mathbb{P}\left[\mathsf{L}(\gtMat) \geq \mathsf{L}(X)\right]
	\label{ineq:union_bound}
\end{align}
It turns out that an interested error event $\{ \mathsf{L}(\gtMat) \geq \mathsf{L}(X) \}$ depends solely on two key parameters which dictate the relationship between $X$ and $M_0 \in {\cal M}^{(\delta)}$. Let us first introduce some  notations relevant to the two parameters. Let $\{\vecV{i}{x}: x\in\{A,B\}, i\in [3]\}$ be the rating vectors w.r.t. $\gtMat$. Let $\{\vecU{i}{x}: x\in\{A,B\}, i\in [3]\}$ be the counterparts w.r.t. $X$. The first key parameter, which we denote by $\kUsrs{i}{j}{x}{y}$, indicates the number of users in group $i$ of cluster $x$ whose rating vector $v_i^x$'s are swapped with the rating vectors $u_j^y$'s of users in group $j$ of cluster $y$. The second key parameter, which we denote by $\dElmntsGen{i}{j}{x}{y}$, is the hamming distance between $v_i^x$ and $u_j^x$: $\hamDist{\vecV{i}{x}(\ell)}{\vecU{j}{x}(\ell)}$. 

Based on these two parameters, the set of rating matrices $\matSet$ is partitioned into a number of classes of matrices $\mathcal{X}(T)$. Here, each matrix class $\mathcal{X}(T)$ is defined as the set of rating matrices that is characterized by a tuple~$T$ where
\begin{align}
	T
	=
	\left(
	\left\{\kUsrs{i}{j}{x}{y}\right\}_{x,y \in \{A,B\}, \: i,j \in [3]}, 
	\left\{\dElmntsGen{i}{j}{x}{y}\right\}_{x,y \in \{A,B\}, \: i,j \in [3]}
	\right).
	\label{eq:tuple_ratMat}
\end{align}
Define $\tupleSetDelta$ as the set of all non-all-zero tuples $T$. Using the introduced set ${\cal T}^{(\delta)}$ and the tuple $T$, we can then rewrite the RHS of \eqref{ineq:union_bound} as:
\begin{align}
	\sum\limits_{X \neq \gtMat} \mathbb{P}\left[\mathsf{L}(\gtMat) \geq \mathsf{L}(X)\right]
	= \sum\limits_{T \in \mathcal{T}^{(\delta)}} 
	\sum\limits_{X \in \mathcal{X}(T)}
	\mathbb{P}\left[\mathsf{L}(\gtMat) \geq \mathsf{L}(X)\right]. \label{eq:matrix_enumerate}
\end{align}

Lemma 2 (stated below) provides an upper bound on $\mathbb{P}\left[\mathsf{L}(\gtMat) \geq \mathsf{L}(X)\right]$ for $X \in {\cal X}(T)$ and $T \in {\cal T}^{(\delta)}$.

Before proceeding with the achievability proof, we define some notations, and then present the following two lemmas that provide an upper bound on the probability that $\mathsf{L}(\gtMat)$ is greater than or equal to $\mathsf{L}(X)$. These lemmas are crucial for the convergence analysis to follow. For a ground truth rating matrix $\gtMat = (\cV_0,\cZ_0)$; a candidate rating matrix $X = (\cV, \cZ)$; and a tuple $T \in \tupleSetDelta$, define the following disjoint sets:
\begin{itemize}[leftmargin=*]
	\item define $\diffEntriesSet = \diffEntriesSet(\gtMat, X)$ as the set of matrix entries where $X \neq \gtMat$. Formally, we have
	\begin{align}
		\diffEntriesSet = \left\{(r,t)\in [n] \times [m]: X(r,t) \neq \gtMat(r,t)\right\};
		\label{eq:N1_defn}
	\end{align}
	\item define $\misClassfSet{\betaEdge}{\alphaEdge} = \misClassfSet{\betaEdge}{\alphaEdge}(\cZ_0, \cZ)$ as the set of pairs of users where the two users of each pair belong to different groups of the same cluster in $\gtMat$ (and therefore they are connected with probability $\betaEdge$), but they are estimated to be in the same group in $X$ (and hence, given the estimator output, the belief for the existence of an edge between these two users is $\alphaEdge$). Formally, we have
	\begin{align}
		\misClassfSet{\betaEdge}{\alphaEdge} 
		&= 
		\left\{
		(a,b):
		a \in Z_0(x,i_1) \cap Z(y,j),\: b \in Z_0(x,i_2) \cap Z(y,j),
		\right. \nonumber \\
		&\phantom{=\{} \left. 
		\mbox{ for } x, y \in \{A,B\},
		i_1, i_2, j \in [3],\: i_1 \neq i_2
		\right\}.
		\label{eq:N2rev_defn}
	\end{align}	
	On the other hand, define $\misClassfSet{\alphaEdge}{\betaEdge} = \misClassfSet{\alphaEdge}{\betaEdge} (\cZ_0, \cZ)$ as
	\begin{align}
		\misClassfSet{\alphaEdge}{\betaEdge} &= \left\{
		(a,b):
		a \in Z_0(x,i) \cap Z(y,j_1),\: b \in Z_0(x,i) \cap Z(y,j_2),
		\right. \nonumber \\
		&\phantom{=\{} \left. 
		\mbox{ for } x, y \in \{A,B\},\:
		i, j_1, j_2 \in [3],\: j_1 \neq j_2
		\right\};
		\label{eq:N2_defn}
	\end{align}
	\item define $\misClassfSet{\gammaEdge}{\alphaEdge} = \misClassfSet{\gammaEdge}{\alphaEdge}(\cZ_0, \cZ)$ as the set of pairs of users where the two users of each pair belong to different clusters in $\gtMat$ (and therefore they are connected with probability $\gammaEdge$), but they are estimated to be in the same group in $X$ (and hence, given the estimator output, the belief for the existence of an edge between these two users is $\alphaEdge$). Formally, we have
	\begin{align}
		\misClassfSet{\gammaEdge}{\alphaEdge} &= \left\{
		(a,b) :
		a \in Z_0(x_1,i_1) \cap Z(y,j),\: b \in Z_0(x_2,i_2) \cap Z(y,j), 
		\right. \nonumber \\
		&\phantom{=\{} \left. 
		\mbox{ for } x_1, x_2, y \in \{A,B\},\: x_1 \neq x_2,\:
		i_1, i_2, j \in [3]
		\right\}.
		\label{eq:N3rev_defn}
	\end{align}
	On the other hand, define $\misClassfSet{\alphaEdge}{\gammaEdge} = \misClassfSet{\alphaEdge}{\gammaEdge}(\cZ_0, \cZ)$ as
	\begin{align}
		\misClassfSet{\alphaEdge}{\gammaEdge} &= \left\{
		(a,b) :
		a \in Z_0(x,i) \cap Z(y_1,j_1),\: b \in Z_0(x,i) \cap Z(y_2, j_2),
		\right. \nonumber \\
		&\phantom{=\{} \left. 
		\mbox{ for } x, y_1, y_2 \in \{A,B\},\: y_1 \neq y_2,\:
		i, j_1, j_2 \in [3]
		\right\};
		\label{eq:N3_defn}
	\end{align}
	\item define $\misClassfSet{\gammaEdge}{\betaEdge} = \misClassfSet{\gammaEdge}{\betaEdge}(\cZ_0, \cZ)$ as the set of pairs of users where the two users of each pair belong to different clusters in $\gtMat$ (and therefore they are connected with probability $\gammaEdge$), but they are estimated to be in different groups of the  same cluster in $X$ (and hence, given the estimator output, the belief for the existence of an edge between these two users is $\betaEdge$). Formally, we have
	\begin{align}
		\misClassfSet{\gammaEdge}{\betaEdge} &\!=\! \left\{
		(a,b) \!:
		a \!\in\! Z_0(x_1,i_1) \!\cap\! Z(y,j_1),\: b \!\in\! Z_0(x_2,i_2) \!\cap\! Z(y,j_2),
		\right. \nonumber \\
		&\phantom{=\{} \left. 
		\!\mbox{ for } x_1, x_2, y \!\in\! \{A,B\},\: x_1 \!\neq\! x_2, \:
		i_1, i_2, j_1, j_2 \!\in\! [3],\: j_1 \!\neq\! j_2
		\right\}\!.
		\label{eq:N4rev_defn}
	\end{align}
	On the other hand, define $\misClassfSet{\betaEdge}{\gammaEdge} = \misClassfSet{\betaEdge}{\gammaEdge}(\cZ_0, \cZ)$ as
	\begin{align}
		\misClassfSet{\betaEdge}{\gammaEdge} &\!=\! \left\{
		(a,b) \!:
		a \!\in\! Z_0(x,i_1) \!\cap\! Z(y_1,j_1),\: b \!\in\! Z_0(x, i_2) \!\cap\! Z(y_2, j_2), 
		\right. \nonumber \\
		&\phantom{=\{} \left. 
		\!\mbox{ for } x, y_1, y_2 \!\in\! \{A,B\},\: y_1 \!\neq\! y_2, \:
		i_1, i_2, j_1, j_2 \!\in\! [3],\: i_1 \!\neq\! i_2
		\right\}\!.
		\label{eq:N4_defn}
	\end{align}
\end{itemize}
Let $\mathsf{B}_i^{(\sigma)}$ denote the $i^{\text{th}}$~Bernoulli random variable with parameter $\sigma \in \{p, \theta, \alphaEdge, \betaEdge, \gammaEdge\}$. Define the following sets of independent Bernoulli random variables:
\begin{align}
	\left\{\mathsf{B}_i^{(p)}: i\in\diffEntriesSet\right\}, \:
	\left\{\mathsf{B}_i^{(\theta)}:i \in \diffEntriesSet \right\}, \:
	\left\{
	\mathsf{B}_i^{(\mu)}: 
	i\in \misClassfSet{\mu}{\nu},\:
	\mu, \nu \in \left\{\alphaEdge,\betaEdge,\gammaEdge \right\}, \:
	\mu \neq \nu
	\right\}.
	\label{eq:bern_sets}
\end{align} 
Now, define $\mathbf{B} = \mathbf{B} \left(\diffEntriesSet, \{\misClassfSet{\mu}{\nu} : \mu, \nu \in \{\alphaEdge,\betaEdge,\gammaEdge \}, \: \mu \neq \nu\}\right)$ as
\begin{align}    
	\mathbf{B}
	&\coloneqq
	\log\left(\frac{1-\theta}{\theta}\right) 
	\sum_{i \in \diffEntriesSet}
	\mathsf{B}_i^{(p)}
	\left(2\mathsf{B}_i^{(\theta)} - 1\right)
	\nonumber\\
	&\phantom{=}
	+ \left(\log\frac{(1-\betaEdge)\alphaEdge}{(1-\alphaEdge)\betaEdge}\right)
	\left(
	\sum\limits_{i \in \misClassfSet{\betaEdge}{\alphaEdge}} \mathsf{B}_i^{(\betaEdge)}
	- \sum\limits_{i \in \misClassfSet{\alphaEdge}{\betaEdge}} \mathsf{B}_i^{(\alphaEdge)} 
	\right)
	+
	\left(\log\frac{1-\alphaEdge}{1-\betaEdge}\right) 
	\left(
	\left\vert \misClassfSet{\betaEdge}{\alphaEdge}\right\vert
	- \left\vert \misClassfSet{\alphaEdge}{\betaEdge}\right\vert
	\right)
	\nonumber\\
	&\phantom{=}
	+ \left(\log\frac{(1-\gammaEdge)\alphaEdge}{(1-\alphaEdge)\gammaEdge}\right)
	\left(
	\sum\limits_{i \in \misClassfSet{\gammaEdge}{\alphaEdge}} \mathsf{B}_i^{(\gammaEdge)}
	- \sum\limits_{i \in \misClassfSet{\alphaEdge}{\gammaEdge}} \mathsf{B}_i^{(\alphaEdge)} 
	\right)
	+  
	\left(\log\frac{1-\alphaEdge}{1-\gammaEdge}\right) 
	\left(
	\left\vert \misClassfSet{\gammaEdge}{\alphaEdge}\right\vert
	- \left\vert \misClassfSet{\alphaEdge}{\gammaEdge}\right\vert
	\right)
	\nonumber\\
	&\phantom{=}
	+ \left(\log\frac{(1-\gammaEdge)\betaEdge}{(1-\betaEdge)\gammaEdge}\right)
	\left(
	\sum\limits_{i \in \misClassfSet{\gammaEdge}{\betaEdge}} \mathsf{B}_i^{(\gammaEdge)}
	- \sum\limits_{i \in \misClassfSet{\betaEdge}{\gammaEdge}} \mathsf{B}_i^{(\betaEdge)} 
	\right)
	+  
	\left(\log\frac{1-\betaEdge}{1-\gammaEdge}\right)
	\left(
	\left\vert \misClassfSet{\gammaEdge}{\betaEdge}\right\vert
	- \left\vert \misClassfSet{\betaEdge}{\gammaEdge}\right\vert
	\right).
	\label{eq:B_middle}
\end{align}

In the following lemma, we write each summand in \eqref{eq:matrix_enumerate} in terms of \eqref{eq:B_middle}.
\begin{lemma}
	\label{lemma:neg_lik_M0_XT}
	For any $X \in \mathcal{X}(T)$ and $T \in \tupleSetDelta$, we have
	\begin{align}
		\mathbb{P}\left[\mathsf{L}\left(\gtMat\right) \geq \mathsf{L}(X)\right]
		= 
		\mathbb{P} \left[\mathbf{B} \geq 0 \right].
		\label{eq:lemma3}
	\end{align}
	We refer to Appendix~\ref{proof:neg_lik_M0_XT} for the proof of Lemma~\ref{lemma:neg_lik_M0_XT}.
\end{lemma}

The following lemma provides an upper bound of the RHS of \eqref{eq:lemma3}. 
\begin{lemma}
	\label{lemma:upper_bound_B}
	For any $\{\misClassfSet{\mu}{\nu} : \mu, \nu \in \{\alphaEdge,\betaEdge,\gammaEdge \}, \: \mu \neq \nu\}$, we have
	\begin{align}
		&\mathbb{P} \left[ \mathbf{B}
		\geq 0 
		\right]
		\leq 
		\exp\left(-\left(1+o(1)\right) \left(\lvert\diffEntriesSet\rvert \Ir 
		+ 
		\Pleftrightarrow{\alphaEdge}{\betaEdge} \: \Ig
		+ \Pleftrightarrow{\alphaEdge}{\gammaEdge} \:\IcOne
		+ \Pleftrightarrow{\betaEdge}{\gammaEdge} \: \IcTwo 
		\right)\right),
		\label{eq:lemma1_2}
	\end{align}
	where
	\begin{align}
		\Pleftrightarrow{\alphaEdge}{\betaEdge} = 
		\frac{\left\vert \misClassfSet{\betaEdge}{\alphaEdge} \right\vert + \left\vert \misClassfSet{\alphaEdge}{\betaEdge} \right\vert}{2},
		\qquad
		\Pleftrightarrow{\alphaEdge}{\gammaEdge} = 
		\frac{\left\vert \misClassfSet{\gammaEdge}{\alphaEdge} \right\vert + \left\vert \misClassfSet{\alphaEdge}{\gammaEdge} \right\vert}{2},
		\qquad
		\Pleftrightarrow{\betaEdge}{\gammaEdge} = 
		\frac{\left\vert \misClassfSet{\gammaEdge}{\betaEdge} \right\vert + \left\vert \misClassfSet{\betaEdge}{\gammaEdge} \right\vert}{2}.
		\label{eq:Pleftrightarrow_defn}
	\end{align}
	We refer to Appendix~\ref{proof:upper_bound_B} for the proof of Lemma~\ref{lemma:upper_bound_B}.
\end{lemma}

By Lemma~\ref{lemma:neg_lik_M0_XT} and Lemma~\ref{lemma:upper_bound_B}, the RHS of~\eqref{eq:matrix_enumerate} is upper bounded by
\begin{align}
	\label{ineq:upper_bound_M0_XT}
	P_e^{(\diff)}(\psi_{\text{ML}}) 
	&\leq
	\sum\limits_{T \in \tupleSetDelta} 
	\sum\limits_{X \in \mathcal{X}(T)}
	\mathbb{P}\left[\mathsf{L}(\gtMat) \geq \mathsf{L}(X)\right].
	\nonumber\\
	&\leq  	
	\sum\limits_{T \in \tupleSetDelta} 
	\sum\limits_{X \in \mathcal{X}(T)}
	\exp\left(-\left(1+o(1)\right) \left(\lvert\diffEntriesSet\rvert \Ir 
	+ 
	\Pleftrightarrow{\alphaEdge}{\betaEdge} \: \Ig
	+ \Pleftrightarrow{\alphaEdge}{\gammaEdge} \:\IcOne
	+ \Pleftrightarrow{\betaEdge}{\gammaEdge} \: \IcTwo 
	\right)\right).
\end{align}

Next, we analyze the performance of the ML decoder by comparing the ground truth user partitioning with that of the decoder. For a non-negative constant $\tau \in (0,\: (\epsilon\log m - (2+\epsilon)\log 4) / (2 (1 + \epsilon)\log m))$, where $\epsilon > \max\{(2 \log 2) / \log n, \: (4 \log 2) / \log (4m), \: (2\log 4)/\log(m/4)\}$, define $\sigma(x,i)$ as the set of pairs of cluster and group in $\cZ$ whose number of overlapped users with $\cZ_0(x,i)$ exceeds a $(1-\tau)$ fraction of the group size. Formally, we have
\begin{align}
	\sigma(x,i)
	=
	\left\{(y,j)\in\{A,B\}\times[3]: 
	\left| Z_0(x,i) \cap Z(y,j)\right|
	\geq 
	(1-\tau) \frac{n}{6} \right\}.
	\label{eq:sigma(x,i)}
\end{align}
Note that $\tau < 0.5$, which implies that $|\sigma(x,i)| \leq 1$ since the size of any group is $n/6$ users. For $|\sigma(x,i)| = 1$, let $\sigma(x,i) = \{(\sigma(x), \sigma(i|x))\}$. Accordingly, partition the set $\tupleSetDelta$ into two subsets $\TsmallErr$ and $\TlargeErr$ that are defined as follows:
\begin{align}
	\TsmallErr
	&= 
	\left\{ T \in \mathcal{T}^{(\delta)} : 
	\forall (x,i) \in \{A,B\}\times [3] \text{ such that }
	\left\vert \sigma(x,i) \right\vert = 1, \:
	\dElmntsGen{i}{\:\sigma(i|x)}{x}{\:\sigma(x)} \leq \tau m \min\{\interTau, \intraTau\}
	\right\},
	\label{eq:TsmallErr}\\
	\TlargeErr
	&= 
	\left\{ T \in \mathcal{T}^{(\delta)} : 
	\exists (x,i) \in \{A,B\}\times [3] \text{ such that }
	\left(\left\vert \sigma(x,i) \right\vert = 0\right)
	\right\}
	\nonumber\\
	&\phantom{=}\cup
	\left\{
	T \in \mathcal{T}^{(\delta)} : 
	\forall (x,i) \in \{A,B\} \times [3] 
	\text{ such that }
	\left\vert \sigma(x,i) \right\vert = 1,\: 
	\right. \nonumber \\
	&\phantom{=\cup\{} \left.
	\exists (x,i) \in \{A,B\} \times [3] 
	\text{ such that }
	\dElmntsGen{i}{\:\sigma(i|x)}{x}{\:\sigma(x)} >  \tau m \min\{\interTau, \intraTau\}
	\right\}.
	\label{eq:TlargeErr}
\end{align}
Intuitively, when $T \in \TsmallErr$, the class of matrices $\mathcal{X}(T)$ corresponds to the typical (i.e., small) error set. On the other hand, when $T \in \TlargeErr$, the class of matrices $\mathcal{X}(T)$ corresponds to the atypical (i.e., large) error set that has negligible probability mass. Consequently, the RHS of \eqref{ineq:upper_bound_M0_XT} is upper bounded by
\begin{align}
	P_e^{(\diff)}(\psi_{\text{ML}})
	&\leq
	\sum\limits_{T \in  \TsmallErr} 
	\sum\limits_{X \in \mathcal{X}(T)}
	\exp\left(-\left(1+o(1)\right) \left(\lvert\diffEntriesSet\rvert \Ir 
	+ 
	\Pleftrightarrow{\alphaEdge}{\betaEdge} \: \Ig
	+ \Pleftrightarrow{\alphaEdge}{\gammaEdge} \:\IcOne
	+ \Pleftrightarrow{\betaEdge}{\gammaEdge} \: \IcTwo 
	\right)\right)
	\nonumber\\
	&\phantom{\leq} 
	+ 
	\sum\limits_{T \in \TlargeErr} 
	\sum\limits_{X \in \mathcal{X}(T)}
	\exp\left(-\left(1+o(1)\right) \left(\lvert\diffEntriesSet\rvert \Ir 
	+ 
	\Pleftrightarrow{\alphaEdge}{\betaEdge} \: \Ig
	+ \Pleftrightarrow{\alphaEdge}{\gammaEdge} \:\IcOne
	+ \Pleftrightarrow{\betaEdge}{\gammaEdge} \: \IcTwo 
	\right)\right).
	\label{partial_sum}
\end{align}
The following lemmas provide an upper bound on each term in \eqref{partial_sum}, and evaluating the limits as $n$ and $m$ tend to infinity.

\begin{lemma}
	\label{lemma:T1_related} 
	For any $\{\misClassfSet{\mu}{\nu} : \mu, \nu \in \{\alphaEdge,\betaEdge,\gammaEdge \}, \: \mu \neq \nu\}$, we have\footnote{As $n$ tends to infinity, $m$ also tends to infinity since $m = \omega(\log n)$.}
	\begin{align}
		& \lim_{n\rightarrow \infty }
		\sum\limits_{T \in \TsmallErr} 
		\sum\limits_{X \in \mathcal{X}(T)}
		\exp\left(-\left(1+o(1)\right) \left(\lvert\diffEntriesSet\rvert \Ir 
		+ 
		\Pleftrightarrow{\alphaEdge}{\betaEdge} \: \Ig
		+ \Pleftrightarrow{\alphaEdge}{\gammaEdge} \:\IcOne
		+ \Pleftrightarrow{\betaEdge}{\gammaEdge} \: \IcTwo 
		\right)\right)
		= 0.
		\label{eq:lemma_T1_1} 
	\end{align}
	We refer to Appendix~\ref{proof:T1_related} for the proof of Lemma~\ref{lemma:T1_related}. 
\end{lemma}

\begin{lemma}
	\label{lemma:eta_omega(nm)}
	For any $\{\misClassfSet{\mu}{\nu} : \mu, \nu \in \{\alphaEdge,\betaEdge,\gammaEdge \}, \: \mu \neq \nu\}$, we have
	\begin{align}
		& \lim_{n\rightarrow \infty }
		\sum\limits_{T \in  \TlargeErr} 
		\sum\limits_{X \in \mathcal{X}(T)}
		\exp\left(-\left(1+o(1)\right) \left(\lvert\diffEntriesSet\rvert \Ir 
		+ 
		\Pleftrightarrow{\alphaEdge}{\betaEdge} \: \Ig
		+ \Pleftrightarrow{\alphaEdge}{\gammaEdge} \:\IcOne
		+ \Pleftrightarrow{\betaEdge}{\gammaEdge} \: \IcTwo 
		\right)\right)
		= 0.
		\label{ineq:eta_omega(nm)}
	\end{align}
	We refer to Appendix~\ref{proof:eta_omega(nm)} for the proof of Lemma~\ref{lemma:eta_omega(nm)}.
\end{lemma}

By Lemma \ref{lemma:T1_related} and Lemma \ref{lemma:eta_omega(nm)}, the limit of the worst-case probability of error $P_e^{(\diff)}(\psi_{\text{ML}})$ in \eqref{partial_sum} as $n$ and $m$ tend to infinity is evaluated as
\begin{align}
	&\lim_{n\rightarrow \infty } P_e^{(\diff)}(\psi_{\text{ML}})
	\nonumber\\
	&\leq 
	\lim_{n\rightarrow \infty } \left(
	\sum\limits_{T \in  \TsmallErr} 
	\sum\limits_{X \in \mathcal{X}(T)}
	\exp\left(-\left(1+o(1)\right) \left(\lvert\diffEntriesSet\rvert \Ir 
	+ 
	\Pleftrightarrow{\alphaEdge}{\betaEdge} \: \Ig
	+ \Pleftrightarrow{\alphaEdge}{\gammaEdge} \:\IcOne
	+ \Pleftrightarrow{\betaEdge}{\gammaEdge} \: \IcTwo 
	\right)\right)
	\right.
	\nonumber\\
	&\phantom{=\lim_{n\rightarrow \infty } \left(\right.}
	\left.
	+ \sum\limits_{T \in \TlargeErr} 
	\sum\limits_{X \in \mathcal{X}(T)}
	\exp\left(-\left(1+o(1)\right) \left(\lvert\diffEntriesSet\rvert \Ir 
	+ 
	\Pleftrightarrow{\alphaEdge}{\betaEdge} \: \Ig
	+ \Pleftrightarrow{\alphaEdge}{\gammaEdge} \:\IcOne
	+ \Pleftrightarrow{\betaEdge}{\gammaEdge} \: \IcTwo 
	\right)\right)
	\right)
	\nonumber\\
	&= 0,
\end{align}
which implies that $\lim_{n\rightarrow \infty } P_e^{(\diff)}(\psi_{\text{ML}}) = 0$. This concludes the achievability proof of Theorem~\ref{Thm:p_star}.
\hfill $\blacksquare$

\subsection{Converse Proof}
Define $I_r \coloneqq p (\sqrt{1-\theta} - \sqrt{\theta})^2$. The goal of the converse proof is to show that $P_e^{(\tau)} (\psi) \nrightarrow 0$ as $n \rightarrow \infty$ for any set of feasible rating matrices $\matSet$ and estimator $\psi$, if at least one of the following conditions is satisfied:
\begin{eqnarray}
	&\frac{1}{3} n I_r 
	\leq
	(1-\epsilon) \log m,  &\textbf{(Perfect clustering/grouping regime)}
	\label{eq:necCond_1}\\
	&\intraTau m I_r + \frac{1}{6} n \Ig
	\leq
	(1-\epsilon) \log n,  &\textbf{(Grouping-limited regime)}
	\label{eq:necCond_2}\\  
	&\interTau m I_r + \frac{1}{6} n \IcOne + \frac{1}{3} n \IcTwo \leq  (1-\epsilon) \log n,
	&\textbf{(Clustering-limited regime)}
	\label{eq:necCond_3}
\end{eqnarray}

We first seek a lower bound on the infimum of the worst-case probability of error over all estimators. Let $\mathbf{M}$ be a random variable that denotes the hidden rating matrix (to be estimated) and is uniformly drawn from $\matSet$. Denote the \emph{success} event of estimation of rating matrix by $S$, which is given by 
\begin{align}
	S \coloneqq
	\bigcap_{\substack{X \in \matSet \\ X \neq \gtMat }}
	\left(\mathsf{L}(X) > \mathsf{L}(\gtMat)\right).
	\label{eq:defS}
\end{align}
From the definition of worst-case probability of error in \eqref{eq:errorProb_wc_app}, we obtain 
\begin{align}
	\inf_{\psi} P_e^{(\tau)} (\psi) 
	&=
	\inf_{\psi} \max_{M \in \matSet} \mathbb{P}\left[\psi(Y^{\Omega}, G) \neq M\right]
	\nonumber\\
	&\geq
	\inf_{\psi} \max_{M \in \matSet} \mathbb{P}\left[\psi(Y^{\Omega}, G) \neq M, \: \mathbf{M} = M\right]
	\nonumber\\
	&=
	\inf_{\psi} \max_{M \in \matSet} \mathbb{P}\left[\psi(Y^{\Omega}, G) \neq M \:\vert\: \mathbf{M} = M\right]
	\label{eq:lowerB_Pe_M_cond}
	\\
	&=
	\inf_{\psi} \max_{M \in \matSet} \sum_{X \neq M} \mathbb{P}\left[\psi(Y^{\Omega}, G) = X \:\vert\: \mathbf{M} = M\right]
	\nonumber\\
	&\geq
	\inf_{\psi} \max_{M \in \matSet} \sum_{\substack{X \in \matSet \\ X \neq M}} \mathbb{P}\left[\psi(Y^{\Omega}, G) = X \:\vert\: \mathbf{M} = M\right]
	\nonumber\\
	&=
	\max_{M \in \matSet} \sum_{\substack{X \in \matSet \\ X \neq M}} \mathbb{P}\left[\estML(Y^{\Omega}, G) = X \:\vert\: \mathbf{M} = M\right]
	\label{eq:lowerB_Pe_noInf}
	\\
	&\geq
	\sum_{\substack{X \in \matSet \\ X \neq \gtMat}} \mathbb{P}\left[\estML(Y^{\Omega}, G) = X \:\vert\: \mathbf{M} = \gtMat\right]
	\label{eq:lowerB_Pe_M_noMax}
	\\
	&=
	\sum_{\substack{X \in \matSet \\ X \neq \gtMat}}
	\mathbb{P}\left[\mathsf{L}(X) \leq \mathsf{L}(\gtMat)\right]
	\label{eq:lowerB_Pe_L}
	\\
	&\geq
	\mathbb{P}\left[\bigcup_{\substack{X \in \matSet \\ X \neq \gtMat }}
	\left(\mathsf{L}(X) \leq \mathsf{L}(\gtMat)\right)\right]
	\label{eq:lowerB_Pe_union}
	\\
	&=
	\mathbb{P}\left[S^c\right]
	\label{eq:lowerB_Pe}
\end{align} 
where \eqref{eq:lowerB_Pe_M_cond} follows because $\mathbf{M}$ is uniformly distributed; \eqref{eq:lowerB_Pe_noInf} follows due to the fact that the maximum likelihood estimator is optimal under uniform prior; \eqref{eq:lowerB_Pe_M_noMax} follows since $\gtMat \in \matSet$; \eqref{eq:lowerB_Pe_L} follows by the definition of negative log-likelihood in \eqref{eq:neg_log_lik}; \eqref{eq:lowerB_Pe_union} follows from union bound; and finally \eqref{eq:lowerB_Pe} follows from \eqref{eq:defS}. Therefore, in order to show that $\lim_{n \rightarrow \infty} \inf_{\psi} P_e^{(\tau)} (\psi) \neq 0$, it suffices to show that $\lim_{n \rightarrow \infty} \mathbb{P} \left[S\right] = 0$.

Next, we show that $\lim_{n \rightarrow \infty} \mathbb{P} \left[S\right] = 0$ under each of the three conditions stated in \eqref{eq:necCond_1}, \eqref{eq:necCond_2}, \eqref{eq:necCond_3}, respectively. Before delving into the convergence proof, we present the following key lemma that is essential for developing the convergence analysis. In this lemma, we use  $B^{(\mu)}$ to refer to a Bernoulli random variable with  (fixed or asymptotic) parameter $\mu\in[0,1]$, that is,  $\mathbb{P}[B^{(\mu)}=1] = 1- \mathbb{P}[B^{(\mu)}=0] = \mu$. 
\begin{lemma} 
	Assume that $\alpha, \beta, \gamma, p = \Theta\left(\frac{\log n}{n}\right)$ and $\theta\in[0,1]$ is a constant. For positive integers $n_1, n_2, n_3, n_4$ satisfying $\max \left\{p n_1, \sqrt{\alpha \beta} n_2, \sqrt{\alpha \gamma} n_3, \sqrt{\beta \gamma} n_4 \right\} = \omega(1)$, consider the sets of independent Bernoulli random variables $\{B_i^{(p)}: i\in[n_1]\}$, $\{B_i^{(\theta)}: i\in[n_1]\}$, $\{B_i^{(\alpha)}: i\in[n_1+1:n_3] \}$, $\{B_i^{(\beta)}: i\in[n_1+1:n_2] \cup [n_1+n_2+n_3+1:n_1+n_2+n_3+n_4\}$, and $\{B_i^{(\gamma)}: i\in[n_1+n_2+1:n_1+n_2+n_3+n_4]\}$. Define 
	\begin{align*}    
		B(n_1,n_2,n_3,n_4) &\coloneqq
		\sum_{i=1}^{n_1} \log\left(\frac{1-\theta}{\theta}\right) B_i^{(p)} \left(2 B_i^{(\theta)} \!\!-\!\! 1\right)
		+  \sum_{j=n_1+1}^{n_1+n_2} \log\left(\frac{(1-\beta)\alpha}{(1-\alpha)\beta}\right)  \left( B_j^{(\beta)} \!\!-\!\! B_j^{(\alpha)} \right)\nonumber\\
		&\phantom{\coloneqq} 
		+  \sum_{k=n_1+n_2+1}^{n_1+n_2+n_3} \log\left(\frac{(1-\gamma)\alpha}{(1-\alpha)\gamma}\right)  \left( B_k^{(\gamma)} \!\!-\!\! B_k^{(\alpha)} \right)\nonumber\\
		&\phantom{\coloneqq} 
		+  \sum_{\ell=n_1+n_2+n_3+1}^{n_1+n_2+n_3+n_4} \log\left(\frac{(1-\gamma)\beta}{(1-\beta)\gamma}\right)  \left( B_\ell^{(\gamma)} \!\!-\!\! B_\ell^{(\beta)} \right).
	\end{align*}
	
	Then, the probability that $B(n_1,n_2,n_3,n_4)$ being non-negative can be lower bounded by 
	\begin{align}
		\mathbb{P}\left[ B(n_1,n_2,n_3,n_4) \geq 0
		\right] 
		\geq 
		\frac{1}{2} \exp \bigl( - (1\!+\!o(1)) \left(n_1 I_r \!+\! n_2 \Ig \!+\! n_3 \IcOne \!+\! n_4 \IcTwo\right) \bigr).
		\label{eq:lowerB_prob}
	\end{align}
	\label{lm:lowerB_prob}
\end{lemma}
We refer to Appendix~\ref{app:lowerB_prob_proof} for the proof of Lemma~\ref{lm:lowerB_prob}.

\paragraph{Failure Proof for the Perfect Clustering/Grouping Regime.}
\label{sec:necCond_1}
Let $\Tau_\ell$ be  a section of columns of $\gtMat$ with $\tau_\ell = |\Tau_\ell|/m = \Theta(1)$, and assume $\ell = b_1b_2b_3\in \{0,1\}^3$. For $c\in \Tau_\ell$, define $\M{c}$ be a rating matrix, which is identical to $\gtMat$, except its $c^{\text{th}}$ column which is given by 
\begin{align*}
	\begin{bmatrix}
		\mathbf{0}_{\frac{n}{6}} &  b_1 \mathbf{1}_{\frac{n}{6}} & b_1\mathbf{1}_{\frac{n}{6}}
		& b_2\mathbf{1}_{\frac{n}{6}} & b_3\mathbf{1}_{\frac{n}{6}} & (b_2\oplus b_3)\mathbf{1}_{\frac{n}{6}}
	\end{bmatrix}.
\end{align*} 
We focus on the family of rating matrices $\{\M{c}: c\in \Tau_\ell\}$. It is easy to verify that the type of all such matrices is given by
\begin{align}
	T &=
	\left(
	\left\{\kUsrs{i}{j}{x}{y} = 0 \right\}_{i,j \in [3], \: x,y \in \{A,B\}}, 
	\left\{\dElmnts{i}{\ell}{A} = 1\right\}_{i\in \{1,3\}, \ell \in \{0,1\}^3}, 
	\left\{\dElmnts{2}{\ell}{A} = 0\right\}_{\ell \in \{0,1\}^3}, \right. \nonumber\\  
	&\phantom{=}\phantom{((}
	\left. 
	\left\{\dElmnts{i}{\ell}{B} = 0\right\}_{i\in [3], \ell \in \{0,1\}^3}
	\right).
	\label{eq:case1:type}
\end{align}
Using the definition of the negative log-likelihood in \eqref{eq:neg_log_lik} for $\M{c}$ with $c\in \Tau_\ell$, we obtain
\begin{align}
	\mathbb{P} \left[ \mathsf{L}(\M{c}) > \mathsf{L}(\gtMat) \right]
	&=
	1 - \mathbb{P} \left[ \mathsf{L}(\M{c}) \leq \mathsf{L}(\gtMat)\right]
	\nonumber\\
	&=
	1 - \mathbb{P} \left[\log\left(\frac{1-\theta}{\theta}\right) \sum_{i=1}^{\numDiffElmnt{M_c}{\gtMat}} B_i^{(p)} \left(2 B_i^{(\theta)} - 1\right) \geq 0\right]
	\nonumber\\
	&=
	1 - \mathbb{P} \left[\log\left(\frac{1-\theta}{\theta}\right) \sum_{i=1}^{\frac{n}{3}} B_i^{(p)} \left(2 B_i^{(\theta)} - 1\right) \geq 0\right]
	\label{eq:1stTermX_prob_cond1_eta}
	\\
	&\leq
	1- \frac{1}{4} \exp \left(- (1+o(1))\frac{n}{3} I_r\right)
	\label{eq:1stTermX_prob_cond1_lm}
	\\
	&\leq
	\exp \left(- \frac{1}{4} \exp \left(- (1+o(1)) \frac{n}{3} I_r\right)\right),
	\label{eq:1stTermX_prob_cond1}
\end{align}
where \eqref{eq:1stTermX_prob_cond1_eta} follows from the evaluation of $\numDiffElmnt{\M{c}}{\gtMat}$ for the type of $\M{c}$ given in \eqref{eq:case1:type},  and \eqref{eq:1stTermX_prob_cond1_lm}~is an immediate consequence of Lemma~\ref{lm:lowerB_prob} by setting $n_1 = \frac{n}{3}$, $n_2 = n_3 = n_4 = 0$.

Next, we can upper bound the success probability of an ML estimator as 
\begin{align}
	\mathbb{P}[S] \leq \mathbb{P} \left[\bigcap_{c\in \Tau_\ell}
	\left(\mathsf{L}(\M{c}) > \mathsf{L}(\gtMat)\right)\right]
	& =
	\prod_{c\in \Tau_\ell} \mathbb{P} \left[ \mathsf{L}(\M{c}) > \mathsf{L}(\gtMat) \right]
	\label{eq:1stTerm_cond1_indp}
	\\
	& \leq 
	\exp \left(- \frac{1}{4}  \exp \left(- (1+o(1)) \frac{n}{3} I_r\right)\right)^{\tau_\ell m}
	\label{eq:1stTerm_cond1_ident}
	\\
	& = 
	\exp \left(- \frac{1}{4} \tau_{\ell} \exp \left(- (1+o(1)) \frac{n}{3} I_r + \log m\right)\right)
	\nonumber\\
	& \leq 
	\exp \left(- \frac{1}{4} \tau_{\ell}
	\exp \Bigl( -\bigl( (1+o(1)) (1-\epsilon) - 1 \bigr) \log m \Bigr)
	\right)
	\label{eq:1stTerm_cond1_cond}\\
	& \leq 
	\exp \left(- \frac{1}{4} \tau_{\ell}
	\exp \Bigl( \bigl( \epsilon - o(1) (1-\epsilon)  \bigr) \log m \Bigr)
	\right),
\end{align}
where \eqref{eq:1stTerm_cond1_indp} follows from the fact that the events $\{\mathsf{L}(M_c) > \mathsf{L}(\gtMat)\}$ are mutually independent for all $c \in \Tau_\ell$, since each event corresponds to a different column within the block of columns $\Tau_{\ell}$; \eqref{eq:1stTerm_cond1_ident}~follows from \eqref{eq:1stTermX_prob_cond1}; and finally, 
\eqref{eq:1stTerm_cond1_cond} follows from \eqref{eq:necCond_1}. Therefore, we get 
\begin{align}
	\lim_{n,m \rightarrow \infty} \mathbb{P} \left[S\right] 
	&\leq
	\lim_{n,m \rightarrow \infty}  \exp \left(- \frac{1}{4} \tau_{\ell}
	\exp \Bigl( \bigl( \epsilon - o(1) (1-\epsilon)  \bigr) \log m \Bigr)
	\right) = 0,
	\label{eq:lim_Ps_cond1}
\end{align}
which shows that if \eqref{eq:necCond_1} holds, then the recovery fails with high probability. 

\paragraph{Failure Proof for the Grouping-Limited Regime.}
\label{sec:necCond_2}
Without loss of generality, assume $\intraTau m = \hamDist{\vecV{1}{A}}{\vecV{2}{A}}$, i.e., the rating vectors of groups $G_1^A$ and $G_2^A$ that have the minimum inter-group Hamming distance.  In the following, we will introduce a class of rating matrices, which are obtained by switching two users between groups  $G_1^A$ and $G_2^A$, and prove that if \eqref{eq:necCond_2} holds, then with high probability the ML estimator will fail by selecting one of the rating matrices from this class, instead of $\gtMat$. 

First, we present the following lemma that guarantees the existence of two  subsets of users with certain properties. The proof of the lemma is presented in Appendix~\ref{app:randGraph}. 
\begin{lemma} 
	Let $\alpha, \beta = \Theta\left(\frac{\log n}{n}\right)$. Consider  groups $G_1^A$ and $G_2^A$.  As $n\rightarrow \infty$, with probability approaching $1$, there exists two subgroups $\tG1A\subset G_1^A$ and $\tG2A\subset G_2^A$ with size $|\tG1A|\geq \frac{n}{\log^3 n}$ and  $|\tG2A|\geq \frac{n}{\log^3 n}$ such that there is no edge between the nodes in $\tG1A \cup \tG2A$, that is, 
	\begin{align*}
		E \cap \left( (\tG1A \cup \tG2A) \times (\tG1A \cup \tG2A)\right) = \varnothing.
	\end{align*}
	\label{lm:randGraph}
\end{lemma}
For given sub-groups $\tG1A$ and $\tG2A$, we define the set of rating matrices 
\begin{align*}
	\{\M{f,g}: f\in \tG1A, g\in\tG2A \}
\end{align*}
where $\M{f,g}$ is identical to $\gtMat$, except its $f^\text{th}$ and $g^\text{th}$ rows, which are swapped. Note that for every $\M{f,g}$ in this class, we have $\Lambda(\M{f,g},\gtMat)=2\delta_g m$. Moreover, the groups induced by  $\M{f,g}$ are $\hat{G}_{1}^A= G_1^A \cup \{g\} \setminus\{f\}$ and $\hat{G}_{2}^A= G_2^A \cup \{f\} \setminus\{g\}$, while the other four groups are identical to those of matrix $\gtMat$. Therefore, for each $\M{f,g}$ we have 
\vspace{6mm}
\begin{align*}
	&\mathsf{L}(\gtMat) - \mathsf{L}(\M{f,g}) 
	\nonumber\\
	&= \log\left(\frac{1-\theta}{\theta}\right)\sum_{i=1}^{2\delta_g m}  B_i^{(p)} \left(2 B_i^{(\theta)} - 1\right)\nonumber\\
	&\phantom{=}+ \log\left(\frac{(1-\beta)\alpha}{(1-\alpha)\beta}\right)
	\left[\sum_{h\in G_1^A\setminus\{f\}} \left(B_{(g,h)}^{(\beta)} - B_{(f,h)}^{(\alpha)}\right) + \sum_{h\in G_2^A \setminus\{g\}} \left( B_{(f,h)}^{(\beta)} - B_{(g,h)}^{(\alpha)}\right)
	\right]\nonumber\\
	&= \log\left(\frac{1-\theta}{\theta}\right)\sum_{i=1}^{2\delta_g m}  B_i^{(p)} \left(2 B_i^{(\theta)} - 1\right) 
	+ \log\left(\frac{(1-\beta)\alpha}{(1-\alpha)\beta}\right) \sum_{j=1}^{2(\frac{n}{6}-1)}
	\left( B_{j}^{(\beta)} - B_{j}^{(\alpha)}\right)\nonumber\\
	&= B\left(2\delta_g m, 2(\frac{n}{6}-1), 0, 0\right)
\end{align*}
Then, using Lemma~\ref{lm:lowerB_prob}, we can write
\begin{align}
	\mathbb{P} \left[
	\mathsf{L}(\M{f,g}) > \mathsf{L}(\gtMat)\right]
	&=
	1 - \mathbb{P}\left[ B(2\delta_g m, 2(\frac{n}{6}-1), 0, 0) \geq 0\right] 
	\nonumber
	\\
	&\leq 
	1 - \frac{1}{4} \exp \left( - (1 + o(1)) \left(2 \intraTau m I_r + 2\left(\frac{n}{6}-1\right) \Ig \right) \right)\nonumber
	\\
	&\leq
	\exp
	\left(
	- \frac{1}{4} \exp \left( - (1+o(1)) \left(2\intraTau m I_r + 2\left(\frac{n}{6}-1\right) \Ig \right) \right)
	\right).
	\label{eq:2nd-L-diff}
\end{align}

Finally, we can bound the success probability of an ML estimator as 
\begin{align}
	\mathbb{P}[S] &\leq \mathbb{P} \left[\bigcap_{f\in \tG1A, g\in \tG2A}
	\left(\mathsf{L}(\M{f,g}) > \mathsf{L}(\gtMat)\right)\right]
	=
	\prod_{f\in \tG1A, g\in \tG2A} \mathbb{P} \left[ \mathsf{L}(\M{f,g}) > \mathsf{L}(\gtMat) \right]
	\label{eq:2nd-indep}
	\\
	& \leq 
	\left(\exp
	\left(
	- \frac{1}{4} \exp \left( - (1+o(1)) \left(2\intraTau m I_r + 2\left(\frac{n}{6}-1\right) \Ig \right) \right)
	\right) \right)^{\left|\tG1A\right|\cdot \left|\tG2A\right|}
	\label{eq:2nd-eval}
	\\
	& = 
	\exp
	\left( -\frac{n^2}{4\log^6 (n)}
	\exp \left( - (1+o(1)) \left(2\intraTau m I_r + 2\left(\frac{n}{6}-1\right) \Ig \right) \right)
	\right) 
	\label{eq:2nd-SetSize}\\
	& \leq 
	\exp
	\left( -\frac{n^2}{4 \log^6 (n)}
	\exp \left( - 2(1+o(1)) (1-\epsilon) \log n \right) \right)
	\label{eq:2d-BadCond}\\
	& \leq 
	\exp \left(-\frac{n^{2(\epsilon - o(1)(1-\epsilon))}}{4 \log^6 (n)}
	\right),
\end{align}
where \eqref{eq:2nd-indep} holds since events $\{\mathsf{L}(\M{f,g}) > \mathsf{L}(\gtMat)\}$ are independent due to the fact that there is no edge between nodes in $\tG1A \cup \tG2A$; \eqref{eq:2nd-eval} follows from \eqref{eq:2nd-L-diff}; we used $|\tG1A| =|\tG2A|=\frac{n}{\log^3 n}$ in \eqref{eq:2nd-SetSize}; and  \eqref{eq:2d-BadCond} follows from the condition in \eqref{eq:necCond_2}. Finally, we obtain 
\begin{align*}
	\lim_{n\rightarrow \infty } \mathbb{P}[S] \leq 
	\lim_{n\rightarrow \infty } 
	\exp \left( -\frac{n^{2(\epsilon - o(1)(1-\epsilon))}}{4 \log^6 (n)}
	\right) =0,
\end{align*}
which implies that the ML estimator will fail in finding $\gtMat$ with high probability. 

\paragraph{Failure Proof for the Clustering-Limited Regime.}
\label{sec:necCond_3}
The proof of this case follows the same structure as that of the grouping-limited regime. Without loss of generality, assume $\vecV{1}{A}$ and $\vecV{2}{B}$ be rating vectors whose minimum hamming distance is $\interTau m$, i.e., $ \hamDist{\vecV{1}{A}}{ \vecV{2}{B}} = \interTau m$. Note that the corresponding groups defined by such rating vectors, $\grpG{1}{A}$ and $\grpG{2}{B}$, belong to different clusters. Similar to Lemma~\ref{lm:randGraph}, we pick subsets $\tG1A\subset G_1^A$ and $\tG2B \subset G_2^B$ with $|\tG1A| = |\tG2B| = \frac{n}{\log^3 n}$. Note that the subgraph induced by $\tG1A \cup \tG2B$ is edge-free. Then, we consider the set of all rating matrices 
\begin{align*}
	\{\M{f,g}: f\in \tG1A, g\in \tG2B\},
\end{align*} 
where 
\begin{align*}
	\M{f,g}(r,:) = \left\{\begin{array}{cc}
		\gtMat(g,:) & \textrm{if $r=f$},\\
		\gtMat(f,:) & \textrm{if $r=g$},\\
		\gtMat(r,:) & \textrm{otherwise}.
	\end{array}
	\right.
\end{align*}
Then, for $\M{f,g}$, we have 
\begin{align*}
	&\mathsf{L}(\gtMat) - \mathsf{L}(\M{f,g}) 
	\nonumber\\
	&= \log\left(\frac{1-\theta}{\theta}\right)\sum_{i=1}^{\Lambda(\M{f,g}, \gtMat)}  B_i^{(p)} \left(2 B_i^{(\theta)} - 1\right)\nonumber\\
	&\phantom{=}+ \log\left(\frac{(1-\gamma)\alpha}{(1-\alpha)\gamma}\right)
	\left[\sum_{h\in G_1^A\setminus\{f\}} \left(B_{(g,h)}^{(\gamma)} - B_{(f,h)}^{(\alpha)}\right) + \sum_{h\in G_2^B \setminus\{g\}} \left( B_{(f,h)}^{(\gamma)} - B_{(g,h)}^{(\alpha)}\right)
	\right]\nonumber\\
	&\phantom{=}+ \log\left(\frac{(1-\gamma)\beta}{(1-\beta)\gamma}\right)
	\left[\sum_{h\in G_2^A\cup G_3^A} \left(B_{(g,h)}^{(\gamma)} - B_{(f,h)}^{(\beta)}\right) + \sum_{h\in G_1^B \cup G_3^B} \left( B_{(f,h)}^{(\gamma)} - B_{(g,h)}^{(\beta)}\right)
	\right]\nonumber\\
	&= B\left(2\delta_c m, 0, 2(\frac{n}{6}-1), \frac{2n}{3}\right).
\end{align*}
Applying Lemma~\ref{lm:lowerB_prob}, we get
\begin{align}
	\mathbb{P} \left[
	\mathsf{L}(\M{f,g}) > \mathsf{L}(\gtMat)\right]
	&=
	1 - \mathbb{P}\left[ B(2\delta_g m, 2(\frac{n}{6}-1), 0, 0) \geq 0\right] 
	\nonumber
	\\
	&\leq
	\exp
	\left(
	- \frac{1}{4} \exp \left( - (1+o(1)) \left(2\interTau m I_r + 2\left(\frac{n}{6}-1\right) \IcOne + 2\frac{n}{3} \IcTwo \right) \right)
	\right)\!.
	\label{eq:3rd-L-diff}
\end{align}
Therefore, the success probability of the ML estimator can be bounded as 
\begin{align}
	\mathbb{P}[S] &\leq \prod_{f\in \tG1A, g\in \tG2B} \mathbb{P} \left[
	\mathsf{L}(\M{f,g}) > \mathsf{L}(\gtMat)\right]\label{eq:3rd-indep}\\
	&\leq   \left(\exp
	\left(
	- \frac{1}{4} \exp \left( - (1+o(1)) \left(2\interTau m I_r + 2\left(\frac{n}{6}-1\right) \IcOne + 2\frac{n}{3} \IcTwo \right) \right)
	\right)\right)^{\left|\tG1A\right|\cdot \left|\tG2A\right|}
	\label{eq:3rd-eval}
	\\
	& \leq 
	\exp
	\left( -\frac{n^2}{4 \log^6 (n)}
	\exp \left( - 2(1+o(1)) (1-\epsilon) \log n \right) \right)
	\label{eq:3rd-BadCond}\\
	& \leq 
	\exp \left(-\frac{n^{2(\epsilon - o(1)(1-\epsilon))}}{4 \log^6 n}
	\right),
\end{align}
where \eqref{eq:3rd-indep} is a consequence of independence of the events $\{\mathsf{L}(\M{f,g}) > \mathsf{L}(\gtMat)\}$; \eqref{eq:3rd-eval} follows from \eqref{eq:3rd-L-diff}; and in \eqref{eq:3rd-BadCond} we have used the condition \eqref{eq:necCond_3}. This immediately implies 
\begin{align*}
	\lim_{n\rightarrow \infty } \mathbb{P}[S] =0,
\end{align*} 
which leads to the failure of the ML estimator. 

Since $\lim_{n \rightarrow \infty} \mathbb{P} \left[S\right] = 0$ is proved under each of the three conditions stated in \eqref{eq:necCond_1}, \eqref{eq:necCond_2}, and \eqref{eq:necCond_3}, the converse proof of Theorem~\ref{Thm:p_star} is concluded.
$\hfill \square$

\section{Proof of Theorem~\ref{Thm:Alg_theory}}
We propose a computationally feasible matrix completion algorithm that achieves the optimal sample complexity characterized by Theorem~\ref{Thm:p_star}. It consists of four phases described as below. 

\textbf{Phase~1 (Exact Recovery of Clusters):}
We use the community detection algorithm in \cite{abbe2015exact} on $\mathcal{G}$ to \emph{exactly}\footnote{Exact recovery requires the number of wrongly clustered users vanishes as the number of users tends to infinity. The formal mathematical definition is given in \cite[Definition~4]{abbe2017community}.} recover the two clusters $A$ and $B$. As proved in \cite{abbe2015exact}, the decomposition of the graph into two clusters is correct with high probability when $\IcTwo > \frac{2 \log n}{n}$. This completes Phase~$1$.

\textbf{Phase~2 (Almost Exact Recovery of Groups):}
The goal of Phase~$2$ is to decompose the set of users in cluster $A$ (or cluster $B$) into three groups, represented by $\grpG{1}{A}$, $\grpG{2}{A}$, $\grpG{3}{A}$ (or $\grpG{1}{B}$, $\grpG{2}{B}$, $\grpG{3}{B}$). It is worth noting that grouping at this stage is \emph{almost exact}\footnote{Almost exact recovery means that groups are recovered with a vanishing \emph{fraction} of misclassified users. The mathematical definition is given in \cite[Definition~4]{abbe2017community}.}, and will be further refined in the next phases.
To this end, we run a spectral clustering algorithm \cite{gao2017achieving} on $A$ and $B$ separately. Let $\grpGhat{i}{x}{0}$ denote the initial estimate of the $i^{\text{th}}$ group of cluster $x$ that is recovered by Phase~$2$, for $i\in[3]$ and $x\in\{A,B\}$. It is shown that the groups within each cluster are recovered with a vanishing fraction of errors if $\Ig = \omega(1/n)$. It is worth mentioning that there are other clustering algorithms \cite{shi2000normalized,ng2002spectral,abbe2015community,chin2015stochastic,lei2015consistency,javanmard2016phase,krzakala2013spectral,mossel2016density} that can be employed for this phase. Examples include: spectral clustering \cite{shi2000normalized,ng2002spectral,abbe2015community,chin2015stochastic,lei2015consistency}, semidefinite programming (SDP) \cite{javanmard2016phase}, non-backtracking matrix spectrum \cite{krzakala2013spectral}, and belief propagation \cite{mossel2016density}. This completes Phase~$2$.

\textbf{Phase~3 (Exact Recovery of Rating Vectors)}:
We propose a novel algorithm that optimally recovers the rating vectors of the groups within each cluster. The algorithm is based on maximum likelihood~(ML) decoding of users' ratings based on the partial and noisy observations. For this model, the ML decoding boils down to a counting rule: for each item, find the group with the maximum gap between the number of observed zeros and ones, and set the rating entry of this group to $0$. The other two rating vectors are either both $0$ or both $1$ for this item, which will be determined based on the majority of the union of their observed entries. It turns out that the vector recovery is exact with probability $1\!-\!o(1)$. We first present the proposed algorithm. Then, the theoretical guarantee of the algorithm is provided. 

Define $\vecVhat{i}{x}$ as the estimated rating vector of $\vecV{i}{x}$, i.e., the output of Algorithm~\ref{algo:phase3}. Let the $c^{\text{th}}$ element of the rating vector $\vecV{i}{x}$ (or $\vecVhat{i}{x}$) be denoted by $\vecVElmnt{i}{x}{c}$ (or $\vecVhatElmnt{i}{x}{c}$) for $i \in [3]$, $x \in \{ A, B\}$ and $c \in [m]$. Let $Y_{r,c}$ be the $(r,c)$-entry of matrix $Y$, and $Z_{r,c}$ be its mapping to $\{+1,0,-1\}$ for $r \in [n]$ and $c \in [m]$.

\begin{remark}
	Algorithm~\ref{algo:phase3} is one of the technical distinctions, relative to the prior works~\cite{ahn2018binary,yoon2018joint} which employ the simple majority voting rule under non-hierarchical SBMs. Also our technical novelty in analysis, reflected in~\eqref{eq:poe_phase3_lineaD} (see below), exploits the hierarchical structure to prove the theoretical guarantee. $\hfill \blacksquare$
\end{remark}
Let us now prove the exact recovery of the rating vectors of the groups within cluster $A$. The proof w.r.t. cluster $B$ follows by symmetry. Without loss of generality, assume that $\vecVElmnt{1}{A}{c} = 0$ for $c \in [m/2]$, and $\vecVElmnt{1}{A}{c} = 1$ for $c \in m \setminus [m/2]$. In what follows, we will prove that $\vecV{1}{A}$ can be exactly recovered, i.e., $\mathbb{P}\left[\vecVhat{1}{A} = \vecV{1}{A}\right] = 1 - o(1)$. Similar proofs can be constructed for $\vecV{2}{A}$ and $\vecV{3}{A}$. 
The probability of error in recovering $\vecV{1}{A}$ is expressed as
\begin{align}
	&\mathbb{P}\left[\vecVhat{1}{A} \neq \vecV{1}{A}\right]
	\nonumber\\
	& = \mathbb{P}\left[\left( \bigcup_{c \in [m/2]} \{ \vecVhatElmnt{1}{A}{c} = 1 \} \right)
	\cup 
	\left(\bigcup_{c \in m \setminus [m/2]} \{ \vecVhatElmnt{1}{A}{c} = 0 \} \right)
	\right]
	\nonumber\\
	& \leq \left(\sum_{c \in [m/2]} \mathbb{P}\left[\vecVhatElmnt{1}{A}{c} = 1\right]\right)
	+ \left(\sum_{c \in m \setminus [m/2]} \mathbb{P}\left[\vecVhatElmnt{1}{A}{c} = 0\right]\right)
	\label{eq:poe_phase3_unionB}\\
	& =
	\left(\sum_{c \in [m/2]} \mathbb{P} \left[\bigcap_{i \in [2]} \left\{ \vecVhatElmnt{i}{A}{c} = 1\right \} \cap \left \{ \vecVhatElmnt{3}{A}{c} = 0 \right \} \right] 
	+ \mathbb{P}\left[\bigcap_{i \in \{1,3\}} \left \{ \vecVhatElmnt{1}{A}{c} = 1 \right \} \cap \left \{ \vecVhatElmnt{2}{A}{c} = 0 \right \} \right]
	\right)
	\nonumber\\
	& \phantom{=} + \left(\sum_{c \in m \setminus [m/2]} \mathbb{P}\!\left[ \bigcap_{i \in [3]} \left \{ \vecVhatElmnt{i}{A}{c} = 0 \right \} \right] 
	+ \mathbb{P}\left[\left \{ \vecVhatElmnt{1}{A}{c} = 0 \right \} \cap \bigcap_{i \in \{2,3\}} \left \{ \vecVhatElmnt{i}{A}{c} = 1 \right \} \right]
	\right)
	\label{eq:poe_phase3_lineaD}\\
	& \leq \left(\sum_{c \in [m/2]} \mathbb{P}\left[ \deltaDiff{1}{A}{c} + \deltaDiff{2}{A}{c} \:\leq\: 0\right]
	+ \sum_{c \in [m/2]} \mathbb{P}\left[\deltaDiff{1}{A}{c} + \deltaDiff{3}{A}{c} \:\leq\: 0\right]\right)
	\nonumber\\
	& \phantom{=} + \left(\sum_{c \in m \setminus [m/2]} \mathbb{P}\left[ \deltaDiff{2}{A}{c} + \deltaDiff{3}{A}{c} \:\geq\: 0\right]
	+ \sum_{c \in m \setminus [m/2]} \mathbb{P}\left[\deltaDiff{2}{A}{c} + \deltaDiff{3}{A}{c} \:\geq\: 0\right]\right)
	\label{eq:poe_phase3_delta}\\
	& = \left(\sum_{c \in [m/2]} \underbrace{\mathbb{P}\left[\sum_{r_1 \in\grpGhat{1}{A}{0}} \!\!\!\!\!\!\Zelmnt{r_1}{c} + \sum_{r_2 \in\grpGhat{2}{A}{0}} \!\!\!\!\!\!\Zelmnt{r_2}{c} \:\leq\: 0\right]}_{\mathsf{Term_1}}
	+ \sum_{c \in [m/2]} \underbrace{\mathbb{P}\left[\sum_{r_1 \in\grpGhat{1}{A}{0}} \!\!\!\!\!\!\Zelmnt{r_1}{c} + \sum_{r_3 \in\grpGhat{3}{A}{0}} \!\!\!\!\!\!\Zelmnt{r_3}{c} \:\leq\: 0\right]}_{\mathsf{Term_2}}
	\right)
	\nonumber\\
	& \phantom{=} + \left(
	\sum_{c \in m \setminus [m/2]} \!\!\!\!\underbrace{\mathbb{P}\left[\sum_{r_2 \in\grpGhat{2}{A}{0}} \!\!\!\!\!\!\Zelmnt{r_2}{c} + \!\!\!\!\!\sum_{r_3 \in\grpGhat{3}{A}{0}} \!\!\!\!\!\!\Zelmnt{r_3}{c} \:\geq\: 0\right]}_{\mathsf{Term_3}} 
	+ \!\!\!\sum_{c \in m \setminus [m/2]} \!\!\!\! \underbrace{\mathbb{P}\left[\sum_{r_2 \in\grpGhat{2}{A}{0}} \!\!\!\!\!\!\Zelmnt{r_2}{c} + \!\!\!\!\!\sum_{r_3 \in\grpGhat{3}{A}{0}} \!\!\!\!\!\!\Zelmnt{r_3}{c} \:\geq\: 0\right]}_{\mathsf{Term_4}} 
	\right)
	\label{eq:poe_phase3_z}
\end{align}
where \eqref{eq:poe_phase3_unionB} follows from the union bound; \eqref{eq:poe_phase3_lineaD} follows from $\vecV{1}{A} \oplus \vecV{2}{A} = \vecV{3}{A}$; \eqref{eq:poe_phase3_delta} follows from the ML decoding outlined in Algorithm~\ref{algo:phase3}; and \eqref{eq:poe_phase3_z} follows from the definition of $\deltaDiff{i}{x}{c}$ on Line~3 in Algorithm~\ref{algo:phase3}.

Next we show that each of the four terms in \eqref{eq:poe_phase3_z} is $o(m^{-1})$. We prove that for $\mathsf{Term_1}$ and $\mathsf{Term_3}$, and similar proofs can be carried out for $\mathsf{Term_2}$ and $\mathsf{Term_4}$. Define $\misClassUsr{i} \coloneqq \grpGhat{i}{A}{0} \setminus \grpG{i}{A}$ and $\misClassFrac{i} \coloneqq \left|\misClassUsr{i}\right| / n$. From the theoretical guarantees (i.e., exact clustering and almost-exact grouping) in Phases 1 and 2, we have $\lim_{n \rightarrow \infty} \misClassFrac{i} = 0, \: \forall i \in [3]$ with high probability. Define $n_{i1} \coloneqq \left(\frac{1}{6}-\misClassFrac{i}\right)n$ and $n_{i2} \coloneqq \misClassFrac{i}n$ for $i \in [3]$. Let $\{ B_i^{(p)} \} \substack{\text{i.i.d.} \\ \sim} \:\: \mathsf{Bern}(p)$, and $\{B_i^{(\theta)} \} \substack{\text{i.i.d.} \\ \sim} \:\: \mathsf{Bern}(\theta)$. Hence, for $c \in [m/2]$, $\mathsf{Term_1}$ can be upper bounded by
\begin{align}
	& \mathbb{P}\left[\sum_{r_1 \in\grpGhat{1}{A}{0}} \Zelmnt{r_1}{c} + \sum_{r_2 \in\grpGhat{2}{A}{0}} \Zelmnt{r_2}{c} \:\leq\: 0\right]
	\nonumber\\
	& = 
	\mathbb{P}\left[
	\sum_{i \in \grpGhat{1}{A}{0} \setminus \misClassUsr{1}} \Zelmnt{i}{c} 
	+ \sum_{j \in \misClassUsr{1}} \Zelmnt{j}{c} 
	+ \sum_{k \in \grpGhat{2}{A}{0} \setminus \misClassUsr{2}} \Zelmnt{k}{c}
	+ \sum_{\ell \in \misClassUsr{2}} \Zelmnt{\ell}{c}
	\:\leq\: 0 \right]
	\nonumber\\
	& \leq
	\mathbb{P}\left[
	\sum_{i \in \grpGhat{1}{A}{0} \setminus \misClassUsr{1}} \Zelmnt{i}{c}
	- \sum_{j \in \misClassUsr{1}} \left| \Zelmnt{j}{c} \right| 
	+ \sum_{k \in \grpGhat{2}{A}{0} \setminus \misClassUsr{2}} \Zelmnt{k}{c}
	- \sum_{\ell \in \misClassUsr{2}} \left| \Zelmnt{\ell}{c} \right|
	\:\leq\: 0 \right]
	\nonumber\\
	& = 
	\mathbb{P}\left[
	- \sum_{i=1}^{n_{11}} B_i^{(p)} \left(2 B_i^{(\theta)} - 1\right)
	- \sum_{j=n_{11}+1}^{n_{11}+n_{12}} B_j^{(p)}
	\right.
	\nonumber\\
	& \left. \phantom{=}
	- \sum_{k=n_{11}+n_{12}+1}^{n_{11}+n_{12}+n_{21}} B_k^{(p)} \left(2 B_k^{(\theta)} - 1\right)
	- \sum_{\ell=n_{11}+n_{12}+n_{21}+1}^{n_{11}+n_{12}+n_{21}+n_{22}} B_\ell^{(p)}
	\:\leq\: 0 \right]
	\label{eq:poe12_bern}\\
	& =
	\mathbb{P}\left[
	\sum_{i=1}^{n_{11}+n_{21}} B_i^{(p)} \left(2 B_i^{(\theta)} - 1\right) \geq 
	- \sum_{j=n_{11}+n_{21}+1}^{n_{11}+n_{21}+n_{12}+n_{22}} B_j^{(p)}
	\right]
	\label{eq:poe12_bern2}
\end{align}
where \eqref{eq:poe12_bern} follows since $\vecVElmnt{1}{A}{c} = 0$ for $c \in [m/2]$,
\begin{align*}
	\Yelmnt{j}{c} =\left\{
	\begin{array}{ll}
		0 & \textrm{w.p. } \: p (1-\theta);\\
		1 & \textrm{w.p. } \: p \theta,
	\end{array}
	\right.
\end{align*}
and $\Zelmnt{j}{c} = -(2\Yelmnt{j}{c}-1)$.

The following lemma introduces a large deviation result employed in \cite{ahn2018binary} to further bound~\eqref{eq:poe12_bern2}.
\begin{lemma}
	\label{lm:largeDev}
	Let $0 < \epsilon < 1$, and $0 < p < 1/2$. Suppose $X \sim \mathsf{Binom}(\epsilon n, p)$. Then, 
	\begin{align}
		\mathbb{P} \left[ X \geq \frac{\kappa n p}{\log (1/\epsilon)} \right]
		\leq 2 \exp \left(- \frac{\kappa n p}{2}\right), \;\; \text{ for any } \kappa \geq 2e.
	\end{align}
\end{lemma}
\begin{proof}
	The proof is given by \cite[Lemma~7]{ahn2018binary}.
\end{proof}

Let $\kappa$ be sufficiently large such that $\kappa > 4e$. Thus, the RHS of \eqref{eq:poe12_bern2} can be upper bounded by
\begin{align}
	& \mathbb{P}\left[\sum_{r_1 \in\grpGhat{1}{A}{0}} \Zelmnt{r_1}{c} + \sum_{r_2 \in\grpGhat{2}{A}{0}} \Zelmnt{r_2}{c} \:\leq\: 0\right]
	\nonumber\\
	& \leq
	\mathbb{P}\left[
	\sum_{i=1}^{n_{11}+n_{21}} B_i^{(p)} \left(2 B_i^{(\theta)} - 1\right) \geq 
	- \sum_{j=n_{11}+n_{21}+1}^{n_{11}+n_{21}+n_{12}+n_{22}} B_j^{(p)}
	\right]
	\nonumber\\
	& \leq
	\mathbb{P}\left[
	\sum_{i=1}^{n_{11}+n_{21}}  B_i^{(p)} \left(2 B_i^{(\theta)} - 1\right) 
	\geq - \frac{\kappa np}{\log \frac{1}{\misClassFrac{1}+\misClassFrac{2}}} 
	\right]
	+ \mathbb{P}\left[
	- \sum_{j=n_{11}+n_{21}+1}^{n_{11}+n_{21}+n_{12}+n_{22}} B_j^{(p)}
	\leq - \frac{\kappa np}{\log \frac{1}{\misClassFrac{1}+\misClassFrac{2}}} 
	\right]
	\nonumber\\
	& \leq
	\mathbb{P}\left[
	\sum_{i=1}^{n_{11}+n_{21}}  B_i^{(p)} \left(2 B_i^{(\theta)} - 1\right) 
	\geq - \frac{\kappa n p}{\log \frac{1}{\misClassFrac{1}+\misClassFrac{2}}} 
	\right]
	+ 2 \exp \left(- \frac{\kappa n p}{2}\right)
	\label{eq:lrgDev_lm1}\\
	& \leq
	\mathbb{P}\left[
	\log \left(\frac{1-\theta}{\theta}\right)\!\!\sum_{i=1}^{n_{11}+n_{21}}  B_i^{(p)} \left(2 B_i^{(\theta)} - 1\right) 
	\geq - \log \left(\frac{1-\theta}{\theta}\right) \frac{cnp}{\log \frac{1}{\misClassFrac{1}+\misClassFrac{2}}} 
	\right]
	+ o(m^{-1})
	\label{eq:lrgDev_approx}\\
	& \leq 
	\exp \left(\frac{1}{2} \log \left(\frac{1-\theta}{\theta}\right) \frac{cnp}{\log \frac{1}{\misClassFrac{1}+\misClassFrac{2}}} 
	- (1+o(1)) \left(\frac{1}{3}-\left(\misClassFrac{1}+\misClassFrac{2}\right)\right) n I_r
	\right)
	+ o(m^{-1})
	\label{eq:lrgDev_boundAchv}\\
	& \approx
	\exp \left(
	- (1+o(1)) \left(\frac{1}{3}-\left(\misClassFrac{1}+\misClassFrac{2}\right)\right) n I_r
	\right)
	+ o(m^{-1})
	\label{eq:lrgDev_approxEta}\\
	& \leq
	\exp \left(
	- (1+o(1)) \left(1+\frac{\epsilon}{4}\right) \log m
	\right) 
	+ o(m^{-1}) 
	\label{eq:lrgDev_achvCond}\\
	& = o(m^{-1})
	\label{eq:lrgDev_last12}
\end{align}
where \eqref{eq:lrgDev_lm1} follows from Lemma~\ref{lm:largeDev}; \eqref{eq:lrgDev_approx} follows since $n p = \Omega(\log m)$; \eqref{eq:lrgDev_boundAchv} readily follows from Lemma 2; \eqref{eq:lrgDev_approxEta} follows as the first term in the exponent is insignificant compared to the other term since $n p = \Theta (nI_r)$ and $\lim_{\misClassFrac{1},\misClassFrac{2} \rightarrow 0^+} \frac{1}{\log \frac{1}{\misClassFrac{1}+\misClassFrac{2}}} = 0$; and \eqref{eq:lrgDev_achvCond}~follows since $\frac{1}{3}nI_r \geq (1+\epsilon) \log m$ guarantees that $\left(\frac{1}{3}\!-\!\left(\misClassFrac{1}\!+\!\misClassFrac{2}\right)\right) n I_r \geq \left(1\!+\!\frac{\epsilon}{4}\right) \log m$ as long as $\left(\misClassFrac{1}\!+\!\misClassFrac{2}\right)$ is sufficiently small compared to $\epsilon$.

Similarly, for $c \in m \setminus [m/2]$, $\mathsf{Term_3}$ can be upper bounded by
\begin{align}
	& \mathbb{P}\left[\sum_{r_2 \in\grpGhat{2}{A}{0}} \Zelmnt{r_2}{c} + \sum_{r_3 \in\grpGhat{3}{A}{0}} \Zelmnt{r_3}{c} \:\geq\: 0\right]
	\nonumber\\
	& = 
	\mathbb{P}\left[
	\sum_{i \in \grpGhat{2}{A}{0} \setminus \misClassUsr{2}} \Zelmnt{i}{c} 
	+ \sum_{j \in \misClassUsr{2}} \Zelmnt{j}{c} 
	+ \sum_{k \in \grpGhat{3}{A}{0} \setminus \misClassUsr{3}} \Zelmnt{k}{c}
	+ \sum_{\ell \in \misClassUsr{3}} \Zelmnt{\ell}{c}
	\:\geq\: 0 \right]
	\nonumber\\
	& \leq
	\mathbb{P}\left[
	\sum_{i \in \grpGhat{2}{A}{0} \setminus \misClassUsr{2}} \Zelmnt{i}{c}
	+ \sum_{j \in \misClassUsr{2}} \left| \Zelmnt{j}{c} \right| 
	+ \sum_{k \in \grpGhat{3}{A}{0} \setminus \misClassUsr{3}} \Zelmnt{k}{c}
	+ \sum_{\ell \in \misClassUsr{3}} \left| \Zelmnt{\ell}{c} \right|
	\:\geq\: 0 \right]
	\nonumber\\
	& = 
	\mathbb{P}\left[
	\sum_{i=1}^{n_{21}} B_i^{(p)} \left(2 B_i^{(\theta)} - 1\right)
	+ \sum_{j=n_{21}+1}^{n_{21}+n_{22}} B_j^{(p)}
	\right.
	\nonumber\\
	& \left. \phantom{=}
	+ \sum_{k=n_{21}+n_{22}+1}^{n_{21}+n_{22}+n_{31}} B_k^{(p)} \left(2 B_k^{(\theta)} - 1\right)
	+ \sum_{\ell=n_{21}+n_{22}+n_{31}+1}^{n_{21}+n_{22}+n_{31}+n_{32}} B_\ell^{(p)}
	\:\geq\: 0 \right]
	\label{eq:poe23_bern}\\
	& =
	\mathbb{P}\left[
	\sum_{i=1}^{n_{21}+n_{31}} B_i^{(p)} \left(2 B_i^{(\theta)} - 1\right) \geq 
	- \sum_{j=n_{21}+n_{31}+1}^{n_{21}+n_{31}+n_{22}+n_{32}} B_j^{(p)}
	\right]
	\label{eq:poe23_bern2}
\end{align}
where \eqref{eq:poe23_bern} follows since $\vecVElmnt{1}{A}{c} = 1$ for $c \in m \setminus [m/2]$, 
\begin{align*}
	\Yelmnt{j}{c} =\left\{
	\begin{array}{ll}
		0 & \textrm{w.p. } \: p \theta;\\
		1 & \textrm{w.p. } \: p (1-\theta),
	\end{array}
	\right.
\end{align*}
and $\Zelmnt{j}{c} = -(2\Yelmnt{j}{c}-1)$. Applying similar bounding techniques used for \eqref{eq:poe12_bern}, one can show that
\begin{align}
	& \mathbb{P}\left[\sum_{r_2 \in\grpGhat{2}{A}{0}} \Zelmnt{r_2}{c} + \sum_{r_3 \in\grpGhat{3}{A}{0}} \Zelmnt{r_3}{c} \:\geq\: 0\right]
	\leq o(m^{-1}). 
	\label{eq:lrgDev_last23}
\end{align}

Finally, by \eqref{eq:lrgDev_last12} and \eqref{eq:lrgDev_last23}, the probability of error in recovering $\vecV{1}{A}$ is upper bounded by
\begin{align}
	& \mathbb{P}\left[\vecVhat{1}{A} \neq \vecV{1}{A}\right]
	\nonumber\\
	& \leq \left(\sum_{c \in [m/2]} o(m^{-1})
	+ \sum_{c \in [m/2]} o(m^{-1})
	\right)
	+ \left(
	\sum_{c \in m \setminus [m/2]} o(m^{-1})
	+ \sum_{c \in m \setminus [m/2]} o(m^{-1})
	\right)
	\nonumber\\
	& = o(1).
\end{align}
This completes the proof of exact recovery of rating vectors.

\textbf{Phase~4 (Exact Recovery of Groups):}
The goal in this last step is to \emph{refine} the groups which are \emph{almost recovered} in Phase~$2$, thereby obtaining an \emph{exact} grouping. To this end, we propose an iterative algorithm that locally refines the estimates on the user grouping within each cluster for $T$ iterations. More specifically, at each iteration, the affiliation of each user is updated to the group that yields the maximum point-wise likelihood w.r.t. the considered user. The exact computation of the point-wise likelihood requires the knowledge of the model parameters $(\alpha, \beta, \theta)$.  But we do not rely on such knowledge, instead estimate them using the given ratings and graph $(Y, {\cal G})$. Hence, we use an \emph{approximated} point-wise log-likelihood which can readily be computed as:
\begin{align}
	\label{eq:pointwiseL}    
	|\{c\!:Y_{r,c}=\vecVhat{i}{x}(c)\}| \cdot \log \left(\frac{1\!-\!\estPrmtr{\theta}}{\estPrmtr{\theta}}\right) 
	+ e\left(\{r\}, \grpGhat{i}{x}{t-1}\right) \cdot\log\! \left(\frac{(1\!-\!\estPrmtr{\beta}) \estPrmtr{\alpha}}{(1\!-\!\estPrmtr{\alpha})
		\estPrmtr{\beta}}\right)
\end{align}
where $(\estPrmtr{\alpha}, \estPrmtr{\beta}, \estPrmtr{\theta})$ denote the maximum likelihood estimates of $(\alpha, \beta, \theta)$. Here $|\{c\!:Y_{r,c}=\vecVhat{i}{x}(c)\}|$ indicates the number of observed rating matrix entries of the user that coincide with the corresponding entries of the rating vector of that group; and $e\left(\{r\}, \grpGhat{i}{x}{t-1}\right)$ denotes the number of edges between the user and the set of users which belong to that group. The pseudocode is described in Algorithm~\ref{algo:phase4_groups}. 

In order to prove that Algorithm~\ref{algo:phase4_groups} ensures the exact recovery of groups, we intend to show that the number of misclassified users in each cluster strictly decreases with each iteration. To this end, we rely on a technique that was employed in many relevant papers~\cite{chen2016community, ahn2018binary,yoon2018joint}. The technique aims to prove that the misclassification error rate is reduced by a factor of 2 with each iteration. More specifically, assuming that the previous phases are executed successfully, if we start with $\eta n$ misclassified users within one cluster, for some small $\eta > 0$, then it intends to show that we end up with $\frac{\eta}{2} n$ misclassified users with high probability as $n \rightarrow \infty$ after one iteration of refinement. Hence, with this technique, running the local refinement for $T = \frac{\log (\eta n)}{\log 2}$ within the groups of each cluster would suffice to converge to the ground truth assignments. The proof of such error rate reduction follows the one in~\cite[Theorem~2]{yoon2018joint} in which the problem of recovering $K$ communities of possibly different sizes is studied. By considering the case of three equal-sized communities, the guarantees of exact recovery of the groups within each cluster readily follows when $T = O(\log n)$. 

This completes the proof of Phase~$4$, and concludes the proof of Theorem~\ref{Thm:Alg_theory}.
$\hfill \square$

\newpage
\section{Supplementary Experimental Results}
Similar to~\cite{ahn2018binary, yoon2018joint,tan2019community,jo2020discrete}, the performance of the proposed algorithm is assessed on semi-real data (real graph but synthetic rating vectors). We consider a subgraph of the political blog network \cite{adamic2005political}, which is shown to exhibit a hierarchical structure \cite{peixoto2014hierarchical}. In particular, we consider a tall matrix setting of $n = 381$ and $m = 200$ in order to investigate the gain in sample complexity due to the graph side information. The selected subgraph consists of two clusters of political parties, each of which comprises three groups. The three groups of the first cluster consist of $98$, $34$ and $103$ users, while the three groups of the second cluster consist of $58$, $68$ and $20$ users.

In order to visualize the underlying hierarchical structure of the considered subgraph of the political blog network, we apply a dimensionality reduction algorithm, called t-Distributed Stochastic Neighbor Embedding (t-SNE) \cite{vanDerMaaten2008tsne} to visualize high-dimensional data in a low-dimensional space. Fig.~\ref{fig:poliblog_tsne} shows two clusters that are colored in \textcolor[rgb]{1,0,0}{red} and \textcolor[rgb]{0,0,1}{blue}. Each cluster comprises three groups, represented by circle, triangle and square. 

\begin{figure*}[h]
	\centering
	\includegraphics[width=0.55\textwidth]{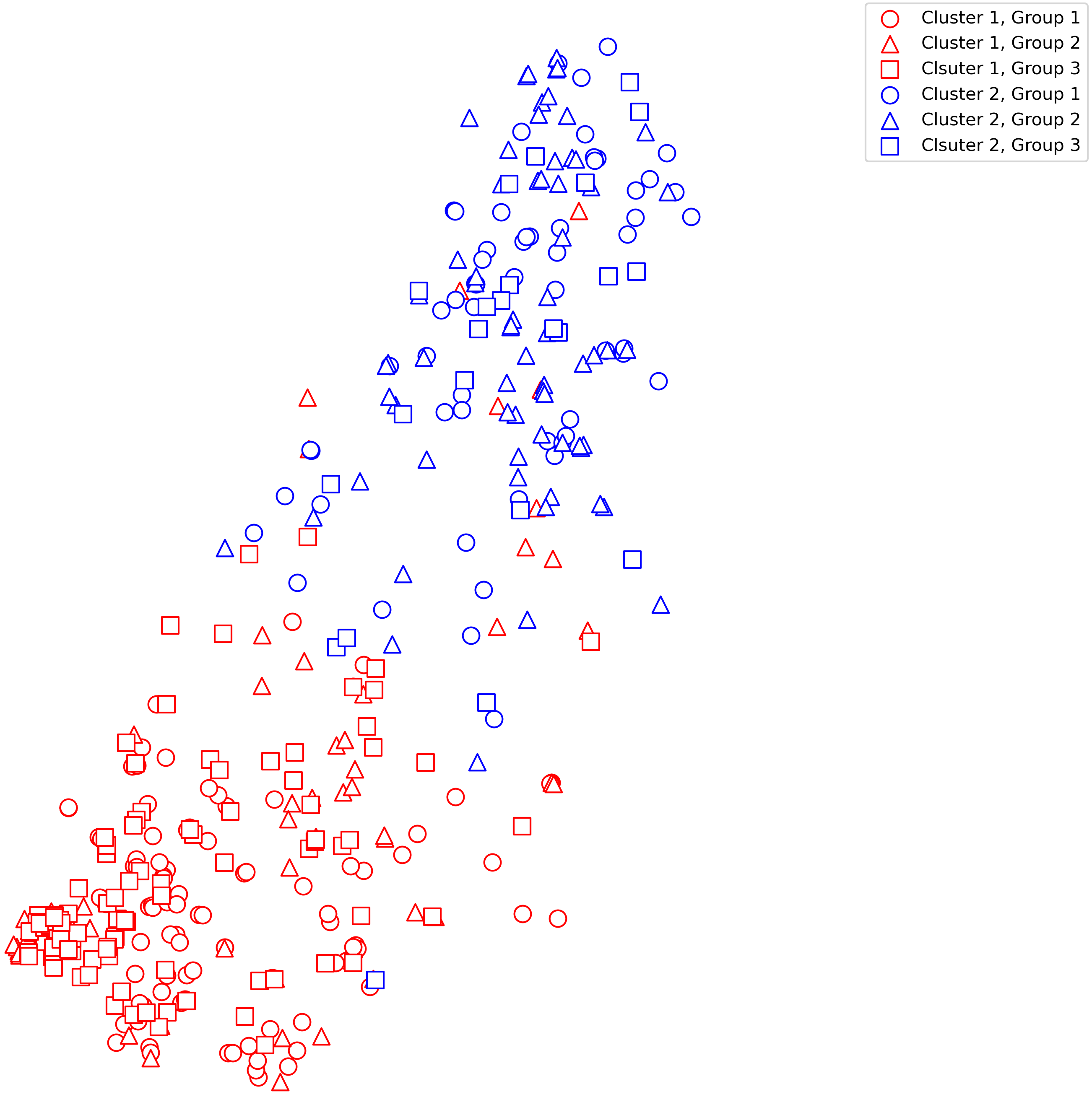}
	\caption{Visualization of a subgraph of the political blog network~\cite{adamic2005political} using t-SNE algorithm \cite{vanDerMaaten2008tsne}.}
	\label{fig:poliblog_tsne}
\end{figure*}

\newpage
\appendix
\section{Proofs of Lemmas for Achievability Proof of Theorem~\ref{Thm:p_star}}
\subsection{Proof of Lemma~\ref{lm:neg_log_likelihood}}
\label{app:neg_log_likelihood}
The negative log-likelihood of a candidate rating matrix $X=(\cV,\mathcal{Z})$, for $X \in \matSet$, given a fixed input pair $(Y,\mathcal{G})$ can be written as
\begin{align}
	\mathsf{L}(X) 
	&= 
	- \log \mathbb{P}\left[(Y,\mathcal{G}) \given \mathbf{X} = X \right]
	\nonumber\\
	&=
	- \log \left(\mathbb{P}\left[Y \given \mathbf{X} = X \right]  
	\:\mathbb{P} \left[\mathcal{G} \given \mathbf{X} = X \right]\right)
	\nonumber\\
	&=
	- \log \mathbb{P}\left[Y \given \mathbf{X} = X \right] 
	- \log \mathbb{P}\left[\mathcal{G} \given \mathbf{X} = X \right],
\end{align}
where
\begin{align}
	\mathbb{P}\left[Y \given \mathbf{X} = X \right] 
	&= 
	p^{|\Omega|}\:
	(1-p)^{nm-|\Omega|}
	\left(\theta\right)^{\numDiffElmnt{Y}{X}}
	(1-\theta)^{|\Omega| - \numDiffElmnt{Y}{X}}, 
	\\
	\mathbb{P}\left[\mathcal{G} \given \mathbf{X} = X \right]
	&=
	\alphaEdge^{\numEdge{\alpha}{\mathcal{G}}{Z}} 
	(1-\alphaEdge)^{ \vert \userPairSet{\alpha}{\mathcal{Z}} \vert - \numEdge{\alpha}{\mathcal{G}}{Z}} \:
	\betaEdge^{\numEdge{\beta}{\mathcal{G}}{Z}} 
	(1-\betaEdge)^{\vert \userPairSet{\beta}{\mathcal{Z}} \vert - \numEdge{\beta}{\mathcal{G}}{Z}} \:
	\nonumber \\
	&\phantom{=\{} 
	\gammaEdge^{\numEdge{\gamma}{\mathcal{G}}{Z}} 
	(1-\gammaEdge)^{ \vert \userPairSet{\gamma}{\mathcal{Z}} \vert - \numEdge{\gamma}{\mathcal{G}}{Z}}.
\end{align}
Consequently, $\mathsf{L}(X)$ is given by
\begin{align}
	\mathsf{L}(X) 
	=  
	\log \left( \frac{1-\theta}{\theta}\right) \numDiffElmnt{Y}{X} 
	+
	\sum\limits_{\mu\in\left\{\alphaEdge,\betaEdge,\gammaEdge\right\}}
	\left(
	\log \left(\frac{1-\mu}{\mu}\right) \numEdge{\mu}{\mathcal{G}}{\mathcal{Z}}
	-
	\log(1-\mu) \left\vert \userPairSet{\mu}{\mathcal{Z}} \right\vert 
	\right).
	\label{eq:neg_log_lik}
\end{align}
This completes the proof of Lemma~\ref{lm:neg_log_likelihood}.
\hfill $\blacksquare$

\subsection{Proof of Lemma~\ref{lemma:neg_lik_M0_XT}}
\label{proof:neg_lik_M0_XT}
By Lemma~\ref{lm:neg_log_likelihood}, the LHS of \eqref{eq:lemma3} can be written as
\begin{align}
	\mathsf{L}\left(\gtMat\right) - \mathsf{L}(X)
	&=
	\log\left(\frac{1-\theta}{\theta}\right) 
	\underbrace{\left(\numDiffElmnt{Y}{\gtMat} - \numDiffElmnt{Y}{X}\right)}_{\mathsf{Term_1}}
	\nonumber \\
	&\phantom{=}
	+ 
	\underbrace{
		\sum\limits_{\mu\in\left\{\alphaEdge,\betaEdge,\gammaEdge\right\}}
		\left(
		\log \left(\frac{1-\mu}{\mu}\right)
		\left(
		\numEdge{\mu}{\mathcal{G}}{\partitionSet_0}
		- \numEdge{\mu}{\mathcal{G}}{\partitionSet}
		\right)
		\right)}_{\mathsf{Term_2}}
	\nonumber\\ 
	&\phantom{=}
	+ 
	\underbrace{
		\sum\limits_{\mu \in \left\{\alphaEdge,\betaEdge,\gammaEdge\right\}}
		\left(
		\log(1-\mu) 
		\left(
		\left\vert \userPairSet{\mu}{\partitionSet} \right\vert
		-
		\left\vert  \userPairSet{\mu}{\partitionSet_0} \right\vert
		\right)
		\right)}_{\mathsf{Term_3}}.
	\label{eq:LX_LM0}
\end{align}
In what follows, we evaluate each of the three terms in \eqref{eq:LX_LM0}.

Recall that $\numDiffElmnt{A}{B}$ denotes the number of different entries between matrices $A_{n \times m}$ and $B_{n \times m}$. Therefore, $\mathsf{Term_1}$ can be expanded as
\begin{align}
	\mathsf{Term_1}
	&= 
	\numDiffElmnt{Y}{\gtMat} - \numDiffElmnt{Y}{X}
	\nonumber\\
	&= 
	\sum_{(r,t)\in \Omega} 
	\left(\indicatorFn{Y(r,t) \neq \gtMat(r,t)}\right)
	- \sum_{(r,t)\in \Omega} 
	\left(\indicatorFn{Y(r,t) \neq X(r,t)}\right)
	\nonumber\\ 
	&=
	\sum_{(r,t)\in \Omega} 
	\left(nm - \indicatorFn{Y(r,t) = \gtMat(r,t)} \right)
	- \left(nm - \indicatorFn{Y(r,t) = X(r,t)} \right)
	\nonumber\\
	&= 
	\sum_{\substack{(r,t) \in \Omega : \\ X(r,t) \neq \gtMat(r,t)}} 
	\indicatorFn{Y(r,t) = X(r,t)} 
	- \indicatorFn{Y(r,t) = \gtMat(r,t)}
	\label{eq:phiX_0}\\ 
	&= \sum_{i \in \lvert\left\{(r,t)\in [n] \times [m] \::\: X(r,t) \neq \gtMat(r,t)\right\}\rvert} 
	\mathsf{B}^{(p)}_i \mathsf{B}_i^{\left(\theta\right)} 
	- \mathsf{B}^{(p)}_i\left(1-\mathsf{B}_i^{(\theta)}\right)
	\label{eq:phiX_00}\\ 
	&= \sum_{i \in \diffEntriesSet}
	\mathsf{B}_i^{(p)}
	\left(2\mathsf{B}_i^{(\theta)} - 1\right), 
	\label{eq:phiX_3}
\end{align}
where \eqref{eq:phiX_0} follows since $\indicatorFn{Y(i,j) = X(i,j)} = \indicatorFn{Y(i,j) = \gtMat(i,j)}$ if $X(i,j) = \gtMat(i,j)$; the first term of each summand in \eqref{eq:phiX_00} follows since the probability that the observed rating matrix entry is $X(i,j)$, which is not equal to $\gtMat(i,j)$, is $p(\theta/)$, while the second term of each summand in \eqref{eq:phiX_00} follows since the probability that the observed rating matrix entry is $\gtMat(i,j)$ is $p(1-\theta)$, for every $(i,j)\in [n] \times [m]$; and finally \eqref{eq:phiX_3} follows from \eqref{eq:N1_defn}.

Next, we evaluate $\mathsf{Term_2}$. We first evaluate the quantity $\numEdge{\alpha}{\mathcal{G}}{\partitionSet_0} - \numEdge{\alpha}{\mathcal{G}}{Z}$ as
\begin{align}
	&\numEdge{\alpha}{\mathcal{G}}{\partitionSet_0} - \numEdge{\alpha}{\mathcal{G}}{Z} 
	\nonumber\\
	&=
	\left\vert
	\left\{
	(a,b) \in \mathcal{E}: 
	a \in Z_0(x,i) \cap Z(y,j_1),\: b \in Z_0(x,i) \cap Z(y,j_2),
	\right.\right. \nonumber \\
	&\phantom{=\{} \left. \left.
	\mbox{ for } x, y \in \{A,B\},\: i, j_1, j_2 \in [3],\: j_1 \neq j_2
	\right\}
	\right\vert
	\nonumber\\
	&\phantom{=} +
	\left\vert
	\left\{
	(a,b) \in \mathcal{E}: 
	a \in Z_0(x,i) \cap Z(y_1,j_1),\: b \in Z_0(x,i) \cap Z(y_2, j_2),
	\right.\right. \nonumber \\
	&\phantom{=\{\vert} \left. \left.
	\mbox{ for } x, y \in \{A,B\},\: y_1 \neq y_2, \:
	i, j_1, j_2 \in [3]
	\right\}
	\right\vert
	\nonumber\\
	&\phantom{=}
	-\left\vert
	\left\{
	(a,b) \in \mathcal{E}: 
	a \in Z_0(x,i_1) \cap Z(y,j),\: b \in Z_0(x,i_2) \cap Z(y,j),
	\right.\right. \nonumber \\
	&\phantom{=\{\vert} \left. \left.
	\mbox{ for } x, y \in \{A,B\},\: i_1, i_2, j \in [3],\: i_1 \neq i_2
	\right\}
	\right\vert
	\nonumber\\
	&\phantom{=} - 
	\left\vert
	\left\{
	(a,b) \in \mathcal{E}: 
	a \in Z_0(x_1,i_1) \cap Z(y,j),\: b \in Z_0(x_2,i_2) \cap Z(y,j),
	\right.\right. \nonumber \\
	&\phantom{=\{\vert} \left. \left.
	\mbox{ for } x_1, x_2, y \in \{A,B\},\: x_1 \neq x_2, \:
	i_1, i_2, j, \in [3]
	\right\}
	\right\vert
	\label{eq:diff_edge0}\\
	&=
	\sum\limits_{i=1}^{\misClassfSet{\alphaEdge}{\betaEdge}} \mathsf{B}_i^{(\alphaEdge)}
	+\sum\limits_{i=1}^{\misClassfSet{\alphaEdge}{\gammaEdge}} \mathsf{B}_i^{(\alphaEdge)}
	-\sum\limits_{i=1}^{\misClassfSet{\betaEdge}{\alphaEdge}} \mathsf{B}_i^{(\betaEdge)} 
	- \sum\limits_{i=1}^{\misClassfSet{\gammaEdge}{\alphaEdge}} \mathsf{B}_i^{(\gammaEdge)},
	\label{eq:diff_edge1}
\end{align}
where \eqref{eq:diff_edge0} holds since the edges that remain after the subtraction are: (i) edges that exist in the same group in $\partitionSet_0$, but are estimated to be in different groups within the same cluster in $\partitionSet$; (ii) edges that exist in the same group in $\partitionSet_0$, but are estimated to be in different clusters in $\partitionSet$; (iii) edges that exist in different groups within the same cluster in $\partitionSet_0$, but are estimated to be in the same group in $\partitionSet$; (iv) edges that exist in different clusters in $\partitionSet_0$, but are estimated to be in the same group in $\partitionSet$; and finally \eqref{eq:diff_edge1} follows from \eqref{eq:N2_defn}--\eqref{eq:N3rev_defn}. In a similar way, one can evaluate the following quantities:
\begin{align}
	\numEdge{\betaEdge}{\mathcal{G}}{\partitionSet_0} - \numEdge{\betaEdge}{\mathcal{G}}{\partitionSet}
	&=
	\sum\limits_{i=1}^{\misClassfSet{\betaEdge}{\alphaEdge}} \mathsf{B}_i^{(\betaEdge)}
	+\sum\limits_{i=1}^{\misClassfSet{\betaEdge}{\gammaEdge}} \mathsf{B}_i^{(\betaEdge)}
	-\sum\limits_{i=1}^{\misClassfSet{\alphaEdge}{\betaEdge}} \mathsf{B}_i^{(\alphaEdge)} 
	- \sum\limits_{i=1}^{\misClassfSet{\gammaEdge}{\betaEdge}} \mathsf{B}_i^{(\gammaEdge)},
	\label{eq:diff_edge2}\\
	\numEdge{\gammaEdge}{\mathcal{G}}{\partitionSet_0} - \numEdge{\gammaEdge}{\mathcal{G}}{\partitionSet}
	&=
	\sum\limits_{i=1}^{\misClassfSet{\gammaEdge}{\alphaEdge}} \mathsf{B}_i^{(\gammaEdge)}
	+\sum\limits_{i=1}^{\misClassfSet{\gammaEdge}{\betaEdge}} \mathsf{B}_i^{(\gammaEdge)}
	-\sum\limits_{i=1}^{\misClassfSet{\alphaEdge}{\gammaEdge}} \mathsf{B}_i^{(\alphaEdge)} 
	- \sum\limits_{i=1}^{\misClassfSet{\betaEdge}{\gammaEdge}} \mathsf{B}_i^{(\betaEdge)}.
	\label{eq:diff_edge3}
\end{align}
Consequently, $\mathsf{Term_2}$ can be written as
\begin{align}
	\mathsf{Term_2} 
	&=
	\sum\limits_{\mu\in{\alphaEdge,\betaEdge,\gammaEdge}}
	\left(
	\log \left(\frac{1-\mu}{\mu}\right)
	\left(
	\numEdge{\mu}{\mathcal{G}}{\partitionSet_0}
	- \numEdge{\mu}{\mathcal{G}}{\partitionSet}
	\right)
	\right)
	\nonumber\\
	&=
	\log \left(\frac{1-\alphaEdge}{\alphaEdge}\right)
	\left(
	\sum_{i \in \misClassfSet{\alphaEdge}{\betaEdge}} 
	\mathsf{B}_i^{(\alphaEdge)}
	+ \sum_{i \in \misClassfSet{\alphaEdge}{\gammaEdge}} 
	\mathsf{B}_i^{(\alphaEdge)}
	- \sum_{i \in \misClassfSet{\betaEdge}{\alphaEdge}} 
	\mathsf{B}_i^{(\betaEdge)}
	- \sum_{i \in \misClassfSet{\gammaEdge}{\alphaEdge}} 
	\mathsf{B}_i^{(\gammaEdge)}
	\right)
	\nonumber\\
	&\phantom{=}
	+ \log \left(\frac{1-\betaEdge}{\betaEdge}\right)
	\left(
	\sum_{i \in \misClassfSet{\betaEdge}{\alphaEdge}} 
	\mathsf{B}_i^{(\betaEdge)}
	+ \sum_{i \in \misClassfSet{\betaEdge}{\gammaEdge}} 
	\mathsf{B}_i^{(\betaEdge)}
	- \sum_{i \in \misClassfSet{\alphaEdge}{\betaEdge}} 
	\mathsf{B}_i^{(\alphaEdge)}
	- \sum_{i \in \misClassfSet{\gammaEdge}{\betaEdge}} 
	\mathsf{B}_i^{(\gammaEdge)}
	\right)
	\nonumber\\
	&\phantom{=}
	+ \log \left(\frac{1-\gammaEdge}{\gammaEdge}\right)
	\left(
	\sum_{i \in \misClassfSet{\gammaEdge}{\alphaEdge}} 
	\mathsf{B}_i^{(\gammaEdge)}
	+ \sum_{i \in \misClassfSet{\gammaEdge}{\betaEdge}} 
	\mathsf{B}_i^{(\gammaEdge)}
	- \sum_{i \in \misClassfSet{\alphaEdge}{\gammaEdge}} 
	\mathsf{B}_i^{(\alphaEdge)}
	- \sum_{i \in \misClassfSet{\betaEdge}{\gammaEdge}} 
	\mathsf{B}_i^{(\betaEdge)}
	\right)
	\label{eq:term2_L(x)_0}\\
	&=
	\left(\log\frac{(1-\betaEdge)\alphaEdge}{(1-\alphaEdge)\betaEdge}\right)
	\left(
	\sum\limits_{i \in \misClassfSet{\betaEdge}{\alphaEdge}} \mathsf{B}_i^{(\betaEdge)}
	- \sum\limits_{i \in \misClassfSet{\alphaEdge}{\betaEdge}} \mathsf{B}_i^{(\alphaEdge)} 
	\right)
	\nonumber\\
	&\phantom{=}
	+
	\left(\log\frac{(1-\gammaEdge)\alphaEdge}{(1-\alphaEdge)\gammaEdge}\right)
	\left(
	\sum\limits_{i \in \misClassfSet{\gammaEdge}{\alphaEdge}} \mathsf{B}_i^{(\gammaEdge)}
	- \sum\limits_{i \in \misClassfSet{\alphaEdge}{\gammaEdge}} \mathsf{B}_i^{(\alphaEdge)} 
	\right)
	\nonumber\\
	&\phantom{=}
	+
	\left(\log\frac{(1-\gammaEdge)\betaEdge}{(1-\betaEdge)\gammaEdge}\right)
	\left(
	\sum\limits_{i \in \misClassfSet{\gammaEdge}{\betaEdge}} \mathsf{B}_i^{(\gammaEdge)}
	- \sum\limits_{i \in \misClassfSet{\betaEdge}{\gammaEdge}} \mathsf{B}_i^{(\betaEdge)} 
	\right),
	\label{eq:term2_L(x)_1}
\end{align}
where \eqref{eq:term2_L(x)_0} follows from \eqref{eq:diff_edge1}--\eqref{eq:diff_edge3}.

Finally, $\mathsf{Term_3}$ is evaluated as follows:
\begin{align}
	\mathsf{Term_3} 
	&=
	\sum\limits_{\mu \in \left\{\alphaEdge,\betaEdge,\gammaEdge\right\}}
	\left(
	\log(1-\mu) 
	\left(
	\left\vert \userPairSet{\mu}{\partitionSet} \right\vert
	-
	\left\vert  \userPairSet{\mu}{\partitionSet_0} \right\vert
	\right)
	\right)
	\nonumber\\
	&=
	\sum\limits_{\mu \in \left\{\alphaEdge,\betaEdge,\gammaEdge\right\}}
	\left(
	\log(1-\mu) 
	\left\vert 
	\bigcup_{\nu \in \left\{\alphaEdge, \betaEdge, \gammaEdge\right\}} \misClassfSet{\nu}{\mu}
	\right\vert
	\right)
	-
	\sum\limits_{\mu \in \left\{\alphaEdge,\betaEdge,\gammaEdge\right\}}
	\left(
	\log(1-\mu)
	\left\vert  
	\bigcup_{\nu \in \left\{\alphaEdge, \betaEdge, \gammaEdge\right\}} \misClassfSet{\mu}{\nu}
	\right\vert
	\right)
	\label{eq:term3_L(x)_0}\\
	&=
	\sum\limits_{\mu \in \left\{\alphaEdge,\betaEdge,\gammaEdge\right\}}
	\left(
	\log(1-\mu) 
	\left( 
	\left\vert\misClassfSet{\alphaEdge}{\mu}\right\vert + \left\vert\misClassfSet{\betaEdge}{\mu}\right\vert + \left\vert\misClassfSet{\gammaEdge}{\mu}\right\vert
	\right)
	\right)
	\nonumber \\
	&\phantom{=}
	-
	\sum\limits_{\mu \in \left\{\alphaEdge,\betaEdge,\gammaEdge\right\}}
	\left(
	\log(1-\mu)
	\left( 
	\left\vert\misClassfSet{\mu}{\alphaEdge}\right\vert + \left\vert\misClassfSet{\mu}{\betaEdge}\right\vert + \left\vert\misClassfSet{\mu}{\gammaEdge}\right\vert
	\right)
	\right)
	\label{eq:term3_L(x)_1}\\
	&=
	\log(1-\alphaEdge) 
	\left( 
	\left\vert\misClassfSet{\betaEdge}{\alphaEdge}\right\vert + \left\vert\misClassfSet{\gammaEdge}{\alphaEdge}\right\vert
	\right)
	-
	\log(1-\alphaEdge)
	\left( 
	\left\vert\misClassfSet{\alphaEdge}{\betaEdge}\right\vert + \left\vert\misClassfSet{\alphaEdge}{\gammaEdge}\right\vert
	\right)
	\nonumber \\
	&\phantom{=}
	+ \log(1-\betaEdge) 
	\left( 
	\left\vert\misClassfSet{\alphaEdge}{\betaEdge}\right\vert +  \left\vert\misClassfSet{\gammaEdge}{\betaEdge}\right\vert
	\right)
	-
	\log(1-\betaEdge)
	\left( 
	\left\vert\misClassfSet{\betaEdge}{\alphaEdge}\right\vert +  \left\vert\misClassfSet{\betaEdge}{\gammaEdge}\right\vert
	\right)
	\nonumber \\
	&\phantom{=}
	+ \log(1-\gammaEdge) 
	\left( 
	\left\vert\misClassfSet{\alphaEdge}{\gammaEdge}\right\vert + \left\vert\misClassfSet{\betaEdge}{\gammaEdge}\right\vert 
	\right)
	-
	\log(1-\gammaEdge)
	\left( 
	\left\vert\misClassfSet{\gammaEdge}{\alphaEdge}\right\vert + \left\vert\misClassfSet{\gammaEdge}{\betaEdge}\right\vert 
	\right)
	\nonumber\\
	&=
	\left(\log\frac{1-\alphaEdge}{1-\betaEdge}\right) 
	\left(
	\left\vert \misClassfSet{\betaEdge}{\alphaEdge}\right\vert
	- \left\vert \misClassfSet{\alphaEdge}{\betaEdge}\right\vert
	\right)
	+  
	\left(\log\frac{1-\alphaEdge}{1-\gammaEdge}\right) 
	\left(
	\left\vert \misClassfSet{\gammaEdge}{\alphaEdge}\right\vert
	- \left\vert \misClassfSet{\alphaEdge}{\gammaEdge}\right\vert
	\right)
	\nonumber \\
	&\phantom{=}
	+  
	\left(\log\frac{1-\betaEdge}{1-\gammaEdge}\right)
	\left(
	\left\vert \misClassfSet{\gammaEdge}{\betaEdge}\right\vert
	- \left\vert \misClassfSet{\betaEdge}{\gammaEdge}\right\vert
	\right),
	\label{eq:term3_L(x)}
\end{align}
where \eqref{eq:term3_L(x)_0} holds since
\begin{align}
	\userPairSet{\mu}{\partitionSet_0} 
	= \bigcup_{\nu \in \left\{\alphaEdge, \betaEdge, \gammaEdge\right\}} \misClassfSet{\mu}{\nu},
	\qquad
	\userPairSet{\mu}{\partitionSet} 
	= \bigcup_{\nu \in \left\{\alphaEdge, \betaEdge, \gammaEdge\right\}} \misClassfSet{\nu}{\mu};
\end{align}
and \eqref{eq:term3_L(x)_1} follows since the sets $\{\misClassfSet{\mu}{\nu} : \mu, \nu \in \{\alphaEdge,\betaEdge,\gammaEdge \}, \: \mu \neq \nu\}$ are disjoint.

By \eqref{eq:phiX_3}, \eqref{eq:term2_L(x)_1} and \eqref{eq:term3_L(x)}, the LHS of \eqref{eq:LX_LM0} is given by
\begin{align}
	\mathsf{L}\left(\gtMat\right) - \mathsf{L}(X)
	&=
	\log\left(\frac{1-\theta}{\theta}\right) 
	\sum_{i \in \diffEntriesSet}
	\mathsf{B}_i^{(p)}
	\left(2\mathsf{B}_i^{(\theta)} - 1\right)
	\nonumber\\
	&\phantom{=}
	+ \left(\log\frac{(1-\betaEdge)\alphaEdge}{(1-\alphaEdge)\betaEdge}\right)
	\left(
	\sum\limits_{i \in \misClassfSet{\betaEdge}{\alphaEdge}} \mathsf{B}_i^{(\betaEdge)}
	- \sum\limits_{i \in \misClassfSet{\alphaEdge}{\betaEdge}} \mathsf{B}_i^{(\alphaEdge)} 
	\right)
	\nonumber \\
	&\phantom{=}
	+
	\left(\log\frac{1-\alphaEdge}{1-\betaEdge}\right) 
	\left(
	\left\vert \misClassfSet{\betaEdge}{\alphaEdge}\right\vert
	- \left\vert \misClassfSet{\alphaEdge}{\betaEdge}\right\vert
	\right)
	\nonumber\\
	&\phantom{=}
	+ \left(\log\frac{(1-\gammaEdge)\alphaEdge}{(1-\alphaEdge)\gammaEdge}\right)
	\left(
	\sum\limits_{i \in \misClassfSet{\gammaEdge}{\alphaEdge}} \mathsf{B}_i^{(\gammaEdge)}
	- \sum\limits_{i \in \misClassfSet{\alphaEdge}{\gammaEdge}} \mathsf{B}_i^{(\alphaEdge)} 
	\right)
	\nonumber \\
	&\phantom{=}
	+  
	\left(\log\frac{1-\alphaEdge}{1-\gammaEdge}\right) 
	\left(
	\left\vert \misClassfSet{\gammaEdge}{\alphaEdge}\right\vert
	- \left\vert \misClassfSet{\alphaEdge}{\gammaEdge}\right\vert
	\right)
	\nonumber\\
	&\phantom{=}
	+ \left(\log\frac{(1-\gammaEdge)\betaEdge}{(1-\betaEdge)\gammaEdge}\right)
	\left(
	\sum\limits_{i \in \misClassfSet{\gammaEdge}{\betaEdge}} \mathsf{B}_i^{(\gammaEdge)}
	- \sum\limits_{i \in \misClassfSet{\betaEdge}{\gammaEdge}} \mathsf{B}_i^{(\betaEdge)} 
	\right)
	\nonumber \\
	&\phantom{=}
	+  
	\left(\log\frac{1-\betaEdge}{1-\gammaEdge}\right)
	\left(
	\left\vert \misClassfSet{\gammaEdge}{\betaEdge}\right\vert
	- \left\vert \misClassfSet{\betaEdge}{\gammaEdge}\right\vert
	\right),
\end{align}
which immediately proves \eqref{eq:lemma3}. This completes the proof of Lemma~\ref{lemma:neg_lik_M0_XT}.
\hfill $\blacksquare$

\subsection{Proof of Lemma~\ref{lemma:upper_bound_B}}
\label{proof:upper_bound_B}
We first define three random variables (indexed by $i$) before proceeding with the proof. Recall from Section~\ref{sec:achv} that $\mathsf{B}_i^{(\sigma)}$ denotes the $i^{\text{th}}$~Bernoulli random variable with parameter $\sigma \in \{p, \theta, \alphaEdge, \betaEdge, \gammaEdge\}$, that is $\mathbb{P}[\mathsf{B}_i^{(\sigma)}=1] = 1- \mathbb{P}[\mathsf{B}_i^{(\sigma)}=0] = \sigma$. For $p=\Theta\left(\frac{\log n}{n}\right)$ and a constant $\theta \in[0,1]$, we define the first random variable $\mathbf{U}_i = \mathbf{U}_i(p,\theta)$ as
\begin{align}
	\mathbf{U}_i(p,\theta) 
	&= \log \left(\frac{1-\theta}{\theta}\right) \mathsf{B}_i^{(p)} \left( 2\mathsf{B}_i^{(\theta)}-1\right)
	\nonumber\\
	&= \left\{
	\begin{array}{ll}
		-\log \left(\frac{1-\theta}{\theta}\right) & \textrm{w.p. } \: p(1-\theta),\\
		0 & \textrm{w.p. } \: (1-p),\\
		\log \left(\frac{1-\theta}{\theta}\right)  & \textrm{w.p. } \: p\theta.
	\end{array}
	\right.
	\label{eq:app_x(p,t,q)}
\end{align}
The moment generating function $M_{\mathbf{U}_i(p,\theta)}\left(t\right)$ of $\mathbf{U}_i(p,\theta,q)$ at $t=1/2$ is evaluated as
\begin{align}
	M_{\mathbf{U}_i(p,\theta)}\left(\frac{1}{2}\right)
	&= 
	\mathbbm{E}\left[\exp\left(\frac{1}{2} \mathbf{U}_i(p,\theta) \right) \right] 
	\nonumber\\
	&=
	\left[
	p(1\!-\!\theta)  \exp\left(\!-\frac{1}{2}\log\left(\!\frac{1\!-\!\theta}{\theta}\right)\!\!\!\right)\!
	\right]
	+ 
	1 \!-\! p
	\nonumber \\
	&\phantom{=}
	+ 
	\left[
	p\theta \exp\left(\frac{1}{2}\log\left(\! \frac{1\!-\!\theta}{\theta}\right)\!\!\!\right)\!
	\right]
	\nonumber \\
	&=
	p\sqrt{\theta(1-\theta)} 
	+ 1-p + p\sqrt{\theta(1-\theta)}   \nonumber \\
	&=
	1-p\left(1-\theta
	-2\sqrt{\theta(1-\theta)}
	+\theta\right)   \nonumber \\
	&=
	1-p\left(\sqrt{1-\theta}-\sqrt{\theta} \right)^2, 
\end{align}
and hence we have 
\begin{align}
	-\log M_{\mathbf{U}_i(p,\theta)}\left(\frac{1}{2}\right) 
	&= - \log \left(
	1-p\left(\sqrt{1-\theta}-\sqrt{\theta} \right)^2
	\right)
	\nonumber\\
	&= p\left(\sqrt{1-\theta}-\sqrt{\theta} \right)^2
	+ O\left(p^2\right)
	\label{eq:taylor_Mx}\\
	&= (1+o(1)) \left(\sqrt{1-\theta} - \sqrt{\theta}\right)^2 p,
	\nonumber\\
	&= (1+o(1)) \Ir,
	\label{eq:log-M-X}
\end{align}
where \eqref{eq:taylor_Mx} follows from Taylor series of $\log (1-x)$, for $x = p\left(\sqrt{1-\theta}-\sqrt{\theta} \right)^2$, which converges since $p=\Theta\left(\frac{\log n}{n}\right)$. Next, for $\mu,\nu = \Theta\left(\frac{\log n}{n}\right)$, define the second random variable $\mathbf{V}_i = \mathbf{V}_i(\mu,\nu)$ as 
\begin{align}
	\mathbf{V}_i(\mu,\nu) 
	&=
	\left(\log\frac{(1-\mu)\nu}{(1-\nu)\mu}\right) \left(\mathsf{B}_i^{(\mu)} - \mathsf{B}_i^{(\nu)}\right)
	=
	\left\{
	\begin{array}{ll}
		-\log\frac{(1-\mu)\nu}{(1-\nu)\mu} & \textrm{w.p. }\: (1-\mu)\nu,\\
		0 & \textrm{w.p. }\: (1-\mu)(1-\nu) + \mu\nu,\\
		\log\frac{(1-\mu)\nu}{(1-\nu)\mu} & \textrm{w.p. }\: \mu(1-\nu).
	\end{array}
	\right.
	\label{eq:app_z(m,n)}
\end{align}
The moment generating function $M_{\mathbf{V}_i(\mu,\nu)}\left(t\right)$ of $\mathbf{V}_i(\mu,\nu)$ at $t=1/2$ is evaluated as
\begin{align}
	M_{\mathbf{V}_i(\mu,\nu)}\left(\frac{1}{2}\right)
	&= 
	\mathbb{E}\left[\exp\left(\frac{1}{2}\mathbf{V}_i(\mu,\nu)\right)\right] 
	\nonumber\\ 
	&= 
	\left[
	(1-\mu)\nu \exp\left(-\frac{1}{2}\log\frac{(1-\mu)\nu}{(1-\nu)\mu} \right) 
	\right]
	+ \left[ 
	(1-\mu)(1-\nu) + \mu\nu 
	\right]
	\nonumber \\
	&\phantom{=}
	+ 
	\left[
	\mu(1-\nu)\exp\left(\frac{1}{2}\log\frac{(1-\mu)\nu}{(1-\nu)\mu} \right) 
	\right]
	\nonumber\\
	&= 
	(1-\mu)\nu \sqrt{\frac{(1-\nu)\mu}{(1-\mu)\nu}}
	+ (1-\mu)(1-\nu) + \mu\nu 
	+ (1-\nu)\mu \sqrt{\frac{(1-\mu)\nu}{(1-\nu)\mu}}
	\nonumber\\
	&= 
	\mu\nu 
	+ 2 \sqrt{(1-\mu)(1-\nu)\mu\nu}
	+ (1-\mu)(1-\nu) 
	\nonumber\\
	&= \left(\sqrt{\mu\nu}+\sqrt{(1-\mu)(1-\nu)}\right)^2,
\end{align}
and thus we have 
\begin{align}
	-\log M_{\mathbf{V}_i(\mu,\nu)}\left(\frac{1}{2}\right) 
	&= -2\log \left(\sqrt{\mu\nu}+\sqrt{(1-\mu)}\sqrt{(1-\nu)}\right)
	\nonumber\\
	&= -2\log \left(
	\sqrt{\mu\nu}
	+ \left(1-\frac{1}{2}\mu+O\left(\mu^2\right)\right) \left(1-\frac{1}{2}\nu+O\left(\nu^2\right)\right)
	\right)
	\label{eq:log-Mz_1}\\
	&= -2\log \left(
	\sqrt{\mu\nu}
	+ \left(1-\frac{1}{2}\mu-\frac{1}{2}\nu+O\left(\mu^2+\nu^2\right)\right)
	\right)
	\nonumber\\
	&=
	-2\log \left(1
	- \left(\frac{1}{2}\mu+\frac{1}{2}\nu-\sqrt{\mu\nu}+O\left(\mu^2+\nu^2\right)\right)
	\right)
	\nonumber\\
	&=
	\left(\sqrt{\mu} - \sqrt{\nu}\right)^2 + O\left(\mu^2+\nu^2\right)
	\label{eq:log-Mz_2}\\
	&= (1+o(1)) \left(\sqrt{\mu} - \sqrt{\nu}\right)^2
	=
	\left\{
	\begin{array}{ll}
		(1+o(1)) \Ig  & \textrm{if }\: \mu = \betaEdge,\: \nu = \alphaEdge,\\
		(1+o(1)) \IcOne & \textrm{if }\: \mu = \gammaEdge,\: \nu = \alphaEdge,\\
		(1+o(1)) \IcTwo & \textrm{if }\: \mu = \gammaEdge,\: \nu = \betaEdge,
	\end{array}
	\right.
	\label{eq:log-M-Y}
\end{align}
where \eqref{eq:log-Mz_1} follows form Taylor series of $\sqrt{1-\mu}$ and $\sqrt{1-\nu}$, which both converge since $\mu,\nu = \Theta\left(\frac{\log n}{n}\right)$; and \eqref{eq:log-Mz_2} follows from Taylor series of $\log (1-x)$, for $x = \frac{1}{2}\mu+\frac{1}{2}\nu-\sqrt{\mu\nu}+O\left(\mu^2+\nu^2\right)$, which also converges for $\mu,\nu = \Theta\left(\frac{\log n}{n}\right)$. Finally, for $\mu,\nu = \Theta\left(\frac{\log n}{n}\right)$, define the third random variable $\mathbf{W}_i = \mathbf{W}_i (\mu,\nu)$ as 
\begin{align}
	\mathbf{W}_i (\mu,\nu) 
	&=
	\left(\log\frac{1-\nu}{1-\mu}\right) + \left(\log\frac{(1-\mu)\nu}{(1-\nu)\mu}\right) \mathsf{B}_i^{(\mu)}
	=
	\left\{
	\begin{array}{ll}
		\log\frac{\nu}{\mu} & \textrm{w.p. }\: \mu,\\
		\log\frac{1-\nu}{1-\mu} & \textrm{w.p. }\: (1-\mu).
	\end{array}
	\right.
	\label{eq:app_w(m,n)}
\end{align}
The moment generating function $M_{\mathbf{W}_i (\mu,\nu)}\left(t\right)$ of $\mathbf{W}_i (\mu,\nu)$ at $t=1/2$ is evaluated as
\begin{align}
	M_{\mathbf{W}_i(\mu,\nu)}\left(\frac{1}{2}\right)
	&= 
	\mathbb{E}\left[\exp\left(\frac{1}{2}\mathbf{W}_i(\mu,\nu)\right)\right] 
	\nonumber\\ 
	&= 
	\left[
	\mu \exp\left(\frac{1}{2}\log\frac{\nu}{\mu} \right) 
	\right]
	+
	\left[
	(1-\mu)\exp\left(\frac{1}{2}\log\frac{1-\nu}{1-\mu} \right) 
	\right]
	\nonumber\\
	&= \sqrt{\mu\nu}+\sqrt{(1-\mu)(1-\nu)},
	\label{eq:M-W}
\end{align}
and hence we have 
\begin{align}
	-\log M_{\mathbf{W}_i(\mu,\nu)}\left(\frac{1}{2}\right) 
	&= -\log \left(\sqrt{\mu\nu}+\sqrt{(1-\mu)}\sqrt{(1-\nu)}\right)
	\nonumber \\
	&= \frac{1}{2} (1+o(1)) \left(\sqrt{\mu} - \sqrt{\nu}\right)^2
	=
	\left\{
	\begin{array}{ll}
		\frac{1}{2} (1+o(1)) \Ig  & \textrm{if }\: \mu = \betaEdge,\: \nu = \alphaEdge,\\
		\frac{1}{2} (1+o(1)) \IcOne  & \textrm{if }\: \mu = \gammaEdge,\: \nu = \alphaEdge,\\
		\frac{1}{2} (1+o(1)) \IcTwo & \textrm{if }\: \mu = \gammaEdge,\: \nu = \betaEdge,
	\end{array}
	\right.
	\label{eq:log-M-W}
\end{align}
where \eqref{eq:log-M-W} follows from \eqref{eq:log-M-Y}. Next, based on the random variables defined in \eqref{eq:app_z(m,n)} and \eqref{eq:app_w(m,n)}, we present the following proposition that will be used in the proof of Lemma~\ref{lemma:upper_bound_B}. The proof of the proposition is presented at the end of this appendix.
\begin{prop}
	\label{prop:A_moment}
	For $\mu,\nu = \Theta\left(\frac{\log n}{n}\right)$, let $\mathbf{A} = \mathbf{A}(\mu,\nu)$ be a random variable that is defined as
	\begin{align}
		\mathbf{A}(\mu,\nu) =
		\left(\log\frac{(1-\mu)\nu}{(1-\nu)\mu}\right)
		\left(
		\sum\limits_{i \in \misClassfSet{\mu}{\nu}} \mathsf{B}_i^{(\mu)}
		- \sum\limits_{i \in \misClassfSet{\nu}{\mu}} \mathsf{B}_i^{(\nu)} 
		\right)
		+
		\left(\log\frac{1-\nu}{1-\mu}\right) 
		\left(
		\left\vert \misClassfSet{\mu}{\nu}\right\vert
		- \left\vert \misClassfSet{\nu}{\mu}\right\vert
		\right),
		\label{eq:prop_A}
	\end{align}
	where $\{\mathsf{B}_i^{(\mu)}: i \in \misClassfSet{\nu}{\mu}\}$  and $\{\mathsf{B}_i^{(\nu)}: i \in \misClassfSet{\nu}{\mu}\}$ are sets of independent and identically distributed Bernoulli random variables. The moment generating function $M_{\mathbf{A}(\mu,\nu)}\left(t\right)$ of $\mathbf{A}(\mu,\nu)$ at $t=1/2$ is given by
	\begin{align}
		M_{\mathbf{A}(\mu,\nu)}\left(t\right)
		&=
		\exp \left( - (1+o(1)) \frac{\left\vert \misClassfSet{\mu}{\nu} \right\vert + \left\vert \misClassfSet{\nu}{\mu} \right\vert}{2} \left(\sqrt{\mu} - \sqrt{\nu}\right)^2 \right).
		\\
		&=
		\left\{
		\begin{array}{ll}
			\exp \left( - (1+o(1)) \Pleftrightarrow{\alphaEdge}{\betaEdge} \: \Ig \right)
			& \textrm{if }\: \mu = \betaEdge,\: \nu = \alphaEdge,\\
			\exp \left( - (1+o(1)) \Pleftrightarrow{\alphaEdge}{\gammaEdge} \: \IcOne \right)
			& \textrm{if }\: \mu = \gammaEdge,\: \nu = \alphaEdge,\\
			\exp \left( - (1+o(1)) \Pleftrightarrow{\betaEdge}{\gammaEdge} \: \IcTwo \right)
			& \textrm{if }\: \mu = \gammaEdge,\: \nu = \betaEdge.
		\end{array}
		\right.
		\label{eq:prop_A_momentGen}
	\end{align}
\end{prop}

Let $\left\{\mathbf{U}_i (p,\theta): i\in \diffEntriesSet \right\}$, and $\{\mathbf{A}(\betaEdge,\alphaEdge), \mathbf{A}(\gammaEdge,\alphaEdge), \mathbf{A}(\gammaEdge,\betaEdge)\}$ be sets of independent and identically distributed random variables defined as per \eqref{eq:app_x(p,t,q)}, and \eqref{eq:prop_A} in Proposition~\ref{prop:A_moment}. Note that the sets $\{\misClassfSet{\mu}{\nu} : \mu, \nu \in \{\alphaEdge,\betaEdge,\gammaEdge \}, \: \mu \neq \nu\}$ are disjoint as per their definitions given by \eqref{eq:N1_defn}--\eqref{eq:N4rev_defn}. Consequently, the LHS of \eqref{eq:lemma1_2} is upper bounded~as
\begin{align}
	\mathbb{P}\left[ \mathbf{B} \geq 0\right]
	&= 
	\mathbb{P}\left[
	\left(\sum_{i\in\diffEntriesSet} 
	\mathbf{U}_i(p,\theta) \right)
	+ \mathbf{A}(\betaEdge,\alphaEdge)
	+ \mathbf{A}(\gammaEdge,\alphaEdge)
	+ \mathbf{A}(\gammaEdge,\betaEdge)
	\geq 0\right]
	\nonumber\\
	&\leq
	\left(M_{\mathbf{U}_i(p,\theta) }\left(\frac{1}{2}\right)\right)^{\lvert\diffEntriesSet\rvert}
	\left(M_{\mathbf{A}(\betaEdge,\alphaEdge)}\left(\frac{1}{2}\right)\right)
	\left(M_{\mathbf{A}(\gammaEdge,\alphaEdge)}\left(\frac{1}{2}\right)\right)
	\left(M_{\mathbf{A}(\gammaEdge,\betaEdge)}\left(\frac{1}{2}\right)\right)
	\label{eq:app_chernoff_1}\\
	&=
	\exp\left(-\left(1+o(1)\right) \left(\lvert\diffEntriesSet\rvert \Ir 
	+ \Pleftrightarrow{\alphaEdge}{\betaEdge} \: \Ig 
	+ \Pleftrightarrow{\alphaEdge}{\gammaEdge} \: \IcOne
	+ \Pleftrightarrow{\betaEdge}{\gammaEdge} \: \IcTwo
	\right)\right), 
	\label{eq:app_chernoff_2}
\end{align}
where \eqref{eq:app_chernoff_1} follows from the Chernoff bound, and the independence of the random variables $\left\{\mathbf{U}_i (p,\theta,q): i\in \diffEntriesSet \right\}$, and $\{\mathbf{A}(\betaEdge,\alphaEdge), \mathbf{A}(\gammaEdge,\alphaEdge), \mathbf{A}(\gammaEdge,\betaEdge)\}$; and finally \eqref{eq:app_chernoff_2} follows from \eqref{eq:log-M-X}, and \eqref{eq:prop_A_momentGen} in Proposition~\ref{prop:A_moment}. This completes the proof of Lemma~\ref{lemma:upper_bound_B}.
\hfill $\blacksquare$

It remains to prove Proposition~\ref{prop:A_moment}. The proof is presented as follows.
\begin{proof}[Proof of Proposition~\ref{prop:A_moment}]
	First, consider the case of $\left\vert \misClassfSet{\mu}{\nu} \right\vert \geq \left\vert \misClassfSet{\nu}{\mu} \right\vert$. Therefore, the random variable $\mathbf{A}(\mu,\nu)$ can be expressed as
	\begin{align}
		\mathbf{A}(\mu,\nu)
		&=
		\sum\limits_{i \in \misClassfSet{\nu}{\mu}} 
		\left(
		\left(\log\frac{(1-\mu)\nu}{(1-\nu)\mu}\right)
		\left(
		\mathsf{B}_i^{(\mu)} - \mathsf{B}_i^{(\nu)} 
		\right)
		\right)
		\nonumber \\
		&\phantom{=}
		+
		\sum\limits_{i \in \misClassfSet{\mu}{\nu} \setminus \misClassfSet{\nu}{\mu}} 
		\left(
		\left(\log\frac{1-\nu}{1-\mu}\right) 
		+
		\left(\log\frac{(1-\mu)\nu}{(1-\nu)\mu}\right)
		\mathsf{B}_i^{(\mu)}
		\right)
		\label{eq:app_A_defn_0}\\
		&=
		\sum\limits_{i \in \misClassfSet{\nu}{\mu}} 
		\mathbf{V}_i
		+
		\sum\limits_{i \in \misClassfSet{\mu}{\nu} \setminus \misClassfSet{\nu}{\mu}} 
		\mathbf{W}_i,
		\label{eq:app_A_defn}
	\end{align}
	where \eqref{eq:app_A_defn_0} holds since the sets $\misClassfSet{\nu}{\mu}$ and $\misClassfSet{\mu}{\nu}$ are disjoint; and \eqref{eq:app_A_defn} follows from \eqref{eq:app_z(m,n)} and \eqref{eq:app_w(m,n)}.
	\begin{align}
		M_{\mathbf{A}(\mu,\nu)}\left(\frac{1}{2}\right)
		&=
		\mathbb{E}\left[\exp\left(\frac{1}{2}\mathbf{A}(\mu,\nu)\right)\right] 
		\nonumber\\ 
		&= 
		\mathbb{E}\left[
		\left(
		\prod_{i \in \misClassfSet{\nu}{\mu}}  
		\exp \left( \frac{1}{2} \mathbf{V}_i(\mu,\nu) \right)
		\right) 
		\left(
		\prod_{i \in \misClassfSet{\mu}{\nu} \setminus \misClassfSet{\nu}{\mu}} 
		\exp \left( \frac{1}{2} \mathbf{W}_i(\mu,\nu) \right)
		\right)
		\right]
		\nonumber\\
		&= 
		\left(
		\prod_{i \in \misClassfSet{\nu}{\mu}}  
		\mathbb{E}\left[
		\exp \left( \frac{1}{2} \mathbf{V}_i(\mu,\nu) \right)
		\right] 
		\right) 
		\left(
		\prod_{i \in \misClassfSet{\mu}{\nu} \setminus \misClassfSet{\nu}{\mu}} 
		\mathbb{E}\left[
		\exp \left( \frac{1}{2} \mathbf{W}_i(\mu,\nu) \right)
		\right] 
		\right)
		\label{eq:prop_A_0}\\
		&=
		\left(
		\prod_{i \in \misClassfSet{\nu}{\mu}}  
		M_{\mathbf{V}_i(\mu,\nu)}\left(\frac{1}{2}\right)
		\right) 
		\left(
		\prod_{i \in \misClassfSet{\mu}{\nu} \setminus \misClassfSet{\nu}{\mu}} 
		M_{\mathbf{W}_i(\mu,\nu)}\left(\frac{1}{2}\right)
		\right)
		\nonumber\\
		&=
		\left(
		\exp \left( - (1+o(1)) \left(\sqrt{\mu} - \sqrt{\nu}\right)^2 \right)
		\right)^{\left\vert \misClassfSet{\nu}{\mu} \right\vert}
		\nonumber \\
		&\phantom{=}
		\left(
		\exp \left( - \frac{1}{2} (1+o(1)) \left(\sqrt{\mu} - \sqrt{\nu}\right)^2 \right)
		\right)^{\left\vert \misClassfSet{\mu}{\nu} \right\vert - \left\vert \misClassfSet{\nu}{\mu} \right\vert}
		\label{eq:prop_A_2}\\
		&=
		\exp \left( - (1+o(1)) \frac{\left\vert \misClassfSet{\mu}{\nu} \right\vert + \left\vert \misClassfSet{\nu}{\mu} \right\vert}{2} \left(\sqrt{\mu} - \sqrt{\nu}\right)^2 \right),
		\label{eq:prop_A_3}
	\end{align}
	where \eqref{eq:prop_A_0} holds since the random variables $\{\mathbf{V}_i : i \in \misClassfSet{\nu}{\mu}\}$ and $\{\mathbf{W}_i : i \in \misClassfSet{\mu}{\nu} \setminus \misClassfSet{\nu}{\mu}\}$ are independent; and \eqref{eq:prop_A_2} follows from \eqref{eq:log-M-Y} and \eqref{eq:log-M-W}.
	
	Next, consider the case of $\left\vert \misClassfSet{\mu}{\nu} \right\vert \leq \left\vert \misClassfSet{\nu}{\mu} \right\vert$. In a similar way, the random variable $\mathbf{A}$ can be written as
	\begin{align}
		\mathbf{A}(\mu,\nu)
		&=
		\sum\limits_{i \in \misClassfSet{\mu}{\nu}} 
		\left(
		\left(\log\frac{(1-\mu)\nu}{(1-\nu)\mu}\right)
		\left(
		\mathsf{B}_i^{(\mu)} - \mathsf{B}_i^{(\nu)} 
		\right)
		\right)
		+
		\nonumber \\
		&\phantom{-}
		\sum\limits_{i \in \misClassfSet{\nu}{\mu} \setminus \misClassfSet{\mu}{\nu}} 
		\left(
		\left(\log\frac{1-\mu}{1-\nu}\right) 
		+
		\left(\log\frac{(1-\nu)\mu}{(1-\mu)\nu}\right)
		\mathsf{B}_i^{(\nu)}
		\right)
		\nonumber\\
		&=
		\sum\limits_{i \in \misClassfSet{\mu}{\nu}} 
		\mathbf{V}_i
		+
		\sum\limits_{i \in \misClassfSet{\nu}{\mu} \setminus \misClassfSet{\mu}{\nu}} 
		\mathbf{W}_i.
	\end{align}
	Following the same procedure presented in the previous case, one can show that $M_{\mathbf{A}(\mu,\nu)}\left(\frac{1}{2}\right)$ is also given by \eqref{eq:prop_A_3} in this case. This completes the proof of Proposition~\ref{prop:A_moment}.
\end{proof}

\subsection{Proof of Lemma~\ref{lemma:T1_related}}
\label{proof:T1_related}
The LHS of \eqref{eq:lemma_T1_1} is given by
\begin{align}
	& \lim_{n\rightarrow \infty}
	\sum\limits_{T \in \TsmallErr} 
	\sum\limits_{X \in \mathcal{X}(T)}
	\exp\left(-\left(1+o(1)\right) \left(\lvert\diffEntriesSet\rvert \Ir 
	+ \Pleftrightarrow{\alphaEdge}{\betaEdge} \: \Ig 
	+ \Pleftrightarrow{\alphaEdge}{\gammaEdge} \: \IcOne 
	+ \Pleftrightarrow{\betaEdge}{\gammaEdge} \: \IcTwo \right)\right) 
	\nonumber \\
	& \qquad =
	\lim_{n\rightarrow \infty}
	\sum\limits_{T \in \TsmallErr} 
	\underbrace{\left\vert \mathcal{X}(T) \right\vert}_{\mathsf{Term_1}}
	\:\:
	\underbrace{\exp\left(-\left(1+o(1)\right) \left(\lvert\diffEntriesSet\rvert \Ir 
		+ \Pleftrightarrow{\alphaEdge}{\betaEdge} \: \Ig 
		+ \Pleftrightarrow{\alphaEdge}{\gammaEdge} \: \IcOne 
		+ \Pleftrightarrow{\betaEdge}{\gammaEdge} \: \IcTwo  \right)\right)}_{\mathsf{Term_2}}.
	\label{eq:largeErr_LHS_app}
\end{align}
In what follows, we derive upper bounds on $\mathsf{Term_1}$ and $\mathsf{Term_2}$ for a fixed non-all-zero tuple $T \in \TsmallErr$ given by
\begin{align}
	T
	=
	\left(
	\left\{\kUsrs{i}{j}{x}{y}\right\}_{x,y \in \{A,B\}, \: i,j \in [3]}, 
	\left\{\dElmntsGen{i}{j}{x}{y}\right\}_{x,y \in \{A,B\}, \: i,j \in [3]}
	\right),
\end{align}
according to \eqref{eq:tuple_ratMat}.

\textbf{Upper Bound on $\mathsf{Term_1}$.} 
The cardinality of the set $\mathcal{X}(T)$ is given by
\begin{align}
	\mathsf{Term_1}
	&= 
	\underbrace{
		\left\vert 
		\mathcal{X}
		\left(
		\left\{\kUsrs{i}{j}{x}{y} : x,y \in \{A,B\}, i,j \in [3] \right\}, \:
		\left\{\dElmntsGenHat{i}{j}{x}{y} = 0 : x,y \in \{A,B\}, i,j \in [3]\right\}
		\right)
		\right\vert
	}_{\mathsf{Term_{1,1}}}
	\nonumber \\
	&\phantom{=} 
	\times
	\underbrace{
		\left\vert 
		\mathcal{X}
		\left(
		\left\{\kUsrsHat{i}{j}{x}{y} = 0 : x,y \in \{A,B\}, i,j \in [3] \right\}, \:
		\left\{\dElmntsGen{i}{j}{x}{y} : x,y \in \{A,B\}, i,j \in [3] \right\}
		\right)
		\right\vert
	}_{\mathsf{Term_{1,2}}},
	\label{eq:term1_11_12}
\end{align}
which follows from the fact that the number of ways of counting the rating matrices subject to $\{\kUsrs{i}{j}{x}{y} : x,y \in \{A,B\}, i,j \in [3]\}$, and subject to $\{\dElmntsGen{i}{j}{x}{y} : x,y \in \{A,B\}, i,j \in [3] \}$ are independent. Next, we provide upper bounds on $\mathsf{Term_{1,1}}$ and $\mathsf{Term_{1,2}}$.

\textbf{(i) Upper Bound on $\mathsf{Term_{1,1}}$:} 
An upper bound on $\mathsf{Term_{1,1}}$ is given by
\begin{align}
	\mathsf{Term_{1,1}}
	&=
	\left\vert 
	\mathcal{X}
	\left(
	\left\{\kUsrs{i}{j}{x}{y} : x,y \in \{A,B\}, i,j \in [3] \right\}, \:
	\left\{\dElmntsGenHat{i}{j}{x}{y} = 0 : x,y \in \{A,B\}, i,j \in [3]\right\}
	\right)
	\right\vert
	\nonumber\\
	&=
	\prod_{x \in \{A,B\}}
	\prod_{i \in [3]} 
	\binom{n/6}
	{\kUsrs{i}{1}{x}{A}, \kUsrs{i}{2}{x}{A}, \kUsrs{i}{3}{x}{A},
		\kUsrs{i}{1}{x}{B}, \kUsrs{i}{2}{x}{B}, \kUsrs{i}{3}{x}{B}}
	\label{eq:term11_00}\\
	&\leq
	\prod_{x \in \{A,B\}}
	\prod_{i \in [3]} 
	\left(
	\frac{n}{6}
	\right)^{
		\textstyle
		\sum_{(y,j) \neq (\sigma(x), \sigma(i|x))} 
		\kUsrs{i}{j}{x}{y}
	}
	\label{eq:term11_0}\\
	&\leq
	\prod_{x \in \{A,B\}}
	\prod_{i \in [3]} 
	\exp
	\left(
	\left(
	\sum_{(y,j) \neq (\sigma(x), \sigma(i|x))} 
	\kUsrs{i}{j}{x}{y}
	\right)
	\log n
	\right)
	\nonumber\\
	&=
	\exp
	\left(
	\log n
	\left(
	\sum_{x \in \{A,B\}} 
	\sum_{i \in [3]} 
	\sum_{(y,j) \neq (\sigma(x), \sigma(i|x))} 
	\kUsrs{i}{j}{x}{y}
	\right)
	\right),
	\label{eq:term11_1}
\end{align}
where \eqref{eq:term11_00} follows from the definitions in \eqref{eq:tuple_ratMat}; \eqref{eq:term11_0} follows from the definition of a multinomial coefficient, and the fact that $\binom{n}{k} \leq n^k$.

\textbf{(ii) Upper Bound on $\mathsf{Term_{1,2}}$:}
An upper bound on $\mathsf{Term_{1,2}}$ is given by
\begin{align}
	\mathsf{Term_{1,2}}
	&=
	\left\vert 
	\mathcal{X}
	\left(
	\left\{\kUsrsHat{i}{j}{x}{y} = 0 : x,y \in \{A,B\}, i,j \in [3] \right\}, \:
	\left\{\dElmntsGen{i}{j}{x}{y} : x,y \in \{A,B\}, i,j \in [3] \right\}
	\right)
	\right\vert
	\nonumber\\
	&\leq
	\left\vert 
	\mathcal{X}
	\left(
	\left\{\kUsrsHat{i}{j}{x}{y} = 0 : x,y \in \{A,B\}, i,j \in [3] \right\}, \:
	\left\{\dElmntsGen{i}{\sigma(i|x)}{x}{\sigma(x)} : x \in \{A,B\}, i \in [3] \right\}, \:
	\right.
	\right.
	\nonumber\\
	&\phantom{\leq \vert \mathcal{X} \left( \right.}
	\left.
	\left.
	\left\{\dElmntsGenHat{i}{j}{x}{y} =  t: 
	0 \leq t \leq m, \:
	x,y \in \{A,B\} ,\: i,j \in [3],\: (y,j) \neq (\sigma(x),\sigma(i|x)) \right\}
	\right)
	\right\vert
	\nonumber\\
	&\leq
	\prod_{z \in \{A,B\}}
	\left\vert 
	\mathcal{X}
	\left(
	\left\{\kUsrsHat{i}{j}{x}{y} = 0 : x,y \!\in\! \{A,B\}, i,j \!\in\! [3] \right\}\!,
	\left\{\dElmntsGen{i}{\sigma(i|z)}{z}{\sigma(z)} : i \!\in\! [3] \right\}\!,
	\right.
	\right.
	\nonumber\\
	&\phantom{= \prod_{z \in \{A,B\}}
		\left\vert 
		\mathcal{X}
		\left(
		\right. \right.}
	\left.
	\left.
	\left\{\dElmntsGenHat{i}{\sigma(i|x)}{x}{\sigma(x)} = t : 0 \leq t \leq m,\: x \!\in\! \{A,B\} \!\setminus\! \{z\},\: i \!\in\! [3] \right\}\!,
	\right.
	\right.
	\nonumber\\
	&\phantom{= \prod_{z \in \{A,B\}}
		\left\vert 
		\mathcal{X}
		\left(
		\right. \right.}
	\left.
	\left.
	\left\{\dElmntsGenHat{i}{j}{x}{y} \!=\!  t\!:\! 
	0 \!\leq\! t \!\leq m, 
	x,y \in \{A,B\},
	\! i,j \in [3],
	\! (y,j) \neq (\sigma(x),\sigma(i|x)) \right\}
	\right)
	\right\vert.
	\label{eq:term12_upperB}
\end{align}

Recall that $\gtRatingCluster{x} \in \mathbb{F}_{q}^{g \times m}$ denotes a matrix that is obtained by stacking all the rating vectors of cluster $x$ given by $\{\vecV{i}{x}: i\in [3]\}$ for $x \in \{A,B\}$, and whose columns are elements of $(g,r)$ MDS code. 
Similarly, define $\XRatingCluster{x} \in \mathbb{F}_{q}^{g \times m}$ as a matrix that is obtained by stacking all the rating vectors of cluster $x$ given by $\{\vecU{i}{x}: i\in [3]\}$ for $x \in \{A,B\}$, and whose columns are also elements of $(g,r)$ MDS code. Furthermore, define the binary matrix $\gtRatingClusterHat{x} \in \mathbb{F}_{q}^{g \times m}$ as follows:
\begin{align}
	\gtRatingClusterHat{x}(i,t)
	= 
	\indicatorFn{\gtRatingCluster{x}(i,t) \neq \XRatingCluster{x}(i,t)}, \: \text{for}\: x\in\{A,B\}, i\in[3], t\in[m]. 
\end{align}
Note that $\gtRatingClusterHat{x}(i,t) = 1$ when there is an error in estimating the rating of the users in cluster $x$ and group $i$ for item $t$ for $x\in\{A,B\}$, $i\in[3]$ and $t\in[m]$. Let any non-zero column of $\gtRatingClusterHat{x}$ be denoted as an ``error column''.

Then, for a given cluster $x \in \{A,B\}$, we enumerate all possible matrices $\gtRatingClusterHat{x}$ subject to a given number of error columns. To this end, define $\totNumErrCol{x}$ as the total number of error columns of $\gtRatingClusterHat{x}$. Moreover, define $\numErrColConfig$ as the number of possible configurations of an error column. Let $\{\colVecRhat{k} : k \in [\numErrColConfig]\}$ be the set of all possible error columns. In this setting, we have $\numErrColConfig = 3$ since the possible configurations of an error column are given by
\begin{align}
	\colVecRhat{1} = 
	\begin{bmatrix}
		\begin{array}{c}
			1 \\ 0 \\ 1
		\end{array}
	\end{bmatrix}, \:\:
	\colVecRhat{2} = 
	\begin{bmatrix}
		\begin{array}{c}
			1 \\ 1 \\ 0
		\end{array}
	\end{bmatrix}, \:\:
	\colVecRhat{\numErrColConfig} = \colVecRhat{3} = 
	\begin{bmatrix}
		\begin{array}{c}
			0 \\ 1 \\ 1
		\end{array}
	\end{bmatrix}.
\end{align}
Let $\numErrColElemnt{k}{x}$ denote the number of columns of $\gtRatingClusterHat{x}$ that are equal to $\colVecRhat{k}$ for $x \in \{A,B\}$ and $k \in [\numErrColConfig]$. Note that $0 \leq \numErrColElemnt{k}{x} \leq \totNumErrCol{x}$ and $\sum_{k=1}^{\numErrColConfig} \numErrColElemnt{k}{x} = \totNumErrCol{x}$. For cluster $x \in \{A,B\}$, let $\ratingVecSet{x}(\totNumErrCol{x},\{\colVecRhat{k} : k \in [\numErrColConfig]\})$ denote the set of matrices $\XRatingCluster{x} \in \mathbb{F}_{q}^{g \times m}$ characterized by $\totNumErrCol{x}$ and $\{\colVecRhat{k} : k \in [\numErrColConfig]\}$. Consequently, an upper bound on  $\vert\ratingVecSet{x}(\totNumErrCol{x}, \{\colVecRhat{k} : k \in [\numErrColConfig]\})\vert$ is given by
\begin{align}
	\left\vert \ratingVecSet{x}\left(\totNumErrCol{x},\{\colVecRhat{k} : k \in [\numErrColConfig]\}\right)\right\vert 
	&\leq
	\binom{m}{\totNumErrCol{x}}  
	\binom{\totNumErrCol{x}+\numErrColConfig-1}{\numErrColConfig-1} 
	\displaystyle 2^{3\totNumErrCol{x}}
	\label{ineq:Xf_1}\\
	&\leq 
	\displaystyle m^{\totNumErrCol{x}} \: 
	2^{\totNumErrCol{x}+\numErrColConfig-1} \: 
	2^{3\totNumErrCol{x}} 
	\label{ineq:Xf_2}\\
	&\leq
	2^{10} (4m)^{\totNumErrCol{x}},
	\label{ineq:Xf_3} 
\end{align}
where 
\begin{itemize}[leftmargin=*]
	\item \eqref{ineq:Xf_1} follows by first choosing $\totNumErrCol{x}$ columns from $m$ columns to be error columns, then counting the number of integer solutions of $\sum_{k=1}^{\numErrColConfig} \numErrColElemnt{k}{x} = \totNumErrCol{x}$, and lastly counting the number of estimation error combination within the $g$ entries of each of the $\totNumErrCol{x}$ error columns;
	\item \eqref{ineq:Xf_2} follows from bounding the first binomial coefficient by $\binom{a}{b} \leq a^b$, and the second binomial coefficient by $\binom{a}{b} \leq \sum_{i=1}^{a} \binom{a}{i} =  2^a$, for $a \geq b$;
	\item and finally \eqref{ineq:Xf_3} follows from $\numErrColConfig \leq 2^3$ which is due to the fact that each entry of a rating matrix column can take one of two values.
\end{itemize}

Next, for a given cluster $x \in \{A,B\}$, we evaluate the maximum number of error columns among all candidate matrices $\XRatingCluster{x}$. On one hand, row-wise counting of the error entries in $\XRatingCluster{x}$, compared to $\gtRatingCluster{x}$, yields
\begin{align}
	\sum_{i\in[3]} \dElmntsGen{i}{\sigma(i|x)}{x}{\sigma(x)}.
	\label{eq:term12_rowCount}
\end{align}
On the other hand, column-wise counting of the error entries in $\gtRatingClusterHat{x}$ (i.e., number of ones) yields
\begin{align}
	\sum_{k\in[\numErrColConfig]}
	\left\| \colVecRhat{k} \right\|_{1} \numErrColElemnt{k}{x}.
	\label{eq:term12_columnCount}
\end{align}
From \eqref{eq:term12_upperB}, we are interested in the class of candidate rating matrices where the clustering and grouping are done correctly without any errors in user associations to their respective clusters and groups. Therefore, the expression given by \eqref{eq:term12_rowCount} and \eqref{eq:term12_columnCount} are counting the elements of the same set, and hence we obtain
\begin{align}
	\sum_{i\in[3]} \dElmntsGen{i}{\sigma(i|x)}{x}{\sigma(x)} 
	&= 
	\sum_{k\in[\numErrColConfig]}
	\left\| \colVecRhat{k} \right\|_{1} \numErrColElemnt{k}{x}
	\nonumber\\
	&\geq 
	2 
	\sum_{k\in[\numErrColConfig]} \numErrColElemnt{k}{x} 
	\label{eq:mds_code_1} \\ 
	&=
	2 \totNumErrCol{x} 
	\label{eq:mds_code_2} ,
\end{align}
where \eqref{eq:mds_code_1} follows since the MDS code structure is known at the decoder side, and the fact that minimum distance between any two codewords in a  $(3,2)$ linear MDS code is $2$. Therefore, by \eqref{eq:mds_code_2}, we get
\begin{align}
	\max \totNumErrCol{x} 
	=
	\frac{1}{2} \sum_{i\in[3]} \dElmntsGen{i}{\sigma(i|x)}{x}{\sigma(x)}.
	\label{eq:mds_code_3} 
\end{align}

Finally, by \eqref{eq:term12_upperB} and \eqref{eq:mds_code_3}, $\mathsf{Term_{1,2}}$ can be further upper bounded by
\begin{align}
	\mathsf{Term_{1,2}}
	&\leq
	\prod_{z\in\{A,B\}}
	\sum_{\ell=1}^{\max \totNumErrCol{z}}
	\left|\ratingVecSet{z}\left(\totNumErrCol{z}= \ell,\{\colVecRhat{k} : k \in [\numErrColConfig]\} \right)\right|
	\nonumber\\
	&\leq
	\prod_{z\in\{A,B\}}
	\sum_{\ell=1}^{\max \totNumErrCol{z}}
	2^{10} \: (4m)^{\ell}
	\label{ineq:term_12_0}\\
	&\leq
	\prod_{z\in\{A,B\}}
	2^{10} \: m^{\max \totNumErrCol{z}}
	\sum_{\ell=1}^{\max \totNumErrCol{z}}
	4^{\ell}
	\nonumber\\
	&\leq 
	\prod_{z\in\{A,B\}}
	2^{10} \: m^{\max \totNumErrCol{z}} \:
	4^{\max \totNumErrCol{z} + 1}
	\label{ineq:term_12_1} \\ 
	&=
	\prod_{z\in\{A,B\}}
	2^{12} \: (4m)^{\max \totNumErrCol{z}}
	\nonumber\\ 
	&=
	\left(2^{12}\right)^2
	(4m)^{\sum\limits_{x\in\{A,B\}} \max \totNumErrCol{x}}
	\nonumber\\ 
	&=
	c_0 \exp\left(\frac{\log(c_1m)}{2}\sum_{x\in\{A,B\}}\sum_{i\in[3]} \dElmntsGen{i}{\sigma(i|x)}{x}{\sigma(x)} \right),
	\label{ineq:term_12_4}
\end{align}
where \eqref{ineq:term_12_0} follows from \eqref{ineq:Xf_3}; \eqref{ineq:term_12_1} follows from $\sum_{\ell=1}^{\max \totNumErrCol{x}} 4^\ell \leq \sum_{\ell=0}^{\max \totNumErrCol{x}} (4)^\ell \leq (4)^{\max \totNumErrCol{x}+1}$; and \eqref{ineq:term_12_4} follows by setting $c_0 = \left(2^{12} \right)^2 \geq 1$ and $c_1 = 4 \geq 1$.

Substituting \eqref{eq:term11_1} and \eqref{ineq:term_12_4} into \eqref{eq:term1_11_12}, an upper bound on $\mathsf{Term_1}$ is thus given by
\begin{align}
	\mathsf{Term_{1}}
	&\leq
	c_0
	\exp
	\left(
	\log n
	\left(
	\sum_{x \in \{A,B\}} 
	\sum_{i \in [3]} 
	\sum_{(y,j) \neq  (\sigma(x), \sigma(i|x))} 
	\kUsrs{i}{j}{x}{y}
	\right)
	\!+\!
	\frac{\log (c_1m)}{2}
	\left(
	\sum\limits_{x\in\{A,B\}}\sum\limits_{i\in[3]}\dElmntsGen{i}{\sigma(i|x)}{x}{\sigma(x)}
	\right)
	\right). \label{eq:term1_lemma_large}
\end{align}

\textbf{Upper Bound on $\mathsf{Term_2}$.} 
To this end, we derive lower bounds on the cardinalities of different sets in the exponent of $\mathsf{Term_2}$. Recall from \eqref{eq:TsmallErr} that
\begin{align}
	\TsmallErr
	&= 
	\left\{ T \in \mathcal{T}^{(\delta)} 
	\!:\! 
	\forall (x,i) \!\in\! \{A,B\}\!\times\! [3] \textit{ s.t. }
	\left\vert \sigma(x) \right\vert = 1,\! 
	\left\vert \sigma(i|x) \right\vert = 1,\! 
	\dElmntsGen{i}{\sigma(i|x)}{x}{\sigma(x)} \leq \tau m \min\{\intraTau, \interTau\}
	\right\},
	\nonumber\\
	&=
	\left\{ T \in \mathcal{T}^{(\delta)} \!:\! 
	\forall (x,i) \!\in\! \{A,B\}\!\times\! [3] \textit{ s.t. }
	\kUsrs{i}{\sigma(i|x)}{x}{\sigma(x)} \geq (1-\relabelConst) \frac{n}{6},\:\:
	\dElmntsGen{i}{\sigma(i|x)}{x}{\sigma(x)} \leq \tau m \min\{\intraTau, \interTau\}
	\right\},
	\label{eq:Tlarge_complement_0}
\end{align}
where \eqref{eq:Tlarge_complement_0} follows from \eqref{eq:sigma(x,i)}.

\textbf{(i) Lower Bound on} $\lvert\diffEntriesSet\rvert$\textbf{:} 
For $T\in \TsmallErr$, a lower bound on  $\lvert\diffEntriesSet\rvert$ is given by
\begin{align}
	\lvert\diffEntriesSet\rvert
	&=
	\sum\limits_{x \in \{A,B\}}
	\sum\limits_{i \in [3]}
	\sum\limits_{y \in \{A,B\}}
	\sum\limits_{j \in [3]}
	\kUsrs{i}{j}{x}{y} \dElmntsGen{i}{j}{x}{y} 
	\label{ineq:Lambda_1}\\
	&=
	\left[
	\sum\limits_{x \in \{A,B\}}
	\sum\limits_{i \in [3]}
	\kUsrs{i}{\sigma(i|x)}{x}{\sigma(x)} 
	\dElmntsGen{i}{\sigma(i|x)}{x}{\sigma(x)}
	\right]
	+
	\left[
	\sum\limits_{x \in \{A,B\}}
	\sum\limits_{i \in [3]}
	\sum\limits_{j \in [3]\setminus \sigma(i|x)}
	\kUsrs{i}{j}{x}{\sigma(x)} 
	\dElmntsGen{i}{j}{x}{\sigma(x)} 
	\right]
	\nonumber \\
	&\phantom{=}
	+
	\left[
	\sum\limits_{x \in \{A,B\}}
	\sum\limits_{y \in \{A,B\} \setminus \sigma(x)}
	\sum\limits_{i \in [3]}
	\sum\limits_{j \in [3]}
	\kUsrs{i}{j}{x}{y}
	\dElmntsGen{i}{j}{x}{y}
	\right]
	\nonumber\\
	&\geq
	\left(
	\sum\limits_{x \in \{A,B\}} \sum\limits_{i \in [3]} 
	\dElmntsGen{i}{\sigma(i|x)}{x}{\sigma(x)}
	\left(
	(1-\relabelConst) \frac{n}{6}
	\right)
	\right)
	\nonumber \\
	&\phantom{=}
	+
	\left(
	\sum_{x\in\{A,B\}}\sum_{i\in[3]}\sum_{j\in[3] \setminus \sigma(i|x)} \kUsrs{i}{j}{x}{\sigma(x)}
	\left(\hamDist{\vecV{\sigma(i|x)}{\sigma(x)}}{\vecV{j}{\sigma(x)}} - \dElmntsGen{i}{\sigma(i|x)}{x}{\sigma(x)} \right)
	\right)
	\nonumber\\
	&\phantom{\leq}
	+ 
	\left(
	\sum_{x\in\{A,B\}}\sum_{y\in\{A,B\}\setminus \sigma(x)} \sum_{i\in[3]} \sum_{j\in[3]} \kUsrs{i}{j}{x}{y}
	\left(\hamDist{\vecV{\sigma(i|x)}{\sigma(x)}}{\vecV{j}{y}} - \dElmntsGen{i}{\sigma(i|x)}{x}{\sigma(x)} \right)
	\right)
	\label{ineq:Lambda_2} 
	\\
	&\geq
	(1-\relabelConst) \frac{n}{6}
	\left(
	\sum\limits_{x \in \{A,B\}} \sum\limits_{i \in [3]} 
	\dElmntsGen{i}{\sigma(i|x)}{x}{\sigma(x)}
	\right)
	+
	\left(\intraTau m - \intraTau\tau m \right)
	\left(\sum_{x\in\{A,B\}}\sum_{i\in[3]}\sum_{j \in [3] \setminus \sigma(i|x)} \kUsrs{i}{j}{x}{\sigma(x)}\right)
	\nonumber\\
	&\phantom{\leq}
	+ 
	\left( \interTau m -  \interTau\tau m \right)
	\left(\sum_{x\in\{A,B\}}\sum_{y\in\{A,B\}\setminus \sigma(x)} \sum_{i\in[3]} \sum_{j\in[3]} \kUsrs{i}{j}{x}{y}\right)
	\label{ineq:Lambda_2_1} 
	\\
	&=
	(1-\tau)\!
	\left(
	\frac{n}{6}
	\left(
	\sum\limits_{x \in \{A,B\}} \sum\limits_{i \in [3]} 
	\dElmntsGen{i}{\sigma(i|x)}{x}{\sigma(x)}
	\right)
	+
	\intraTau m\left(\sum_{x\in\{A,B\}}\sum_{i\in[3]}\sum_{j \in [3] \setminus \sigma(i|x)} \kUsrs{i}{j}{x}{\sigma(x)}\right)
	\right.
	\nonumber \\
	&\phantom{=(1-\tau)(} \left.
	+ 
	\interTau m\left(\sum_{x\in\{A,B\}}\sum_{y\in\{A,B\}\setminus \sigma(x)} \sum_{i\in[3]} \sum_{j\in[3]} \kUsrs{i}{j}{x}{y}\right)\!\!\! 
	\right)\!\!,
	\label{ineq:Lambda_3} 
\end{align}
where \eqref{ineq:Lambda_1} follows from the definitions in \eqref{eq:N1_defn}; \eqref{ineq:Lambda_2} follows from \eqref{eq:Tlarge_complement_0} and the triangle inequality; and \eqref{ineq:Lambda_2_1} follows from \eqref{eq:Tlarge_complement_0} and the fact that the minimum hamming distance between any two different rating vectors in $\cV$ is $\min\{\intraTau,\interTau\}m$.

\textbf{(ii) Lower Bound on $\lvert \misClassfSet{\betaEdge}{\alphaEdge} \rvert$ and $\lvert \misClassfSet{\alphaEdge}{\betaEdge} \rvert$:} 
For $T\in \TsmallErr$, a lower bound on $\lvert \misClassfSet{\alphaEdge}{\betaEdge} \rvert$ is given by
\begin{align}
	\lvert\misClassfSet{\betaEdge}{\alphaEdge}\rvert
	&=
	\sum\limits_{x \in \{A,B\}}
	\sum\limits_{y \in \{A,B\}}
	\sum\limits_{i \in [3]}
	\sum\limits_{j \in [3] \setminus i}
	\sum\limits_{k \in [3]}
	\kUsrs{i}{k}{x}{y} \kUsrs{j}{k}{x}{y}
	\label{eq:Nbetaalpha}\\
	&\geq
	\sum\limits_{x \in \{A,B\}}
	\sum\limits_{i \in [3]}
	\sum\limits_{j \in [3] \setminus i}
	\sum\limits_{k \in [3]}
	\kUsrs{i}{k}{x}{\sigma(x)} \kUsrs{j}{k}{x}{\sigma(x)}
	\nonumber\\
	&\geq
	\sum\limits_{x \in \{A,B\}}
	\sum\limits_{i \in [3]}
	\sum\limits_{j \in [3] \setminus i}
	\sum\limits_{k \in [3] \setminus \sigma(i|x)}
	\kUsrs{i}{k}{x}{\sigma(x)} \kUsrs{j}{k}{x}{\sigma(x)}
	\nonumber\\
	&=
	\sum\limits_{x \in \{A,B\}}
	\sum\limits_{i \in [3]}
	\sum\limits_{k \in [3] \setminus \sigma(i|x)}
	\kUsrs{i}{k}{x}{\sigma(x)}
	\left(
	\sum\limits_{j \in [3] \setminus i}
	\kUsrs{j}{k}{x}{\sigma(x)}
	\right)
	\nonumber\\
	&\geq
	\sum\limits_{x \in \{A,B\}}
	\sum\limits_{i \in [3]}
	\sum\limits_{k \in [3] \setminus \sigma(i|x)}
	\kUsrs{i}{k}{x}{\sigma(x)}
	\left(
	\left(1 -\tau\right) \frac{n}{6}
	\right)
	\label{ineq:Nbetaalpha_1}\\
	&=
	\left(1 -\tau\right) \frac{n}{6}
	\sum\limits_{x \in \{A,B\}}
	\sum\limits_{i \in [3]}
	\sum\limits_{j \in [3] \setminus \sigma(i|x)}
	\kUsrs{i}{j}{x}{\sigma(x)},
	\label{ineq:Nbetaalpha_2}
\end{align}
where \eqref{eq:Nbetaalpha} follows from the definitions in \eqref{eq:N2rev_defn}; and \eqref{ineq:Nbetaalpha_1} follows from \eqref{eq:Tlarge_complement_0}. Similarly, for $T\in \TsmallErr$, a lower bound on $\lvert \misClassfSet{\alphaEdge}{\betaEdge} \rvert$ is given by
\begin{align}
	\lvert\misClassfSet{\alphaEdge}{\betaEdge}\rvert
	&=
	\sum_{x\in\{A,B\}}
	\sum_{y\in\{A,B\}}
	\sum_{i\in[3]} 
	\sum_{j\in[3]} 
	\sum_{k\in[3] \setminus j}
	\kUsrs{i}{j}{x}{y} \kUsrs{i}{k}{x}{y}
	\label{eq:Nalphabeta}\\
	&\geq
	\sum_{x\in\{A,B\}}
	\sum_{i\in[3]} 
	\sum_{j\in[3]} 
	\sum_{k\in[3] \setminus j}
	\kUsrs{i}{j}{x}{\sigma(x)} \kUsrs{i}{k}{x}{\sigma(x)}
	\nonumber\\
	&\geq
	\sum_{x\in\{A,B\}}
	\sum_{i\in[3]} 
	\sum_{k\in[3] \setminus \sigma(i|x)}
	\kUsrs{i}{\sigma(i|x)}{x}{\sigma(x)}
	\kUsrs{i}{k}{x}{\sigma(x)}
	\nonumber\\
	&\geq
	\sum_{x\in\{A,B\}}
	\sum_{i\in[3]} 
	\sum_{k\in[3] \setminus \sigma(i|x)}
	\left(1 -\tau\right) \frac{n}{6}
	\kUsrs{i}{k}{x}{\sigma(x)}
	\nonumber\\
	&=
	\left(1 -\tau\right) \frac{n}{6}
	\sum\limits_{x\in\{A,B\}} 
	\sum_{i\in[3]} 
	\sum_{j\in[3] \setminus \sigma(i|x)} 
	\kUsrs{i}{j}{x}{\sigma(x)},
	\label{ineq:Nalphabeta_3}
\end{align}
where \eqref{eq:Nalphabeta} follows from the definitions in \eqref{eq:N2_defn}. Therefore, by \eqref{ineq:Nbetaalpha_2} and \eqref{ineq:Nalphabeta_3}, we obtain
\begin{align}
	\frac{
		\lvert\misClassfSet{\betaEdge}{\alphaEdge}\rvert
		+ \lvert\misClassfSet{\alphaEdge}{\betaEdge}\rvert}
	{2}
	\geq 
	\left(1 -\tau\right) \frac{n}{6}
	\sum\limits_{x\in\{A,B\}} 
	\sum_{i\in[3]} 
	\sum_{j\in[3] \setminus \sigma(i|x)} 
	\kUsrs{i}{j}{x}{\sigma(x)}.
	\label{ineq:Ng}
\end{align}

\textbf{(iii) Lower Bound on $\lvert \misClassfSet{\gammaEdge}{\alphaEdge} \rvert$ and $\lvert \misClassfSet{\alphaEdge}{\gammaEdge} \rvert$:} 
For $T\in \TsmallErr$, a lower bound on $\lvert \misClassfSet{\gammaEdge}{\alphaEdge} \rvert$ is given by
\begin{align}
	\lvert\misClassfSet{\gammaEdge}{\alphaEdge}\rvert
	&=
	\sum_{x \in \{A,B\}}
	\sum_{y \in \{A,B\}}
	\sum_{z \in \{A,B\} \setminus x}
	\sum_{i \in [3]}
	\sum_{j \in [3]}
	\sum_{k \in [3]}
	\kUsrs{i}{j}{x}{y}  \kUsrs{k}{j}{z}{y} 
	\label{eq:Ngammaalpha}\\
	&\geq
	\sum_{x \in \{A,B\}}
	\sum_{y \in \{A,B\} \setminus \sigma(x)}
	\sum_{z \in \{A,B\} \setminus x}
	\sum_{i \in [3]}
	\sum_{j \in [3]}
	\sum_{k \in [3]}
	\kUsrs{i}{j}{x}{y}  \kUsrs{k}{j}{z}{y} 
	\nonumber\\
	&=
	\sum_{x \in \{A,B\}}
	\sum_{y \in \{A,B\} \setminus \sigma(x)}
	\sum_{i \in [3]}
	\sum_{j \in [3]}
	\kUsrs{i}{j}{x}{y}
	\left(
	\sum_{z \in \{A,B\} \setminus x}
	\sum_{k \in [3]}
	\kUsrs{k}{j}{z}{y} 
	\right)
	\nonumber\\
	&\geq
	\sum_{x \in \{A,B\}}
	\sum_{y \in \{A,B\} \setminus \sigma(x)}
	\sum_{i \in [3]}
	\sum_{j \in [3]}
	\kUsrs{i}{j}{x}{y}
	\left(
	\left(1 -\tau\right) \frac{n}{6} 
	\right)
	\label{ineq:Ngammaalpha_2}\\
	&=
	\left(1 -\tau\right) \frac{n}{6} 
	\sum_{x \in \{A,B\}}
	\sum_{y \in \{A,B\} \setminus \sigma(x)}
	\sum_{i \in [3]}
	\sum_{j \in [3]}
	\kUsrs{i}{j}{x}{y}, 
	\label{ineq:Ngammaalpha_4}
\end{align}
where \eqref{eq:Ngammaalpha} follows from the definitions in \eqref{eq:N3rev_defn}; and \eqref{ineq:Ngammaalpha_2} follows from \eqref{eq:Tlarge_complement_0}. Similarly, for $T\in \TsmallErr$, a lower bound on $\lvert \misClassfSet{\alphaEdge}{\gammaEdge} \rvert$ is given by
\begin{align}
	\lvert\misClassfSet{\alphaEdge}{\gammaEdge}\rvert
	&=
	\sum_{x \in \{A,B\}}
	\sum_{y \in \{A,B\}}
	\sum_{z \in \{A,B\} \setminus y}
	\sum_{i \in [3]}
	\sum_{j \in [3]}
	\sum_{k \in [3]}
	\kUsrs{i}{j}{x}{y}  \kUsrs{i}{k}{x}{z} 
	\label{eq:Nalphagamma}\\
	&\geq
	\sum_{x \in \{A,B\}}
	\sum_{z \in \{A,B\} \setminus \sigma(x)}
	\sum_{i \in [3]}
	\sum_{k \in [3]}
	\sum_{j \in [3]}
	\kUsrs{i}{j}{x}{\sigma(x)}  \kUsrs{i}{k}{x}{z} 
	\nonumber\\
	&\geq
	\sum_{x \in \{A,B\}}
	\sum_{z \in \{A,B\} \setminus \sigma(x)}
	\sum_{i \in [3]}
	\sum_{k \in [3]}
	\kUsrs{i}{\sigma(i|x)}{x}{\sigma(x)}  \kUsrs{i}{k}{x}{z}  
	\nonumber\\
	&=
	\sum_{x \in \{A,B\}}
	\sum_{z \in \{A,B\} \setminus \sigma(x)}
	\sum_{i \in [3]}
	\sum_{k \in [3]}
	\left(1 -\tau\right) \frac{n}{6} 
	\kUsrs{i}{k}{x}{z}  
	\nonumber\\
	&=
	\left(1 -\tau\right) \frac{n}{6} 
	\sum_{x \in \{A,B\}}
	\sum_{y \in \{A,B\} \setminus \sigma(x)}
	\sum_{i \in [3]}
	\sum_{j \in [3]}
	\kUsrs{i}{j}{x}{y}, 
	\label{ineq:Nalphagamma_2}
\end{align}
where \eqref{eq:Nalphagamma} follows from the definitions in \eqref{eq:N3_defn}. Hence, by \eqref{ineq:Ngammaalpha_4} and \eqref{ineq:Nalphagamma_2}, we obtain
\begin{align}
	\frac{
		\lvert\misClassfSet{\gammaEdge}{\alphaEdge}\rvert
		+
		\lvert\misClassfSet{\alphaEdge}{\gammaEdge}\rvert}{2}
	\geq 
	\left(1 -\tau\right) \frac{n}{6}
	\sum_{x \in \{A,B\}}
	\sum_{y \in \{A,B\} \setminus x}
	\sum_{i \in [3]}
	\sum_{j \in [3]}
	\kUsrs{i}{j}{x}{y}.
	\label{ineq:Nc1}
\end{align}

\textbf{(iv) Lower Bound on $\lvert \misClassfSet{\gammaEdge}{\betaEdge} \rvert$ and $\lvert \misClassfSet{\betaEdge}{\gammaEdge} \rvert$:} 
For $T\in \TsmallErr$, a lower bound on $\lvert \misClassfSet{\gammaEdge}{\betaEdge} \rvert$ is given by
\begin{align}
	\lvert\misClassfSet{\gammaEdge}{\betaEdge}\rvert
	&=
	\sum_{x \in \{A,B\}}
	\sum_{y \in \{A,B\}}
	\sum_{z \in \{A,B\} \setminus x}
	\sum_{i \in [3]}
	\sum_{k \in [3]}
	\sum_{j \in [3]}
	\sum_{\ell \in [3] \setminus j}
	\kUsrs{i}{j}{x}{y} 
	\kUsrs{k}{\ell}{z}{y} 
	\label{eq:Ngammabeta}\\
	&\geq
	\sum_{x \in \{A,B\}}
	\sum_{y \in \{A,B\} \setminus \sigma(x)}
	\sum_{z \in \{A,B\} \setminus x}
	\sum_{i \in [3]}
	\sum_{k \in [3]}
	\sum_{j \in [3]}
	\sum_{\ell \in [3] \setminus j}
	\kUsrs{i}{j}{x}{y} 
	\kUsrs{k}{\ell}{z}{y} 
	\nonumber\\
	&=
	\sum_{x \in \{A,B\}}
	\sum_{y \in \{A,B\} \setminus \sigma(x)}
	\sum_{i \in [3]}
	\sum_{j \in [3]}
	\kUsrs{i}{j}{x}{y} 
	\left(
	\sum_{\ell \in [3] \setminus j}
	\sum_{z \in \{A,B\} \setminus x}
	\sum_{k \in [3]}
	\kUsrs{k}{\ell}{z}{y}
	\right)
	\nonumber\\
	&\geq
	\sum_{x \in \{A,B\}}
	\sum_{y \in \{A,B\} \setminus \sigma(x)}
	\sum_{i \in [3]}
	\sum_{j \in [3]}
	\kUsrs{i}{j}{x}{y} 
	\left(
	\sum_{\ell \in [3] \setminus j}
	\left(1 -\tau\right) \frac{n}{6}
	\right)
	\label{ineq:Ngammabeta_2}\\
	&=
	\left(1 -\tau\right) \frac{n}{3}
	\sum_{x \in \{A,B\}}
	\sum_{y \in \{A,B\} \setminus x}
	\sum_{i \in [3]}
	\sum_{j \in [3]}
	\kUsrs{i}{j}{x}{y}, 
	\label{ineq:Ngammabeta_3}
\end{align}
where \eqref{eq:Ngammabeta} follows from the definitions in \eqref{eq:N4rev_defn}; and \eqref{ineq:Ngammabeta_2} follows from \eqref{eq:Tlarge_complement_0}. Similarly, for $T\in \TsmallErr$, a lower bound on $\lvert \misClassfSet{\betaEdge}{\gammaEdge} \rvert$ is given by
\begin{align}
	\lvert\misClassfSet{\betaEdge}{\gammaEdge}\rvert
	&=
	\sum_{x \in \{A,B\}}
	\sum_{y \in \{A,B\}}
	\sum_{z \in \{A,B\} \setminus y}
	\sum_{i \in [3]}
	\sum_{k \in [3] \setminus i}
	\sum_{j \in [3]}
	\sum_{\ell \in [3]}
	\kUsrs{i}{j}{x}{y} 
	\kUsrs{k}{\ell}{x}{z} 
	\label{eq:Nbetagamma}\\
	&\geq
	\sum_{x \in \{A,B\}}
	\sum_{z \in \{A,B\} \setminus \sigma(x)}
	\sum_{i \in [3]}
	\sum_{k \in [3] \setminus i}
	\sum_{j \in [3]}
	\sum_{\ell \in [3]}
	\kUsrs{i}{j}{x}{\sigma(x)} 
	\kUsrs{k}{\ell}{x}{z} 
	\nonumber\\
	&=
	\sum_{x \in \{A,B\}}
	\sum_{z \in \{A,B\} \setminus \sigma(x)}
	\sum_{k \in [3]}
	\sum_{i \in [3] \setminus k}
	\kUsrs{i}{\sigma(x|i)}{x}{\sigma(x)} 
	\sum_{\ell \in [3]}
	\kUsrs{k}{\ell}{x}{z} 
	\nonumber\\
	&\geq
	\sum_{x \in \{A,B\}}
	\sum_{z \in \{A,B\} \setminus \sigma(x)}
	\sum_{k \in [3]}
	\sum_{i \in [3] \setminus k}
	\left(1 -\tau\right) \frac{n}{6}
	\sum_{\ell \in [3]}
	\kUsrs{k}{\ell}{x}{z} 
	\nonumber\\    
	&=
	\sum_{x \in \{A,B\}}
	\sum_{z \in \{A,B\} \setminus \sigma(x)}
	\sum_{k \in [3]}
	\left(1 -\tau\right) \frac{n}{3}
	\sum_{\ell \in [3]}
	\kUsrs{k}{\ell}{x}{z} 
	\nonumber\\
	&=
	\left(1 -\tau\right) \frac{n}{3}
	\sum_{x \in \{A,B\}}
	\sum_{y \in \{A,B\} \setminus \sigma(x)}
	\sum_{i \in [3]}
	\sum_{j \in [3]}
	\kUsrs{i}{j}{x}{y}, 
	\label{ineq:Nbetagamma_3}
\end{align}
where \eqref{eq:Nbetagamma} follows from the definitions in \eqref{eq:N4_defn}. Therefore, by \eqref{ineq:Ngammabeta_3} and \eqref{ineq:Nbetagamma_3}, we obtain
\begin{align}
	\frac{
		\lvert\misClassfSet{\gammaEdge}{\betaEdge}\rvert
		+
		\lvert\misClassfSet{\betaEdge}{\gammaEdge}\rvert}{2}
	\geq 
	\left(1 -\tau\right) \frac{n}{3}
	\sum_{x \in \{A,B\}}
	\sum_{y \in \{A,B\} \setminus x}
	\sum_{i \in [3]}
	\sum_{j \in [3]}
	\kUsrs{i}{j}{x}{y}. 
	\label{ineq:Nc2}
\end{align}

By \eqref{ineq:Lambda_3}, \eqref{ineq:Ng}, \eqref{ineq:Nc1} and \eqref{ineq:Nc2}, an upper bound on $\mathsf{Term_2}$ is given by
\begin{align}
	\mathsf{Term_2}
	&=
	\exp\left(-\left(1+o(1)\right)  \left(\lvert\diffEntriesSet\rvert \Ir 
	+ \Pleftrightarrow{\alphaEdge}{\betaEdge} \: \Ig 
	+ \Pleftrightarrow{\alphaEdge}{\gammaEdge} \: \IcOne 
	+ \Pleftrightarrow{\betaEdge}{\gammaEdge} \: \IcTwo 
	\right)\right)
	\nonumber\\
	&\leq
	\exp
	\left(
	- (1-\tau)
	\left(
	\frac{n \Ir}{6}
	\left(
	\sum\limits_{x \in \{A,B\}} \sum\limits_{i \in [3]} 
	\dElmntsGen{i}{\sigma(i|x)}{x}{\sigma(x)}
	\right)
	\right.\right. \nonumber \\
	&\phantom{\leq\exp
		\left(- (1-\tau)\left(\right.\right.}\left.\left.
	+
	\left(
	\intraTau m \Ir + \frac{n \Ig}{6}
	\right)
	\left(\sum_{x\in\{A,B\}}\sum_{i\in[3]}\sum_{j \in [3] \setminus \sigma(i|x)} \kUsrs{i}{j}{x}{\sigma(x)}\right)
	\right.
	\right.
	\nonumber\\
	&\phantom{\leq\exp
		\left(- (1-\tau)\left(\right.\right.} \:\:\:
	\left.
	\left.
	+ 
	\:
	\left(
	\interTau m \Ir + \frac{n \IcOne}{6} + \frac{n \IcTwo}{3}
	\right)
	\left(\sum_{x \in \{A,B\}}
	\sum_{y \in \{A,B\} \setminus \sigma(x)}
	\sum_{i \in [3]}
	\sum_{j \in [3]}
	\kUsrs{i}{j}{x}{y}\right) 
	\right)
	\right)
	\nonumber\\
	&\leq
	\exp
	\left(
	- (1-\tau) (1+\epsilon)
	\left(
	\frac{\log m}{2}
	\left(
	\sum\limits_{x \in \{A,B\}} \sum\limits_{i \in [3]} 
	\dElmntsGen{i}{\sigma(i|x)}{x}{\sigma(x)}
	\right)
	\right.
	\right.
	\nonumber\\
	&\phantom{\leq\exp
		\left(- (1-\tau)(1+\epsilon)\left(\right.\right.} \:\:\:
	\left.
	\left.
	+
	\log n
	\left(\sum_{x\in\{A,B\}}\sum_{i\in[3]}\sum_{j \in [3] \setminus \sigma(i|x)} \kUsrs{i}{j}{x}{\sigma(x)}\right)
	\right.
	\right.
	\nonumber\\
	&\phantom{\leq\exp
		\left(- (1-\tau)(1+\epsilon)\left(\right.\right.} \:\:\:
	\left.
	\left.
	+ 
	\:
	\log n
	\left(\sum_{x \in \{A,B\}}
	\sum_{y \in \{A,B\} \setminus \sigma(x)}
	\sum_{i \in [3]}
	\sum_{j \in [3]}
	\kUsrs{i}{j}{x}{y}\right) 
	\right)
	\right)
	\label{ineq:upp_exp_2}\\
	&\leq
	\exp
	\left(
	- \left( 1+\frac{\epsilon}{2} \right)
	\left(
	\frac{\log (c_1 m)}{2}
	\left(
	\sum\limits_{x \in \{A,B\}} \sum\limits_{i \in [3]} 
	\dElmntsGen{i}{\sigma(i|x)}{x}{\sigma(x)}
	\right)
	\right.
	\right.
	\nonumber\\
	&\phantom{\leq\exp
		\left(- (1-\tau)(1+\epsilon)\left(\right.\right.} \:\:\:
	\left.
	\left.
	+
	\log n 
	\left(\sum_{x\in\{A,B\}}\sum_{i\in[3]}\sum_{j \in [3] \setminus \sigma(i|x)} \kUsrs{i}{j}{x}{\sigma(x)}\right)
	\right.
	\right.
	\nonumber\\
	&\phantom{\leq\exp
		\left(- (1-\tau)(1+\epsilon)\left(\right.\right.} \:\:\:
	\left.
	\left.
	+ 
	\:
	\log n
	\left(\sum_{x \in \{A,B\}}
	\sum_{y \in \{A,B\} \setminus \sigma(x)}
	\sum_{i \in [3]}
	\sum_{j \in [3]}
	\kUsrs{i}{j}{x}{y}\right) 
	\right)
	\right)
	\label{ineq:upp_exp_3}\\
	&=
	\exp
	\left(
	- \left( 1+\frac{\epsilon}{2} \right)
	\left(
	\frac{\log (c_1 m)}{2}
	\left(
	\sum\limits_{x \in \{A,B\}} \sum\limits_{i \in [3]} 
	\dElmntsGen{i}{\sigma(i|x)}{x}{\sigma(x)}
	\right)
	\right.
	\right.
	\nonumber\\
	&\phantom{\leq\exp
		\left(- (1+\frac{\epsilon}{2})\left(\right.\right.} \:\:\:
	\left.
	\left.
	+ 
	\log n
	\left(
	\sum_{x \in \{A,B\}} 
	\sum_{i \in [3]} 
	\sum_{(y,j) \neq (\sigma(x),\sigma(i|x))} 
	\kUsrs{i}{j}{x}{y}\right) 
	\right)
	\right),
	\label{ineq:upp_exp_4}
\end{align}
where \eqref{ineq:upp_exp_2} follows from the sufficient conditions in \eqref{eq:suffCond_1}, \eqref{eq:suffCond_2} and \eqref{eq:suffCond_3}; and \eqref{ineq:upp_exp_3} holds since $\tau \leq  (\epsilon\log m - (2+\epsilon)\log(2q)) / (2 (1 + \epsilon)\log m))$ implies that $(1-\tau)(1+\epsilon)\log m \geq (1+\epsilon/2)\log(c_1m)$ and $(1-\tau)(1+\epsilon) \geq (1+(\epsilon/2))$.

Finally, by \eqref{eq:term1_lemma_large} and \eqref{ineq:upp_exp_4}, the function in the RHS of \eqref{eq:largeErr_LHS_app} is upper bounded by
\begin{align}
	&\sum\limits_{T \in \TsmallErr} 
	\left\vert \mathcal{X}(T) \right\vert
	\:\:
	\exp\left(-\left(1+o(1)\right) \left(\lvert\diffEntriesSet\rvert \Ir 
	+ \Pleftrightarrow{\alphaEdge}{\betaEdge} \: \Ig 
	+ \Pleftrightarrow{\alphaEdge}{\gammaEdge} \: \IcOne 
	+ \Pleftrightarrow{\betaEdge}{\gammaEdge} \: \IcTwo 
	\right)\right).
	\nonumber \\
	&\leq
	\sum\limits_{T \in \TsmallErr} 
	c_0
	\exp
	\left(
	- \frac{\epsilon}{2}
	\left(
	\frac{\log (c_1m)}{2}
	\left(
	\sum\limits_{x \in \{A,B\}} \sum\limits_{i \in [3]} 
	\dElmntsGen{i}{\sigma(i|x)}{x}{\sigma(x)}
	\right)
	\right.\right. \nonumber \\
	&\phantom{\leq
		\sum\limits_{T \in \TsmallErr}c_o\exp(}
	\left.\left.
	+ 
	\log n
	\left(
	\sum_{x \in \{A,B\}} 
	\sum_{i \in [3]} 
	\sum_{(y,j) \neq (\sigma(x),\sigma(i|x))} 
	\kUsrs{i}{j}{x}{y}\right) 
	\right)
	\right).
	\nonumber\\
	&\quad=
	\sum_{\ell_1 = 0}^{6\tau  \min\{\intraTau, \interTau\} m}\:
	\sum_{\ell_2 = 0}^{\relabelConst n}
	\left| \left\{
	\sum\limits_{x \in \{A,B\}} \sum\limits_{i \in [3]} 
	\dElmntsGen{i}{\sigma(i|x)}{x}{\sigma(x)} = \ell_1, \:
	\sum_{x \in \{A,B\}} 
	\sum_{i \in [3]} 
	\sum_{(y,j) \neq (\sigma(x),\sigma(i|x))} 
	\kUsrs{i}{j}{x}{y} = \ell_2
	\right\} \right|
	\nonumber\\
	&\quad\phantom{=
		\sum_{\ell_1 = 1}^{6\tau \min\{\intraTau, \interTau\} m}
		\sum_{\ell_2 = 1}^{\relabelConst n}}
	\times
	\exp
	\left(
	- \frac{\epsilon \log (c_1m)}{4}
	\ell_1
	- \frac{\epsilon \log n}{2}
	\ell_2
	\right)
	\label{eq:largeErr_errComvg_0}\\
	&\quad=
	\sum_{\ell_1 = 1}^{6\tau \min\{\intraTau, \interTau\} m}
	\left| \left\{
	\sum\limits_{x \in \{A,B\}} \sum\limits_{i \in [3]} 
	\dElmntsGen{i}{\sigma(i|x)}{x}{\sigma(x)} = \ell_1
	\right\} \right|
	\left| \left\{
	\sum_{x \in \{A,B\}} 
	\sum_{i \in [3]} 
	\sum_{(y,j) \neq (\sigma(x),\sigma(i|x))} 
	\kUsrs{i}{j}{x}{y} \!=\!0
	\right\} \right|
	\nonumber \\
	&\phantom{=\sum_{\ell_1 = 1}^{6\tau \min\{\intraTau, \interTau\} m}}
	\times
	\exp
	\left(
	- \frac{\epsilon \log (c_1m)}{4}
	\ell_1
	\right)
	\nonumber\\
	&\quad\phantom{=}\:
	+ \sum_{\ell_2 = 1}^{\relabelConst n}
	\left| \left\{
	\sum\limits_{x \in \{A,B\}} \sum\limits_{i \in [3]} 
	\dElmntsGen{i}{\sigma(i|x)}{x}{\sigma(x)} = 0
	\right\} \right|
	\left| \left\{
	\sum_{x \in \{A,B\}} 
	\sum_{i \in [3]} 
	\sum_{(y,j) \neq (\sigma(x),\sigma(i|x))} 
	\kUsrs{i}{j}{x}{y} = \ell_2
	\right\} \right|
	\nonumber \\
	&\phantom{=\quad\sum_{\ell_2 = 1}^{\relabelConst n}}
	\times
	\exp
	\left(
	- \frac{\epsilon \log n}{2}
	\ell_2
	\right)
	\nonumber\\
	&\quad\phantom{=}\:
	+ \sum_{\ell_1 = 1}^{6\tau \min\{\intraTau, \interTau\} m}\:
	\sum_{\ell_2 = 1}^{\relabelConst n}
	\left| \left\{
	\sum\limits_{x \in \{A,B\}} \sum\limits_{i \in [3]} 
	\dElmntsGen{i}{\sigma(i|x)}{x}{\sigma(x)} = \ell_1
	\right\} \right|
	\nonumber \\
	&\phantom{\quad\phantom{=}\:
		+ \sum_{\ell_1 = 1}^{6\tau \min\{\intraTau, \interTau\} m} \sum_{\ell_2 = 1}^{\relabelConst n}}
	\times
	\left| \left\{
	\sum_{x \in \{A,B\}} 
	\sum_{i \in [3]} 
	\sum_{(y,j) \neq (\sigma(x),\sigma(i|x))} 
	\kUsrs{i}{j}{x}{y} = \ell_2
	\right\} \right|
	\nonumber\\
	&\quad\phantom{\leq
		\sum_{\ell_1 = 1}^{6\tau \min\{\intraTau, \interTau\} m}
		\sum_{\ell_2 = 1}^{\relabelConst n}}
	\times
	\exp
	\left(
	- \frac{\epsilon \log (c_1m)}{4}
	\ell_1
	- \frac{\epsilon \log n}{2}
	\ell_2
	\right)
	\label{eq:largeErr_errComvg_1}\\
	&\quad=
	\sum_{\ell_1 = 1}^{6\tau \min\{\intraTau, \interTau\} m}
	\binom{\ell_1 + 6}{6}
	\exp
	\left(
	- \frac{\epsilon \log (c_1m)}{4}
	\ell_1
	\right)
	+
	\sum_{\ell_2 = 1}^{\relabelConst n}
	\binom{\ell_2 + 6}{6}
	\exp
	\left(
	- \frac{\epsilon \log n}{2}
	\ell_2
	\right)
	\nonumber\\
	&\quad\phantom{=}\:
	+ \sum_{\ell_1 = 1}^{6\tau \min\{\intraTau, \interTau\} m}\:
	\sum_{\ell_2 = 1}^{\relabelConst n}
	\binom{\ell_1 + 5}{5}
	\binom{\ell_2 + 5}{5}
	\exp
	\left(
	- \frac{\epsilon \log (c_1m)}{4}
	\ell_1
	- \frac{\epsilon \log n}{2}
	\ell_2
	\right)
	\label{eq:largeErr_errComvg_2}\\
	&\quad\leq
	\sum_{\ell_1 = 1}^{6\tau \min\{\intraTau, \interTau\} m}
	2^{\left(\displaystyle{\ell_1 + 6}\right)}\:\:
	(c_1m)^{\left(\displaystyle - \frac{\epsilon}{4}
		\ell_1\right)}
	\:\:+\:\:
	\sum_{\ell_2 = 1}^{\relabelConst n}
	2^{\left(\displaystyle \ell_2 + 6\right)}\:\:
	n^{\left( \displaystyle -  \frac{\epsilon}{2} \ell_2\right)}
	\nonumber\\
	&\quad\phantom{\leq}\:
	+ \sum_{\ell_1 = 1}^{6\tau \min\{\intraTau, \interTau\} m}
	2^{\left(\displaystyle{\ell_1 + 6}\right)}\:\:
	(c_1m)^{\left(\displaystyle - \frac{\epsilon}{4}
		\ell_1\right)}
	\left(
	\sum_{\ell_2 = 1}^{\relabelConst n}
	2^{\left(\displaystyle \ell_2 + 6\right)}\:\:
	n^{\left( \displaystyle -  \frac{\epsilon}{2} \ell_2\right)}
	\right)
	\label{eq:largeErr_errComvg_3}\\
	&\quad\leq
	2^{\displaystyle 6}
	\sum_{\ell_1 = 1}^{\infty}
	\left(2\:(c_1m)^{\left(\displaystyle - \frac{\epsilon}{4}\right)}\right)^{\displaystyle \ell_1}
	\:\:+\:\:
	2^{\displaystyle 6}\:
	\sum_{\ell_2 = 1}^{\infty}
	\left(2 n^{\left( \displaystyle -  \frac{\epsilon}{2}\right)}\right)^{\displaystyle \ell_2}
	\nonumber\\
	&\quad\phantom{\leq}\:
	+ 2^{\displaystyle 12}\:
	\sum_{\ell_1 = 1}^{\infty}
	\left(2 \:(c_1 m)^{\left(\displaystyle - \frac{\epsilon}{4}\right)}\right)^{\displaystyle \ell_1}
	\left[
	\sum_{\ell_2 = 1}^{\infty}
	\left(2 n^{\left( \displaystyle -  \frac{\epsilon}{2}\right)}\right)^{\displaystyle \ell_2}
	\right]
	\nonumber\\
	&\quad=
	2^{\displaystyle 6}\:
	\frac{2 \:(c_1 m)^{\left(\displaystyle - \frac{\epsilon}{4}\right)}}{1 - 2 \:(c_1m)^{\left(\displaystyle - \frac{\epsilon}{4}\right)}}
	\:+\:
	2^{\displaystyle 6}\:
	\frac{2 n^{\left( \displaystyle -  \frac{\epsilon}{2}\right)}}{1 - 2 n^{\left( \displaystyle -  \frac{\epsilon}{2}\right)}}
	\:+\: 
	2^{\displaystyle 12}\:
	\frac{2\:(c_1m)^{\left(\displaystyle - \frac{\epsilon}{4}\right)}}{1 - 2 \:(c_1m)^{\left(\displaystyle - \frac{\epsilon}{4}\right)}}
	\frac{2 n^{\left( \displaystyle -  \frac{\epsilon}{2}\right)}}{1 - 2 n^{\left( \displaystyle - \frac{\epsilon}{2}\right)}},
	\label{eq:largeErr_errComvg_4}
\end{align}
where
\begin{itemize}[leftmargin=*]
	\item \eqref{eq:largeErr_errComvg_0} readily follows from \eqref{eq:Tlarge_complement_0};
	\item in \eqref{eq:largeErr_errComvg_1}, we break the summation into three summations, and use the fact that the enumeration of the first element of the set is independent of the enumeration of the second element;
	\item in \eqref{eq:largeErr_errComvg_2}, we use the fact that the number of integer solutions of $\sum_{i=1}^{n} x_i = s$ is equal to $\binom{s+n-1}{n-1}$;
	\item in \eqref{eq:largeErr_errComvg_3}, we bound each binomial coefficient by $\binom{a}{b} \leq \sum_{i=0}^a \binom{a}{i} \leq 2^a$, for $a \geq b$;
	\item and finally in \eqref{eq:largeErr_errComvg_4}, we evaluate the infinite geometric series, where $\epsilon > \max\{(2 \log 2) / \log n, \: (4 \log 2) / \log m\}$.
\end{itemize}
Therefore, by \eqref{eq:largeErr_errComvg_4}, the RHS of \eqref{eq:largeErr_LHS_app} is given by
\begin{align}
	\lim_{n\rightarrow \infty}
	\sum\limits_{T \in \TsmallErr} 
	\sum\limits_{X \in \mathcal{X}(T)}
	\exp\left(-\left(1+o(1)\right) \left(\lvert\diffEntriesSet\rvert \Ir 
	+ \Pleftrightarrow{\alphaEdge}{\betaEdge} \: \Ig 
	+ \Pleftrightarrow{\alphaEdge}{\gammaEdge} \: \IcOne 
	+ \Pleftrightarrow{\betaEdge}{\gammaEdge} \: \IcTwo  \right)\right) = 0.
\end{align}
Note that as $n$ tends to infinity, the condition on $\epsilon$ becomes 
\begin{align}
	\epsilon > \lim_{n\rightarrow \infty}
	\max\left\{
	\frac{2 \log 2}{\log n}, \: 
	\frac{4 \log 2}{\log(c_1m)},
	\frac{2\log c_1}{\log(m/c_1)}
	\right\} = 0.
\end{align}
This completes the proof of Lemma~\ref{lemma:T1_related}.
\hfill $\blacksquare$

\subsection{Proof of Lemma~\ref{lemma:eta_omega(nm)}}
\label{proof:eta_omega(nm)}
The LHS of \eqref{ineq:eta_omega(nm)} is upper bounded by
\begin{align}
	& \lim_{n\rightarrow \infty}
	\sum\limits_{T \in \TlargeErr} 
	\sum\limits_{X \in \mathcal{X}(T)}
	\exp\left(-\left(1+o(1)\right) \left(\lvert\diffEntriesSet\rvert \Ir 
	+ \Pleftrightarrow{\alphaEdge}{\betaEdge} \: \Ig 
	+ \Pleftrightarrow{\alphaEdge}{\gammaEdge} \: \IcOne 
	+ \Pleftrightarrow{\betaEdge}{\gammaEdge} \: \IcTwo \right)\right)
	\nonumber \\
	&\qquad \leq
	\lim_{n\rightarrow \infty}
	\sum\limits_{T \in \TlargeErr} 
	\left\vert \mathcal{X}(T) \right\vert
	\exp\left(-\left(\lvert\diffEntriesSet\rvert \Ir 
	+ \Pleftrightarrow{\alphaEdge}{\betaEdge} \: \Ig 
	+ \Pleftrightarrow{\alphaEdge}{\gammaEdge} \: \IcOne 
	+ \Pleftrightarrow{\betaEdge}{\gammaEdge} \: \IcTwo  \right)\right).
	\label{ineq:largeErr}
\end{align}
We first partition the set $\TlargeErr$ into two subsets (regimes), denoted by $\TlargeOne$ and $\TlargeTwo$. They are defined as follows:
\begin{align}
	\TlargeOne
	&=  
	\left\{ T \in \TlargeErr : 
	\exists (x,i) \in \{A,B\} \times [3],\:
	\text{ such that }
	\left(\left\vert \sigma(x,i) \right\vert = 0 \right)
	\right\}, \label{eq:Tlarge_1}
	\\
	\TlargeTwo
	&=  
	\left\{ T \in \TlargeErr : 
	\forall (x,i) \in \{A,B\} \!\times\! [3]
	\text{ such that }
	\left\vert \sigma(x,i) \right\vert = 1,
	\right. \nonumber \\
	&\phantom{=\{}\left.
	\text{ and }
	\exists (x,i) \in \{A,B\} \times [3]
	\text{ such that }
	\dElmntsGen{i}{\sigma(i|x)}{x}{\sigma(x)} > \tau m \min\{\interTau, \intraTau\}
	\right\}.
	\label{eq:Tlarge_2}
\end{align}
Therefore, the RHS of \eqref{ineq:largeErr} is upper bounded by
\begin{align}
	&\lim_{n\rightarrow \infty}
	\sum\limits_{T \in \TlargeErr} 
	\sum\limits_{X \in \mathcal{X}(T)}
	\exp\left(-\left(1+o(1)\right) \left(\lvert\diffEntriesSet\rvert \Ir 
	+ \Pleftrightarrow{\alphaEdge}{\betaEdge} \: \Ig 
	+ \Pleftrightarrow{\alphaEdge}{\gammaEdge} \: \IcOne 
	+ \Pleftrightarrow{\betaEdge}{\gammaEdge} \: \IcTwo  \right)\right)
	\nonumber \\
	&\qquad \leq
	\lim_{n\rightarrow \infty}
	\left[
	\sum\limits_{T \in \TlargeOne} 
	\left\vert \mathcal{X}(T) \right\vert
	\exp\left(-\left(\lvert\diffEntriesSet\rvert \Ir 
	+ \Pleftrightarrow{\alphaEdge}{\betaEdge} \: \Ig 
	+ \Pleftrightarrow{\alphaEdge}{\gammaEdge} \: \IcOne 
	+ \Pleftrightarrow{\betaEdge}{\gammaEdge} \: \IcTwo  \right)\right)
	\right.
	\nonumber\\
	&\qquad \phantom{\leq}
	\left.
	+ 
	\sum\limits_{T \in \TlargeTwo} 
	\left\vert \mathcal{X}(T) \right\vert
	\exp\left(-\left(\lvert\diffEntriesSet\rvert \Ir 
	+ \Pleftrightarrow{\alphaEdge}{\betaEdge} \: \Ig 
	+ \Pleftrightarrow{\alphaEdge}{\gammaEdge} \: \IcOne 
	+ \Pleftrightarrow{\betaEdge}{\gammaEdge} \: \IcTwo  \right)\right)
	\right].
	\label{ineq:largeErr_1}
\end{align}
In what follows, we derive upper bounds on each summation term in \eqref{ineq:largeErr_1}.

\subsubsection{Large Grouping Error Regime (Regime 1):}
This regime corresponds to $\TlargeOne$ characterized by \eqref{eq:Tlarge_1}. Suppose that there exist a cluster $x_0 \in \{A,B\}$ and a group $i_0 \in [3]$ such that $\left\vert \sigma(x_0,i_0) \right\vert = 0$. By \eqref{eq:sigma(x,i)}, this implies that
\begin{align}
	\kUsrs{i_0}{j}{x_0}{y}
	=
	\left\vert Z_0(x_0,i_0) \cap Z(y,j) \right\vert
	\leq 
	(1-\tau) \frac{n}{6},
	\:\:\forall (y,j) \in \{A,B\} \times [3].
	\label{eq:regime2_1}
\end{align}
We further partition the set $\TlargeOne$ into three subregimes $\TlargeOneOne$, $\TlargeOneTwo$, and $\TlargeOneThree$ that are defined as follows:
\begin{align}
	\TlargeOneOne
	&=
	\left\{
	T \in \TlargeOne:
	\exists \mu > 0,\:
	\exists (y_1, j_1) \in \{A,B\} \times [3],\:
	\exists (y_2, j_2) \in \{A,B\} \times [3]
	\right. \nonumber \\
	&\phantom{=\{}\left.
	\text{ such that }
	\kUsrs{i_0}{j_1}{x_0}{y_1} \geq \mu n,\:
	\kUsrs{i_0}{j_2}{x_0}{y_2} \geq \mu n
	\right\}, \label{eq:TlargeOneOne} 
	\\
	\TlargeOneTwo
	&=
	\left\{
	T \in \TlargeOne:
	\exists \mu > 0,\:
	\exists (y_1, j_1) \in \{A,B\} \times [3]
	\text{ such that }
	\kUsrs{i_0}{j_1}{x_0}{y_1} \geq \mu n
	\right\}, 
	\label{eq:TlargeOneTwo}\\
	\TlargeOneThree
	&=
	\left\{
	T \in \TlargeOne:
	\forall \mu > 0,\:
	\forall (y, j) \in \{A,B\} \times [3]
	\text{ such that }
	\kUsrs{i_0}{j}{x_0}{y} < \mu n
	\right\}
	. \label{eq:TlargeOneThree}
\end{align}

\textbf{(i) Subregime 1-1:}
This subregime corresponds to $\TlargeOneOne$ characterized by \eqref{eq:TlargeOneOne}. Suppose that there exist a constant $\mu > 0$, and two distinct pairs $(y_1,j_1), (y_2,j_2) \in \{A,B\}\times[3]$ such that 
\begin{align}
	\kUsrs{i_0}{j_1}{x_0}{y_1}
	\geq 
	\mu n,
	\text{ and }
	\kUsrs{i_0}{j_2}{x_0}{y_2}
	\geq
	\mu n.
	\label{prop:R2_largeErr}
\end{align}
There are $n$ users, each of which belong to one of the $6$ groups. Moreover, each user rates $m$ items, where each item rating can be one of $2^{6}$ possible ratings across all users. Therefore, a loose upper bound on the number of matrices that belong to matrix class $\mathcal{X}(T)$ is given by
\begin{align}
	\left\vert \mathcal{X}(T) \right\vert 
	\leq 
	6^n \left(2^{6}\right)^{m},
	\quad \forall T \in \tupleSetDelta.
	\label{ineq:cardinality}
\end{align}
Next, a lower bound on $\lvert\diffEntriesSet\rvert$ is given by
\begin{align}
	\lvert\diffEntriesSet\rvert
	&=
	\sum\limits_{x \in \{A,B\}}
	\sum\limits_{i \in [3]}
	\sum\limits_{y \in \{A,B\}}
	\sum\limits_{j \in [3]}
	\kUsrs{i}{j}{x}{y} \dElmntsGen{i}{j}{x}{y}
	\label{eq:regime2_Pd_0}\\
	&\geq
	\kUsrs{i_0}{j_1}{x_0}{y_1} \dElmntsGen{i_0}{j_1}{x_0}{y_1}
	+
	\kUsrs{i_0}{j_2}{x_0}{y_2} \dElmntsGen{i_0}{j_2}{x_0}{y_2}
	\nonumber\\
	&> 
	\mu n
	\left(
	\dElmntsGen{i_0}{j_1}{x_0}{y_1}
	+ \dElmntsGen{i_0}{j_2}{x_0}{y_2}
	\right)
	\label{eq:regime2_Pd_1}\\
	&\geq
	\mu n\:
	\hamDist{\vecV{j_1}{y_1}}{\vecV{j_2}{y_2}} 
	\label{eq:regime2_Pd_2}\\
	&\geq
	\mu\min\left\{\intraTau,\interTau \right\}
	nm,
	\label{eq:regime2_Pd_3}
\end{align}
where \eqref{eq:regime2_Pd_0} follows from the definitions of $\diffEntriesSet$, $\kUsrs{i}{j}{x}{y}$ and $\dElmntsGen{i}{j}{x}{y}$ in \eqref{eq:N1_defn}; \eqref{eq:regime2_Pd_1} follows from \eqref{prop:R2_largeErr}; \eqref{eq:regime2_Pd_2} follows from the triangle inequality; and \eqref{eq:regime2_Pd_3} holds since the minimum hamming distance between any two different rating vectors in $\cV$ is $\min\{\intraTau,\interTau\}m$. Furthermore, if $y_1 = y_2$, then $\lvert\misClassfSet{\alphaEdge}{\betaEdge} \rvert$ is lower bounded by
\begin{align}
	\lvert\misClassfSet{\alphaEdge}{\betaEdge} \rvert
	&=
	\sum_{x\in\{A,B\}}
	\sum_{y\in\{A,B\}}
	\sum_{i\in[3]}
	\sum_{k\in[3]}
	\sum_{j\in[3]\setminus k}
	\kUsrs{i}{j}{x}{y}\kUsrs{i}{k}{x}{y} 
	\label{eq:regime2_Palphabeta_0}\\
	&\geq
	\kUsrs{i_0}{j_1}{x_0}{y_1}\kUsrs{i_0}{j_2}{x_0}{y_2}
	\nonumber \\
	&\geq 
	(\mu n)^2,
	\label{eq:regime2_Palphabeta}
\end{align}
where \eqref{eq:regime2_Palphabeta_0} follows from the definitions in \eqref{eq:N2_defn}. On the other hand, if $y_1\neq y_2$, then $\lvert\misClassfSet{\alphaEdge}{\gammaEdge}\rvert$ is lower bounded by
\begin{align}
	\lvert\misClassfSet{\alphaEdge}{\gammaEdge}\rvert
	&=
	\sum_{x \in \{A,B\}}
	\sum_{z \in \{A,B\}}
	\sum_{y \in \{A,B\} \setminus z}
	\sum_{i \in [3]}
	\sum_{k \in [3]}
	\sum_{j \in [3]}
	\kUsrs{i}{j}{x}{y}  \kUsrs{i}{k}{x}{z} 
	\label{eq:regime2_Palphagamma_0}\\
	&\geq
	\kUsrs{i_0}{j_1}{x_0}{y_1}\kUsrs{i_0}{j_2}{x_0}{y_2}
	\nonumber \\
	&\geq 
	(\mu n)^2,
	\label{eq:regime2_Palphagamma}
\end{align}
where \eqref{eq:regime2_Palphagamma_0} follows from the definitions in \eqref{eq:N3_defn}. Finally, the first summation term in the RHS of \eqref{ineq:largeErr} is upper bounded by
\begin{align}
	&\sum\limits_{T \in \TlargeOneOne} 
	\left\vert \mathcal{X}(T) \right\vert
	\exp\left(-\left(\lvert\diffEntriesSet\rvert \Ir 
	+ \Pleftrightarrow{\alphaEdge}{\betaEdge} \: \Ig 
	+ \Pleftrightarrow{\alphaEdge}{\gammaEdge} \: \IcOne 
	+ \Pleftrightarrow{\betaEdge}{\gammaEdge} \: \IcTwo \right)\right)
	\nonumber\\
	&\leq
	\sum\limits_{T \in \TlargeOneOne} 
	\left\vert \mathcal{X}(T) \right\vert
	\exp\left(
	-\left(\lvert\diffEntriesSet\rvert \Ir 
	+  \frac{\Ig}{2}\lvert \misClassfSet{\alphaEdge}{\betaEdge} \rvert
	+  \frac{\IcOne}{2} \lvert \misClassfSet{\alphaEdge}{\gammaEdge} \rvert
	\right)\right)
	\nonumber\\
	&=
	\exp\left(
	-\left(c_2 nm \frac{\log m}{n}
	+  c_3 \frac{\log n}{n} n^2
	\right)\right)
	\sum\limits_{T \in \TlargeOneOne} 
	\left\vert \mathcal{X}(T) \right\vert
	\label{eq:largeErrConvg_R2_0}\\
	&\leq 
	\exp\left(
	-\left(c_2 m \log m + c_3 n \log n
	\right)\right)
	\exp\left(n \log 6\right)
	\exp\left(m 6 \log 2\right)
	\label{eq:largeErrConvg_R2_1}\\
	&= 
	\exp\left(
	-\left(m \left(c_2 \log m - 6 \log 2\right) 
	+ n \left(c_3 \log n - \log 6\right)
	\right)\right),
	\label{eq:largeErrConvg_R2_2}
\end{align}
where $c_2$ and $c_3$ in \eqref{eq:largeErrConvg_R2_0} are some positive constants; \eqref{eq:largeErrConvg_R2_0} follows from \eqref{eq:suffCond_1}, \eqref{eq:regime2_Pd_3}, \eqref{eq:regime2_Palphabeta} and \eqref{eq:regime2_Palphagamma}; and \eqref{eq:largeErrConvg_R2_1} follows from \eqref{ineq:cardinality}. 

\textbf{(ii) Subregime 1-2:}
This subregime corresponds to $\TlargeOneTwo$ characterized by \eqref{eq:TlargeOneTwo}. Suppose that there exist only one pair $(y_1,j_1) \in \{A,B\} \times [3]$, and a constant $\mu > 0$ such that 
\begin{align}
	\kUsrs{i_0}{j_1}{x_0}{y_1} \geq \mu n.
\end{align}
This implies that
\begin{align}
	\kUsrs{i_0}{j}{x_0}{y} < \frac{\tau}{30} n, 
	\:\:\text{ for }
	(y,j) \neq (y_0, j_0). 
	\label{ineq:regime1_2_1}
\end{align}
Therefore, by \eqref{ineq:regime1_2_1}, we have
\begin{align}
	\kUsrs{i_0}{j_0}{x_0}{y_0} 
	&=
	\frac{n}{6} - \sum_{(y,j) \neq (y_0, j_0)} \kUsrs{i_0}{j}{x_0}{y}
	\nonumber \\
	&> 
	\frac{n}{6} - 5 \frac{\tau}{5} \frac{n}{6}
	\nonumber\\
	&=
	(1-\tau)\frac{n}{6}. \label{ineq:regime1_2_2}
\end{align}
However, this is in contradiction with \eqref{eq:regime2_1}. Hence, we conclude that subregime $\TlargeOneTwo$ is impossible to exist.

\textbf{(iii) Subregime 1-3:}
This subregime corresponds to $\TlargeOneThree$ characterized by \eqref{eq:TlargeOneThree}. Due to the fact that
\begin{align}
	\sum_{y\in\{A,B\}}\sum_{j\in[3]}\kUsrs{i_0}{j}{x_0}{y} = \frac{n}{6}, 
\end{align}
there should be at least one pair $(y_1, j_1)$ such that 
\begin{align}
	\kUsrs{i_0}{j_1}{x_0}{y_1} \geq \mu n,
\end{align}
for some $\mu > 0$. However, this is in contradiction with \eqref{eq:TlargeOneThree}. Thus, we conclude that subregime $\TlargeOneThree$ is impossible to exist.

As a result, we conclude that
\begin{align}
	&\sum\limits_{T \in \TlargeOne} 
	\left\vert \mathcal{X}(T) \right\vert
	\exp\left(-\left(\lvert\diffEntriesSet\rvert \Ir 
	+ \Pleftrightarrow{\alphaEdge}{\betaEdge} \: \Ig 
	+ \Pleftrightarrow{\alphaEdge}{\gammaEdge} \: \IcOne 
	+ \Pleftrightarrow{\betaEdge}{\gammaEdge} \: \IcTwo  \right)\right)
	\nonumber\\
	&\qquad \leq
	\exp\left(
	-\left(m \left(c_2 \log m - 6 \log 2\right) 
	+ n \left(c_3 \log n - \log 6\right)
	\right)\right).
	\label{eq:R1_upperBound}
\end{align}

\subsubsection{Large Rating Estimation Error Regime (Regime 2):}
This regime corresponds to $\TlargeTwo$ characterized by \eqref{eq:Tlarge_2}. Suppose the following conditions hold:
\begin{itemize}[leftmargin=*]
	\item For every $(x,i) \in \{A,B\}\times [3]$, there exists a pair $(y,j) \in \{A,B\}\times[3]$ such that $\left\vert \sigma(x,i) \right\vert = 1$. That is, by \eqref{eq:sigma(x,i)}, we have
	\begin{align}
		\exists (y,j) = (\sigma(x), \sigma(i|x)): 
		\kUsrs{i}{j}{x}{y} = 
		\left| Z_0(x,i) \cap  Z(y,j)\right| 
		\geq (1-\relabelConst) \frac{n}{6}, 
		\forall (x,i) \!\in\! \{A,B\}\!\times\! [3].
		\label{eq:regime3_large_1}
	\end{align}
	\item There exists $(x_0, i_0) \in \{A,B\}\times[3]$ such that $\left\vert \sigma(x_0,i_0) \right\vert = 1$, and
	\begin{align}
		\dElmntsGen{i_0}{j_0}{x_0}{y_0}
		= \dElmntsGen{i_0}{\sigma(i_0 | x_0)}{x_0}{\sigma(x_0)} > \tau m \min\{\interTau, \intraTau\}.
		\label{eq:regime3_large_2}
	\end{align}
\end{itemize}

We first provide an upper bound on $\left\vert \mathcal{X}(T) \right\vert$. By \eqref{eq:term1_11_12}, \eqref{eq:term11_1} and \eqref{ineq:cardinality}, an upper bound on $\left\vert \mathcal{X}(T) \right\vert$ is given by
\begin{align}
	\left\vert \mathcal{X}(T) \right\vert
	&\leq
	\left(2^{6}\right)^{m}
	\exp
	\left(
	\left(
	\sum_{x \in \{A,B\}} 
	\sum_{i \in [3]} 
	\sum_{(y,j) \neq (\sigma(x), \sigma(i|x))} 
	\kUsrs{i}{j}{x}{y}
	\right)
	\log n
	\right).
	\label{eq:XT_regime2}
\end{align}

Next, we provide a lower bound on $\lvert\diffEntriesSet\rvert$. Based on \eqref{eq:regime3_large_1}, if there exists at least one other pair $(\widehat{y},\widehat{j}) \neq (\sigma(x), \sigma(i|x))$ for some $(x,i) \in \{A,B\} \times [3]$ such that
\begin{align}
	\kUsrs{i}{\widehat{j}}{x}{\widehat{y}} 
	= \left| Z_0(x,i) \cap  Z(\widehat{y},\widehat{j})\right| 
	\geq \mu n,
\end{align}
for some constant $\mu > 0$, then the analysis of this case boils down to Subregime 1-1. Therefore, we assume that for every $(x,i) \in \{A,B\}\times [3]$, we have
\begin{align}
	\kUsrs{i}{j}{x}{y} < \mu n,
	\quad \forall (y,j) \neq (\sigma(x), \sigma(i|x)),
	\: \forall \mu > 0.
	\label{eq:regime3_large_3}
\end{align} 
Consequently, a lower bound on $\lvert\diffEntriesSet\rvert$ is given by
\begin{align}
	\lvert\diffEntriesSet\rvert
	&=
	\sum\limits_{x \in \{A,B\}}
	\sum\limits_{i \in [3]}
	\sum\limits_{y \in \{A,B\}}
	\sum\limits_{j \in [3]}
	\kUsrs{i}{j}{x}{y} \dElmntsGen{i}{j}{x}{y}
	\nonumber\\
	&\geq
	\kUsrs{i_0}{j_0}{x_0}{y_0} \dElmntsGen{i_0}{j_0}{x_0}{y_0}
	\nonumber\\
	&>
	\left((1-\relabelConst) \frac{n}{6}\right)
	\left(\tau m \min\{\interTau, \intraTau\} \right)
	\label{eq:regime3_Pd_0}\\
	&=
	\left( \frac{(1-\relabelConst) \tau \min\{\interTau, \intraTau\}}{12} \right)
	m n
	+
	\left( \frac{(1-\relabelConst) \tau \min\{\interTau, \intraTau\}}{12} \right)
	m n
	\nonumber \\
	&\geq
	\left( \frac{(1-\relabelConst) \tau \min\{\interTau, \intraTau\}}{12} \right)
	m n
	+
	((1-\relabelConst)m)
	\left(
	\sum_{x \in \{A,B\}} 
	\sum_{i \in [3]} 
	\sum_{(y,j) \neq \sigma(x,i)} 
	\kUsrs{i}{j}{x}{y}
	\right)
	\label{eq:regime3_Pd_1}\\
	&=
	\left( \frac{(1-\relabelConst) \tau \min\{\interTau, \intraTau\}}{12} \right)
	m n
	\nonumber \\
	&\phantom{=}
	+
	((1-\relabelConst)m)
	\left(
	\sum\limits_{x \in \{A,B\}}
	\sum\limits_{i \in [3]}
	\sum\limits_{j \in [3]\setminus \sigma(i|x)}
	\kUsrs{i}{j}{x}{\sigma(x)} 
	+
	\sum\limits_{x \in \{A,B\}}
	\sum\limits_{y \in \{A,B\} \setminus \sigma(x)}
	\sum\limits_{i \in [3]}
	\sum\limits_{j \in [3]}
	\kUsrs{i}{j}{x}{y}
	\right)
	\nonumber\\
	&\geq
	c_4 m n
	+
	((1-\relabelConst)m)
	\left[
	\intraTau
	\left(
	\sum\limits_{x \in \{A,B\}}
	\sum\limits_{i \in [3]}
	\sum\limits_{j \in [3]\setminus \sigma(i|x)}
	\kUsrs{i}{j}{x}{\sigma(x)} 
	\right)
	\right.\nonumber \\
	&\phantom{\geq
		c_4 m n
		+
		((1-\relabelConst)m)
		[}\left.
	+
	\interTau
	\left(
	\sum\limits_{x \in \{A,B\}}
	\sum\limits_{y \in \{A,B\} \setminus \sigma(x)}
	\sum\limits_{i \in [3]}
	\sum\limits_{j \in [3]}
	\kUsrs{i}{j}{x}{y}
	\right)
	\right],
	\label{eq:regime3_Pd_2}
\end{align}
where \eqref{eq:regime3_Pd_0} follows from \eqref{eq:regime3_large_1} and \eqref{eq:regime3_large_2}; \eqref{eq:regime3_Pd_1} follows from \eqref{eq:regime3_large_3} for $\mu = (\tau \min\{\intraTau, \interTau\}) / (10\cdot6^2)$; and \eqref{eq:regime3_Pd_2} follows by setting $c_4 = ((1-\relabelConst) \tau \min\{\interTau, \intraTau\})/12$ where $0 \leq c_4 < 1$.

On the other hand, recall from \eqref{ineq:Ng}, \eqref{ineq:Nc1} and \eqref{ineq:Nc2} the following  lower bounds: 
\begin{align}
	\begin{split}
		\frac{
			\lvert\misClassfSet{\betaEdge}{\alphaEdge}\rvert
			+ \lvert\misClassfSet{\alphaEdge}{\betaEdge}\rvert}
		{2}
		&\geq
		\left(1 -\tau\right) \frac{n}{6}
		\sum\limits_{x\in\{A,B\}} 
		\sum_{i\in[3]} 
		\sum_{j\in[3] \setminus \sigma(i|x)} 
		\kUsrs{i}{j}{x}{\sigma(x)},
		\\
		\frac{
			\lvert\misClassfSet{\gammaEdge}{\alphaEdge}\rvert
			+
			\lvert\misClassfSet{\alphaEdge}{\gammaEdge}\rvert}{2}
		&\geq
		\left(1 -\tau\right) \frac{n}{6}
		\sum_{x \in \{A,B\}}
		\sum_{y \in \{A,B\} \setminus \sigma(x)}
		\sum_{i \in [3]}
		\sum_{j \in [3]}
		\kUsrs{i}{j}{x}{y},
		\\
		\frac{
			\lvert\misClassfSet{\gammaEdge}{\betaEdge}\rvert
			+
			\lvert\misClassfSet{\betaEdge}{\gammaEdge}\rvert}{2}
		&\geq
		\left(1 -\tau\right) \frac{n}{3}
		\sum_{x \in \{A,B\}}
		\sum_{y \in \{A,B\} \setminus \sigma(x)}
		\sum_{i \in [3]}
		\sum_{j \in [3]}
		\kUsrs{i}{j}{x}{y}.
		\label{eq:regime3_abg_2}
	\end{split}
\end{align}

Finally, the second summation term in the RHS of \eqref{ineq:largeErr} is upper bounded by
\begin{align}
	&\sum\limits_{T \in \TlargeTwo} 
	\left\vert \mathcal{X}(T) \right\vert
	\exp\left(-\left(\lvert\diffEntriesSet\rvert \Ir 
	+ \Pleftrightarrow{\alphaEdge}{\betaEdge} \: \Ig 
	+ \Pleftrightarrow{\alphaEdge}{\gammaEdge} \: \IcOne 
	+ \Pleftrightarrow{\betaEdge}{\gammaEdge} \: \IcTwo  \right)\right)
	\nonumber\\
	&\leq
	\sum\limits_{T \in \TlargeTwo} 
	\left\vert \mathcal{X}(T) \right\vert
	\exp\left(
	- c_4 m n \frac{\log m}{n}
	\right)
	\nonumber \\
	&\phantom{\leq
		\sum\limits_{T \in \TlargeTwo}}
	\exp \left(
	- (1-\relabelConst)
	\left(
	\left(
	\intraTau m \Ir + \frac{n\Ig}{6} 
	\right)
	\sum\limits_{x \in \{A,B\}}
	\sum\limits_{i \in [3]}
	\sum\limits_{j \in [3]\setminus \sigma(i|x)}
	\kUsrs{i}{j}{x}{\sigma(x)} 
	\right.
	\right.
	\nonumber\\
	&\phantom{\leq \sum\limits_{T \in \TlargeTwo}\exp}
	\left.
	\left.
	+ \left(
	\interTau m \Ir + \frac{n \IcOne}{6} + \frac{ (g-1) n \IcTwo}{6} 
	\right)
	\sum_{x\in\{A,B\}}\sum_{y\in\{A,B\}\setminus x} \sum_{i\in[3]} \sum_{j\in[3]} \kUsrs{i}{j}{x}{y} 
	\right)
	\right)
	\label{eq:largeErrConvg_R1_2}\\
	&\leq 
	\sum\limits_{T \in \TlargeTwo} 
	\exp
	\left(
	\log n
	\left(
	\sum_{x \in \{A,B\}} 
	\sum_{i \in [3]} 
	\sum_{(y,j) \neq \sigma(x,i)} 
	\kUsrs{i}{j}{x}{y}
	\right)
	\right)
	\times
	\left(2^{6}\right)^{m}
	\exp\left(
	- c_4 m \log m
	\right)
	\nonumber\\
	&\phantom{\leq}
	\times 
	\exp \left(
	- (1-\relabelConst) (1+\epsilon) \log n
	\left(
	\sum_{x \in \{A,B\}} 
	\sum_{i \in [3]} 
	\sum_{(y,j) \neq \sigma(x,i)} 
	\kUsrs{i}{j}{x}{y}
	\right)
	\right)
	\label{eq:largeErrConvg_R1_3}\\
	&\leq 
	\sum\limits_{T \in \TlargeTwo} 
	\left(2^{6}\right)^{m}
	\exp\left(
	- c_4 m \log m
	\right)
	\exp
	\left(
	\log n
	\left(
	\sum_{x \in \{A,B\}} 
	\sum_{i \in [3]} 
	\sum_{(y,j) \neq \sigma(x,i)} 
	\kUsrs{i}{j}{x}{y}
	\right)
	\right)
	\nonumber\\
	&\phantom{\leq}
	\times 
	\exp \left(
	- \left(1+\frac{\epsilon}{2}\right) \log n
	\left(
	\sum_{x \in \{A,B\}} 
	\sum_{i \in [3]} 
	\sum_{(y,j) \neq \sigma(x,i)} 
	\kUsrs{i}{j}{x}{y}
	\right)
	\right)
	\label{eq:largeErrConvg_R1_3_1}\\
	&=
	\exp\left(
	-m \left(c_4 \log m - 6 \log 2\right) 
	\right)
	\sum\limits_{T \in \TlargeTwo} 
	\exp \left(
	- \frac{\epsilon}{2} \log n
	\left(
	\sum_{x \in \{A,B\}} 
	\sum_{i \in [3]} 
	\sum_{(y,j) \neq \sigma(x,i)} 
	\kUsrs{i}{j}{x}{y}
	\right)
	\right)
	\nonumber\\
	&=
	\exp\left(
	-m \left(c_4 \log m - 6 \log 2\right) 
	\right)
	\sum_{\ell = 0}^{\relabelConst n}
	\left| \left\{
	\sum_{x \in \{A,B\}} 
	\sum_{i \in [3]} 
	\sum_{(y,j) \neq (\sigma(x),\sigma(i|x))} 
	\kUsrs{i}{j}{x}{y} = \ell
	\right\} \right|
	\exp
	\left(
	- \frac{\epsilon \log n}{2}
	\ell
	\right)
	\label{eq:largeErrConvg_R1_5}\\
	&=
	\exp\left(
	-m \left(c_4 \log m - 6 \log 2\right) 
	\right)
	\left[
	\left| \left\{
	\sum_{x \in \{A,B\}} 
	\sum_{i \in [3]} 
	\sum_{(y,j) \neq (\sigma(x),\sigma(i|x))} 
	\kUsrs{i}{j}{x}{y} = 0
	\right\} \right|
	\right.
	\nonumber\\
	&\phantom{=}
	\left.
	+ 
	\sum_{\ell = 1}^{\relabelConst n}
	\left| \left\{
	\sum_{x \in \{A,B\}} 
	\sum_{i \in [3]} 
	\sum_{(y,j) \neq (\sigma(x),\sigma(i|x))} 
	\kUsrs{i}{j}{x}{y} = \ell
	\right\} \right|
	\exp
	\left(
	- \frac{\epsilon \log n}{2}
	\ell
	\right)
	\right]
	\label{eq:largeErrConvg_R1_5_1}\\
	&=
	\exp\left(
	-m \left(c_4 \log m - 6 \log 2\right) 
	\right)
	\left[
	1
	+
	\sum_{\ell = 1}^{\relabelConst n}
	\binom{\ell + 6}{6}
	\exp
	\left(
	- \frac{\epsilon \log n}{2}
	\ell
	\right)
	\right]
	\label{eq:largeErrConvg_R1_6}\\
	&\leq
	\exp\left(
	-m \left(c_4 \log m - 6 \log 2\right) 
	\right)
	\left[
	1
	+
	\sum_{\ell_2 = 1}^{\relabelConst n}
	2^{\left(\displaystyle \ell_2 + 6\right)}\:\:
	n^{\left( \displaystyle -  \frac{\epsilon}{2} \ell_2\right)}
	\right]
	\label{eq:largeErrConvg_R1_7}\\
	&\leq
	\exp\left(
	-m \left(c_4 \log m - 6 \log 2\right) 
	\right)
	\left[
	1
	+
	2^{\displaystyle 6}
	\sum_{\ell_2 = 1}^{\infty}
	\left(2 n^{\left( \displaystyle -  \frac{\epsilon}{2}\right)}\right)^{\displaystyle \ell_2}
	\right]
	\nonumber\\
	&=
	\exp\left(
	-m \left(c_4 \log m - 6 \log 2\right) 
	\right)
	\left[
	1
	+
	2^{\displaystyle 6}
	\frac{2 n^{\left( \displaystyle -  \frac{\epsilon}{2}\right)}}{1 - 2 n^{\left( \displaystyle -  \frac{\epsilon}{2}\right)}}
	\right],
	\label{eq:largeErrConvg_R1_8}
\end{align}
where 
\begin{itemize}[leftmargin=*]
	\item \eqref{eq:largeErrConvg_R1_2} follows from \eqref{eq:regime3_Pd_2} and \eqref{eq:regime3_abg_2};
	\item \eqref{eq:largeErrConvg_R1_3} follows from the sufficient conditions in \eqref{eq:suffCond_1}, \eqref{eq:suffCond_2} and \eqref{eq:suffCond_3};
	\item \eqref{eq:largeErrConvg_R1_3_1} follows from $\tau \leq  (\epsilon\log m - (2+\epsilon)\log 4) / (2 (1 + \epsilon)\log m))$ which implies that $(1-\tau)(1+\epsilon) \geq (1+(\epsilon/2))$;
	\item \eqref{eq:largeErrConvg_R1_5} readily follows from \eqref{eq:regime3_large_1};
	\item in \eqref{eq:largeErrConvg_R1_5_1}, we break the summation into two summations, and use the fact that the enumeration of the first element of the set is independent of the enumeration of the second element;
	\item in \eqref{eq:largeErrConvg_R1_6}, we use the fact that the number of integer solutions of $\sum_{i=1}^{n} x_i = s$ is equal to $\binom{s+n-1}{n-1}$;
	\item in \eqref{eq:largeErrConvg_R1_7}, we bound each binomial coefficient by $\binom{a}{b} \leq \sum_{i=0}^a \binom{a}{i} \leq 2^a$, for $a \geq b$;
	\item and finally in \eqref{eq:largeErrConvg_R1_8}, we evaluate the infinite geometric series, where $\epsilon > (2 \log 2) / \log n$.
\end{itemize}

By \eqref{eq:R1_upperBound} and \eqref{eq:largeErrConvg_R1_8}, the RHS of \eqref{ineq:largeErr_1} is upper bounded by
\begin{align}
	&
	\lim_{n\rightarrow \infty}
	\sum\limits_{T \in \TlargeErr} 
	\sum\limits_{X \in \mathcal{X}(T)}
	\exp\left(-\left(1+o(1)\right) \left(\lvert\diffEntriesSet\rvert \Ir 
	+ \Pleftrightarrow{\alphaEdge}{\betaEdge} \: \Ig 
	+ \Pleftrightarrow{\alphaEdge}{\gammaEdge} \: \IcOne 
	+ \Pleftrightarrow{\betaEdge}{\gammaEdge} \: \IcTwo  \right)\right)
	\nonumber \\
	&\qquad \leq
	\lim_{n\rightarrow \infty}
	\exp\left(
	-\left(m \left(c_2 \log m - 6 \log 2\right) 
	+ n \left(c_3 \log n - \log 6\right)
	\right)\right)
	\nonumber\\
	&\qquad \phantom{\leq}
	+ 
	\lim_{n\rightarrow \infty}
	\exp\left(
	-m \left(c_4 \log m - 6 \log 2\right) 
	\right)
	\left[
	1
	+
	2^{\displaystyle 6}
	\frac{2 n^{\left( \displaystyle -  \frac{\epsilon}{2}\right)}}{1 - 2 n^{\left( \displaystyle -  \frac{\epsilon}{2}\right)}}
	\right]
	\nonumber\\
	&\qquad =
	0.
\end{align}
Note that as $n$ tends to infinity, the condition on $\epsilon$ becomes 
\begin{align}
	\epsilon > \lim_{n\rightarrow \infty}
	\max\left\{
	\frac{2 \log 2}{\log n}, \: 
	\frac{4 \log 2}{\log(c_1m)},
	\frac{2\log c_1}{\log(m/c_1)}
	\right\} = 0.
\end{align}
This completes the proof of Lemma~\ref{lemma:eta_omega(nm)}.
\hfill $\blacksquare$

\section{Proofs of Lemmas for Converse Proof of Theorem~\ref{Thm:p_star}}
\subsection{Proof of Lemma~\ref{lm:lowerB_prob}}
\label{app:lowerB_prob_proof}
We will follow a similar proof technique to that of Lemma~5.2 in \cite{zhang2016minimax}. Recall that we denote by $B^{(\mu)}$ a Bernoulli random variable with parameter $\mu$, that is, $\mathbb{P}[B^{(\mu)}=1] = 1- \mathbb{P}[B^{(\mu)}=0] = \mu$. 

For $p=\Theta\left(\frac{\log n}{n}\right)$ and a constant $\theta \in[0,1]$, we can define $X(p,\theta) = \log\left(\frac{1-\theta}{\theta}\right) B^{(p)} (2B^{(\theta)}-1)$, with $c'=\log\left(\frac{1-\theta}{\theta}\right)$, that is, 
\begin{align*}
	X(p,\theta) = \left\{
	\begin{array}{ll}
		-\log\left(\frac{1-\theta}{\theta}\right) & \textrm{w.p. } p(1-\theta),\\
		0 & \textrm{w.p. } 1-p,\\
		\log\left(\frac{1-\theta}{\theta}\right)  & \textrm{w.p. }p\theta.
	\end{array}
	\right.
\end{align*}
Then, we can evaluate the moment generating function of $X(p,\theta)$ at $t=1/2$ as
\begin{align}
	M_{X(p,\theta)}\left(\frac{1}{2}\right) 
	&= \mathbb{E}\left[\exp(X/2)\right]
	\nonumber\\
	&= 
	p(1-\theta)  \exp\left(-\frac{1}{2}\log\left(\frac{1-\theta}{\theta}\right)\right) +  (1-p) + p\theta  \exp\left(\frac{1}{2}\log\left(\frac{1-\theta}{\theta}\right)\right)\nonumber\\
	&=
	p (1-\theta) \sqrt{\frac{\theta}{1-\theta}} +  (1-p) + p\theta \sqrt{\frac{1-\theta}{\theta}}  \nonumber\\
	&= 2p \sqrt{\theta(1-\theta)} + 1-p,
	\label{eq:M-X}
\end{align}
which implies 
\begin{align}
	-\log M_{X(p,\theta)}\left(\frac{1}{2}\right) = (1+o(1))  (\sqrt{1-\theta} - \sqrt{\theta})^2 p.
\end{align}
We also define $\widehat{X}=\widehat{X}(p,\theta)$ as a new random variable with the same range as  $X(p,\theta)$, and probability mass function given by
\begin{align*}
	f_{\widehat{X}} (x) = \frac{\exp(\frac{x}{2}) f_X(x)}{M_X(\frac{1}{2})}.
\end{align*}
More precisely, we have
\begin{align*}
	\widehat{X}(p,\theta) = \left\{
	\begin{array}{ll}
		-\log\left(\frac{1-\theta}{\theta}\right) & \textrm{w.p. } \frac{p\sqrt{\theta(1-\theta)}}{M_X(\frac{1}{2})},\\
		0 & \textrm{w.p. } \frac{1-p}{M_X(\frac{1}{2})},\\
		\log\left(\frac{1-\theta}{\theta}\right)  & \textrm{w.p. }\frac{p\sqrt{\theta(1-\theta)}}{M_X(\frac{1}{2})}.
	\end{array}
	\right.
\end{align*}
Then it is straightforward to see that 
\begin{align}
	\mathbb{E}[\widehat{X}(p,\theta)] 
	&=0
	\label{eq:E-hat-X}\\
	\mathsf{Var}[\widehat{X}(p,\theta)]
	&= 
	\frac{2p\sqrt{\nu(1-\theta)}}{2p \sqrt{\theta (1-\theta)} + 1-p} \left(\log\frac{1-\theta}{\theta}\right)^2 = O(p).\label{eq:Var-hat-X}
\end{align}
Next, for $\mu,\nu = \Theta\left(\frac{\log n}{n}\right) [0,1]$, define $Y(\mu,\nu)=c (B^{(\mu)} - B^{(\nu)})$, where $c=\log\left(\frac{(1-\mu)\nu}{(1-\nu)\mu}\right)$.  More precisely, we have  
\begin{align*}
	Y(\mu,\nu) =\left\{
	\begin{array}{ll}
		-\log\left(\frac{(1-\mu)\nu}{(1-\nu)\mu}\right) & \textrm{w.p. } (1-\mu)\nu,\\
		0 & \textrm{w.p. } (1-\mu)(1-\nu) + \mu\nu,\\
		\log\left(\frac{(1-\mu)\nu}{(1-\nu)\mu}\right) & \textrm{w.p. } \mu(1-\nu).
	\end{array}
	\right.
\end{align*}
The moment generating function of $Y(\mu,\nu)$ at $t=1/2$ is given by
\begin{align}
	M_{Y(\mu,\nu)}\left(\frac{1}{2}\right) &= \mathbb{E}[\exp(Y/2)] 
	\nonumber\\
	&= (1-\mu)\nu \exp(-c/2) + \mu(1-\nu)\exp(c/2) + (1-\mu)(1-\nu) + \mu\nu\nonumber\\
	&= (1-\mu)\nu \sqrt{\frac{(1-\nu)\mu}{(1-\mu)\nu}}+ (1-\nu)\mu \sqrt{\frac{(1-\mu)\nu}{(1-\nu)\mu}}
	+ (1-\mu)(1-\nu) + \mu\nu\nonumber\\
	&= 2 \sqrt{(1-\mu)(1-\nu)\mu\nu}
	+ (1-\mu)(1-\nu) + \mu\nu\nonumber\\
	&= \left(\sqrt{\mu\nu}+\sqrt{(1-\mu)(1-\nu)}\right)^2,
	\label{eq:M-Y}
\end{align}
which implies 
\begin{align}
	-\log M_{Y(\mu,\nu)}\left(\frac{1}{2}\right) = (1+o(1)) \left(\sqrt{\nu}-\sqrt{\mu}\right)^2.
\end{align}
Define a random variable $\widehat{Y} = \widehat{Y}(\mu,\nu)$ with $f_{\widehat{Y}} (y) = \frac{\exp(\frac{y}{2}) f_Y(y)}{M_Y(\frac{1}{2})}$. Then, for $\widehat{Y}(\mu,\nu)$, we have
\begin{align}
	\mathbb{E}[\widehat{Y}(\mu,\nu)] &= \frac{1}{M_Y(\frac{1}{2})} \left[-(1-\mu)\nu  \exp\left(-\frac{c}{2}\right) \cdot c
	+ 
	\mu (1-\nu)  \exp\left(\frac{c}{2}\right) \cdot c \right]\nonumber\\
	&= \frac{1}{M_Y(\frac{1}{2})} \left[
	-(1-\mu)\nu \sqrt{\frac{(1-\nu)\mu}{(1-\mu)\nu}} \cdot c
	+ (1-\nu)\mu \sqrt{\frac{(1-\mu)\nu}{(1-\nu)\mu}} \cdot c
	\right]\nonumber\\
	&= \frac{1}{M_Y(\frac{1}{2})} \left[
	-\sqrt{(1-\mu)(1-\nu)\mu\nu} \cdot c
	+\sqrt{(1-\mu)(1-\nu)\mu\nu} \cdot c
	\right] = 0,
	\label{eq:E-hat-Y}
\end{align}
and 
\begin{align}
	\mathsf{Var}[\widehat{Y}(\mu,\nu)] 
	&= \frac{\sqrt{(1-\mu)(1-\nu)\mu\nu}}{\left(\sqrt{\mu\nu}+\sqrt{(1-\mu)(1-\nu)}\right)^2} \left(\log\frac{(1-\mu)\nu}{(1-\nu)\mu}\right)^2 = O\left(\sqrt{\mu\nu} \right),
	\label{eq:Var-hat-Y}
\end{align}
where $\mu, \nu = \Theta\left(\frac{\log n}{n}\right)$.

Now, we can rewrite the random variable of interest in the lemma as 
\begin{align}
	B &\coloneqq
	\sum_{i=1}^{n_1} \log\left(\frac{1-\theta}{\theta}\right) B_i^{(p)} \left(2 B_i^{(\theta)} \!\!-\!\! 1\right)
	+  \sum_{j=n_1+1}^{n_1+n_2} \log\left(\frac{(1-\beta)\alpha}{(1-\alpha)\beta}\right)  \left( B_j^{(\beta)} \!\!-\!\! B_j^{(\alpha)} \right)\nonumber\\
	&\phantom{\coloneqq} 
	+  \!\!\!\sum_{k=n_1+n_2+1}^{n_1+n_2+n_3} \log\left(\frac{(1-\gamma)\alpha}{(1-\alpha)\gamma}\right)  \left( B_k^{(\gamma)} \!\!-\!\! B_k^{(\alpha)} \right)
	+  \!\!\!\sum_{\ell=n_1+n_2+n_3+1}^{n_1+n_2+n_3+n_4} \log\left(\frac{(1-\gamma)\beta}{(1-\beta)\gamma}\right)  \left( B_\ell^{(\gamma)} \!\!-\!\! B_\ell^{(\beta)} \right)\!,
	\nonumber\\
	&=
	\sum_{i=1}^{n_1}  {X}_i(p,\theta)
	+  \sum_{j=n_1+1}^{n_1+n_2}  {Y}_j(\beta,\alpha) 
	+  \sum_{k=n_1+n_2+1}^{n_1+n_2+n_3}  {Y}_k(\gamma,\alpha)
	+  \sum_{\ell=n_1+n_2+n_3+1}^{n_1+n_2+n_3+n_4}  {Y}_\ell(\gamma,\beta).
	\label{eq:B-in-YX}
\end{align}
Therefore, we can write
\begin{align}
	&\mathbb{P} \left[B \geq 0 \right] 
	\nonumber\\
	&= \mathbb{P}\left[\sum_{i=1}^{n_1}  {X}_i(p,\theta)
	+  \sum_{j=n_1+1}^{n_1+n_2}  {Y}_j(\beta,\alpha) 
	+  \sum_{k=n_1+n_2+1}^{n_1+n_2+n_3}  {Y}_k(\gamma,\alpha)
	+  \sum_{\ell=n_1+n_2+n_3+1}^{n_1+n_2+n_3+n_4}  {Y}_\ell(\gamma,\beta) \geq 0\right]\nonumber\\
	&\geq \mathbb{P}\left[0\leq \sum_{i=1}^{n_1}  {X}_i(p,\theta)
	+  \sum_{j=n_1+1}^{n_1+n_2}  {Y}_j(\beta,\alpha) 
	+  \sum_{k=n_1+n_2+1}^{n_1+n_2+n_3}  {Y}_k(\gamma,\alpha)
	+  \sum_{\ell=n_1+n_2+n_3+1}^{n_1+n_2+n_3+n_4}  {Y}_\ell(\gamma,\beta) < \xi\right]\nonumber\\
	&\stackrel{(a)}{=}
	\sum_{\mathcal{R}(\xi)}\left[
	\prod_{i=1}^{n_1} f_{X(p,\theta)} (x_i)
	\!\!\prod_{j=n_1+1}^{n_1+n_2} f_{Y(\beta,\alpha)} (y_j)
	\!\!\!\prod_{k=n_1+n_2+1}^{n_1+n_2+n_3} f_{Y(\gamma,\alpha)} (y_k)
	\!\!\!\!\prod_{\ell=n_1+n_2+n_3+1}^{n_1+n_2+n_3+n_4} f_{Y(\gamma,\beta)} (y_\ell)
	\right]
	\nonumber\\ 
	&\stackrel{(b)}{\geq}
	\frac{
		\left(M_{X(p,\theta)}\left(\frac{1}{2}\right)\right)^{n_1}
		\left(M_{Y(\beta,\alpha)}\left(\frac{1}{2}\right)\right)^{n_2}
		\left(M_{Y(\gamma,\alpha)}\left(\frac{1}{2}\right)\right)^{n_3}
		\left(M_{Y(\gamma,\beta)}\left(\frac{1}{2}\right)\right)^{n_4}
	}{\exp\left(\frac{1}{2} \xi\right)} 
	\nonumber\\
	&\phantom{\geq}
	\times
	\sum_{\mathcal{R}(\xi)}
	\left[\prod_{i=1}^{n_1} \frac{\exp\left(\frac{1}{2} x_i \right) f_{X(p,\theta)} (x_i)}{M_{X(p,\theta)}\left(\frac{1}{2}\right)}
	\prod_{j=n_1+1}^{n_1+n_2} \frac{\exp\left(\frac{1}{2} y_j \right) f_{Y(\beta,\alpha)} (y_j)}{M_{Y(\beta,\alpha)}\left(\frac{1}{2}\right)}\right.
	\nonumber\\
	&\phantom{
		\geq \times 
		\sum_{\substack{
				\left\{w_1\right\}_{j=1}^{n_2},
				\left\{y_1\right\}_{k=1}^{n_3}, 
		}}
	}
	\left. \cdot
	\prod_{k=n_1+n_2+1}^{n_1+n_2+n_3} \frac{\exp\left(\frac{1}{2} y_k \right) f_{Y(\gamma,\alpha)} (y_k)}{M_{Y(\gamma,\alpha)}\left(\frac{1}{2}\right)}
	\prod_{\ell=n_1+n_2+n_3+1}^{n_1+n_2+n_3+n_4} \frac{\exp\left(\frac{1}{2} y_\ell \right) f_{Y(\gamma,\beta)} (y_\ell)}{M_{Y(\gamma,\beta)}\left(\frac{1}{2}\right)}\right]\nonumber\\
	& = 
	\exp\left(
	n_1 \log M_{X(p,\theta)} \left(\frac{1}{2}\right) 
	+ n_2 \log M_{Y(\beta,\alpha)} \left(\frac{1}{2}\right) 
	+ n_3 \log M_{Y(\gamma,\alpha)}\left(\frac{1}{2}\right)
	+n_4 \log M_{Y(\gamma,\beta)}\left(\frac{1}{2}\right)-\frac{1}{2} \xi\right)\nonumber\\
	&\phantom{=} \times 
	\sum_{\mathcal{R}(\xi)} 
	\left[\prod_{i=1}^{n_1} f_{\widehat{X}(p,\theta)} (x_i) 
	\prod_{j=n_1+1}^{n_1+n_2} f_{\widehat{Y}(\beta,\alpha)} (y_j)
	\prod_{k=n_1+n_2+1}^{n_1+n_2+n_3} f_{\widehat{Y}(\gamma,\alpha)} (y_k)
	\prod_{\ell=n_1+n_2+n_3+1}^{n_1+n_2+n_3+n_4}  f_{\widehat{Y}(\gamma,\beta)} (y_\ell)\right]\nonumber\\
	&\stackrel{(c)}{=} \exp\left( -(1+o(1))(
	n_1 I_r
	+ n_2 \Ig
	+ n_3 \IcOne
	+ n_4 \IcTwo)
	-\frac{1}{2} \xi\right)\nonumber\\ 
	&\phantom{=} \times \mathbb{P} \left[0\leq \sum_{i=1}^{n_1}  \widehat{X}_i(p,\theta)
	+  \sum_{j=n_1+1}^{n_1+n_2}  \widehat{Y}_j(\beta,\alpha) 
	+  \sum_{k=n_1+n_2+1}^{n_1+n_2+n_3}  \widehat{Y}_k(\gamma,\alpha)
	+  \sum_{\ell=n_1+n_2+n_3+1}^{n_1+n_2+n_3+n_4}  \widehat{Y}_\ell(\gamma,\beta) < \xi\right],
	\label{eq:Xi_ineq_1}
\end{align}
where $(a)$ follows from independence of  $X_i(\cdot,\cdot)$'s and $Y_i(\cdot,\cdot)$'s variables in \eqref{eq:B-in-YX} since their indices are different, hence they are generated from independent Bernoulli random variables, and note that the summation in $(a)$ is over 
\begin{align*}
	\mathcal{R}(\xi)=\left\{
	\left\{x_i\right\}_{i=1}^{n_1},
	\left\{y_j\right\}_{j=n_1+1}^{n_1+n_2+n_3+n_4}:   0\leq \sum_{i=1}^{n_1} x_i + \sum_{j=n_1+1}^{n_1+n_2+n_3+n_4} y_j  < \xi
	\right\}.
\end{align*}
Moreover, $(b)$ holds since $\exp \left( \frac{1}{2} \left(\sum_{i=1}^{n_1} x_i + \sum_{j=n_1+1}^{n_1+n_2+n_3+n_4} y_j \right) \right) < \exp\left(\frac{1}{2} \xi \right)$; and $(c)$ holds due to the independence of $\widehat{Y}_i(\cdot,\cdot)$'s and $\widehat{X}_i(\cdot,\cdot)$'s. Finally $I_r$, $\Ig$, $\IcOne$ and $\IcTwo$ in \eqref{eq:Xi_ineq_1} are given by
\begin{align*}
	I_r &= p\left(\sqrt{1-\theta}-\sqrt{\theta}\right)^2,\\
	\Ig &= \left(\sqrt{\alpha}-\sqrt{\beta}\right)^2,\\
	\IcOne &= \left(\sqrt{\alpha}-\sqrt{\gamma}\right)^2,\\
	\IcTwo &= \left(\sqrt{\beta}-\sqrt{\gamma}\right)^2,
\end{align*}
which follow from \eqref{eq:log-M-X} and \eqref{eq:log-M-Y}. 

Note that \eqref{eq:Xi_ineq_1} holds for any value of $\xi$. In particular, we can choose $\xi_n$ satisfying 
\begin{align}
	&\lim_{n\rightarrow \infty} \frac{\xi_n}{n_1 I_r + n_2\Ig + n_3\IcOne + n_4\IcTwo}=0,\label{eq:xi-small}\\
	&\lim_{n\rightarrow \infty} \frac{n_1 p + n_2\sqrt{\alpha \beta} + n_3\sqrt{\alpha \gamma} + n_4\sqrt{\beta \gamma}}{\xi_n^2} =0.
	\label{eq:xi-large}
\end{align}
Therefore, \eqref{eq:xi-small} implies that the exponent in \eqref{eq:Xi_ineq_1} can be rewritten as 
\begin{align}
	-(1+o(1))\left(
	n_1 I_r
	+ n_2 \Ig
	+ n_3 \IcOne
	+ n_4 \IcTwo \right)
	-\frac{1}{2} \xi_n = 
	-(1+o(1))\left(
	n_1 I_r
	+ n_2 \Ig
	+ n_3 \IcOne
	+ n_4 \IcTwo\right).
	\label{eq:Xi_ineq_1:exp}
\end{align}
Moreover, the probability in \eqref{eq:Xi_ineq_1} can be bounded as
\begin{align}
	\mathbb{P} &\left[0\leq \sum_{i=1}^{n_1}  \widehat{X}_i(p,\theta)
	+  \sum_{j=n_1+1}^{n_1+n_2}  \widehat{Y}_j(\beta,\alpha) 
	+  \sum_{k=n_1+n_2+1}^{n_1+n_2+n_3}  \widehat{Y}_k(\gamma,\alpha)
	+  \sum_{\ell=n_1+n_2+n_3+1}^{n_1+n_2+n_3+n_4}  \widehat{Y}_\ell(\gamma,\beta) < \xi_n\right]\nonumber\\
	&\stackrel{(a)}{\geq} \frac{1}{2} - \mathbb{P} \left[ \sum_{i=1}^{n_1}  \widehat{X}_i(p,\theta)
	+  \sum_{j=n_1+1}^{n_1+n_2}  \widehat{Y}_j(\beta,\alpha) 
	+  \!\!\sum_{k=n_1+n_2+1}^{n_1+n_2+n_3}  \widehat{Y}_k(\gamma,\alpha)
	+  \!\!\!\sum_{\ell=n_1+n_2+n_3+1}^{n_1+n_2+n_3+n_4}  \widehat{Y}_\ell(\gamma,\beta) \geq \xi_n\right]\nonumber\\
	& \stackrel{(b)}{\geq} \frac{1}{2} - \frac{n_1 \mathsf{Var}[\widehat{X}(p,\theta)] + 
		n_2\mathsf{Var} [\widehat{Y}(\beta,\alpha)] + n_3\mathsf{Var}[\widehat{Y}(\gamma,\alpha)] + 
		n_4\mathsf{Var}[\widehat{Y}(\gamma,\beta)]  
	}{\xi_n^2}\nonumber\\
	&\stackrel{(c)}{=}\frac{1}{2} -\frac{n_1 O(p) + n_2O(\sqrt{\alpha \beta}) + n_3 O(\sqrt{\alpha \gamma}) + n_4 O(\sqrt{\beta \gamma})}{\xi_n^2}\nonumber\\
	&\stackrel{(d)}{=} \frac{1}{2} - o(1)>\frac{1}{4}, 
	\label{eq:Xi_ineq_1:prob}
\end{align}
where $(a)$ is due to the symmetry of random variables $\widehat{X}(\cdot,\cdot)$ and $\widehat{Y}(\cdot,\cdot)$, $(b)$ follows from Chebyshev's inequality, in $(c)$ the variances are replaced by \eqref{eq:Var-hat-X} and \eqref{eq:Var-hat-Y}, and finally $(d)$ is a consequence of \eqref{eq:xi-large}. Plugging \eqref{eq:Xi_ineq_1:exp} and \eqref{eq:Xi_ineq_1:prob} in \eqref{eq:Xi_ineq_1}, we get the desired bound in Lemma~\ref{lm:lowerB_prob}.
$\hfill \square$

\subsection{Proof of Lemma~\ref{lm:randGraph}}
\label{app:randGraph}
The proof hinges on the alteration method \cite{alon2004probabilistic}. We present a constructive proof for the existence of subgroups $\tG1A$ and $\tG2A$. Let $r= \frac{n}{\log^3 n}$. We start by sampling two random subsets $\uG{i}{A}$ from $G_i^A$ of size $|\uG{i}A|= 2r$, for $i=1,2$. Then, we prune these sets to obtain the desired edge free subsets. To this end, for any pair of nodes $f,g \in \uG1A \cup \uG2A$, we remove both $f$ and $g$ from $\uG1A \cup \uG2A$ if $(f,g)\in E$. We continue this process until the remaining set of nodes is edge-free. Let $\cP$ be the set of nodes we remove from $\uG1A \cup \uG2A$ throughout the pruning process. The expected value of $\cP$ can be upper bounded by 
\begin{align*}
	\E[|\cP|] &\leq 2\E\left[\sum_{f,g\in \uG1A \cup \uG2A} \indicatorFn{(f,g)\in E}\right]\nonumber\\
	&= 2\sum_{f,g\in \uG1A} \E[\indicatorFn{(f,g)\in E}] + 2\sum_{f,g\in \uG2A} \E[\indicatorFn{(f,g)\in E}]+ 
	2\sum_{f\in \uG1A} \sum_{g\in \uG2A} \E[\indicatorFn{(f,g)\in E}] \nonumber\\
	&= 2\sum_{f,g\in \uG1A} \alpha + 2\sum_{f,g\in \uG2A} \alpha + 
	2\sum_{f\in \uG1A} \sum_{g\in \uG2A} \beta \nonumber\\
	&= 2\binom{2r}{2} \alpha + 2\binom{2r}{2} \alpha + 2(2r)^2 \beta \leq 16 r^2 \alpha
\end{align*}
where the last inequality holds since $\beta <\alpha$. Using Markov's inequality for the non-negative random variable $|\cP|$, we obtain
\begin{align}
	\mathbb{P} \left[|\cP| \geq r 
	\right]
	\leq \frac{\mathbb{E} \left[N\right]}{r}
	\leq \frac{16 n}{\log^3 n} \alpha 
	=
	\Theta \left(\frac{n}{\log^3 n} \times \frac{\log n}{n}\right)
	=
	o(1).
	\label{eq:appC_markov}
\end{align}
Therefore, the number of remaining nodes (after pruning) satisfies
\begin{align*}
	\mathbb{P} \left[|\uG1A\cup\uG2A \setminus \cP| > 3r\right] = \mathbb{P} \left[ |\cP| < r\right] = 1- \mathbb{P} \left[ |\cP| \geq r\right] = 1-o(1).
\end{align*}
Hence, $\uG1A \setminus \cP$ and $\uG2A\setminus \cP$ together have at least $3r$ elements. This, together with the fact that $|\uG1A| = |\uG2A| = 2r$, implies each of $\uG1A \setminus \cP$ and $\uG2A\setminus \cP$ have at least $r$ elements. Therefore, we can choose $r$ from $\uG{i}A \setminus \cP$ to form the desired set $\tG{i}A$, for $i=1,2$. This completes the proof of Lemma~\ref{lm:randGraph}.
$\hfill \square$

\end{document}